\documentclass{egpublpreprint}
\usepackage{egpreprint}

\SpecialIssueSubmission    

\CGFccby

\usepackage[T1]{fontenc}
\usepackage{dfadobe}

\biberVersion
\BibtexOrBiblatex
\usepackage[backend=biber,bibstyle=EG,citestyle=alphabetic,backref=true]{biblatex}
\electronicVersion
\PrintedOrElectronic

\ifpdf \usepackage[pdftex]{graphicx} \pdfcompresslevel=9
\else \usepackage[dvips]{graphicx} \fi

\usepackage{egweblnk}

\usepackage{tikz}
\usetikzlibrary{spy}
\usepackage{algorithm}

\usepackage{algpseudocode}
\usepackage{color}
\usepackage{amsmath}
\usepackage{amssymb}
\newtheorem{proposition}{Proposition}[section]
\usepackage{booktabs}
\usepackage{multirow}

\definecolor{brickred}{rgb}{0.8, 0.25, 0.33}
\definecolor{coverletter}{rgb}{0.,0.,0.} 

\newcommand{\nh}{n_{\mathrm{h}}}
\newcommand{\nw}{n_{\mathrm{w}}}
\newcommand{\np}{n_{\mathrm{p}}}
\newcommand{\nc}{n_{\mathrm{c}}}

\newcommand \R {\mathbb{R}}
\newcommand{\mean}{\operatorname{mean}} 
\newcommand{\std}{\operatorname{std}} 

\title{Scaling Painting Style Transfer}

\author[B. Galerne, L. Raad, J. Lezama \& J.-M. Morel]
{\parbox{\textwidth}{\centering Bruno Galerne$^{1,2}$,
Lara Raad$^{3}$,
José Lezama$^{3}$\thanks{José Lezama is now at Google Research. Contributed to this work while at Universidad de la República.},
and Jean-Michel Morel$^{4}$
        }
        \\
{\parbox{\textwidth}{\centering $^1$Institut Denis Poisson, Université d'Orléans, Université de Tours, CNRS \quad $^2$Institut Universitaire de France (IUF) \\
$^3$Instituto de Ingeniería El\'ectrica, Facultad de Ingeniería, Universidad de la Rep\'ublica
\quad
$^4$City University of Hong Kong
       }
}
}

\begin{document}

\teaser{
\newlength{\resultwidth}
\setlength{\resultwidth}{0.235\linewidth}
\begin{tikzpicture}[spy using outlines={rectangle, magnification=8,height=\resultwidth,width=\resultwidth, every spy on node/.append style={thick}}, every node/.style={inner sep=0,outer sep=0}]%
    \node[anchor=south west] (style) at (0,0){\includegraphics[width=0.6595982142857143
\resultwidth]{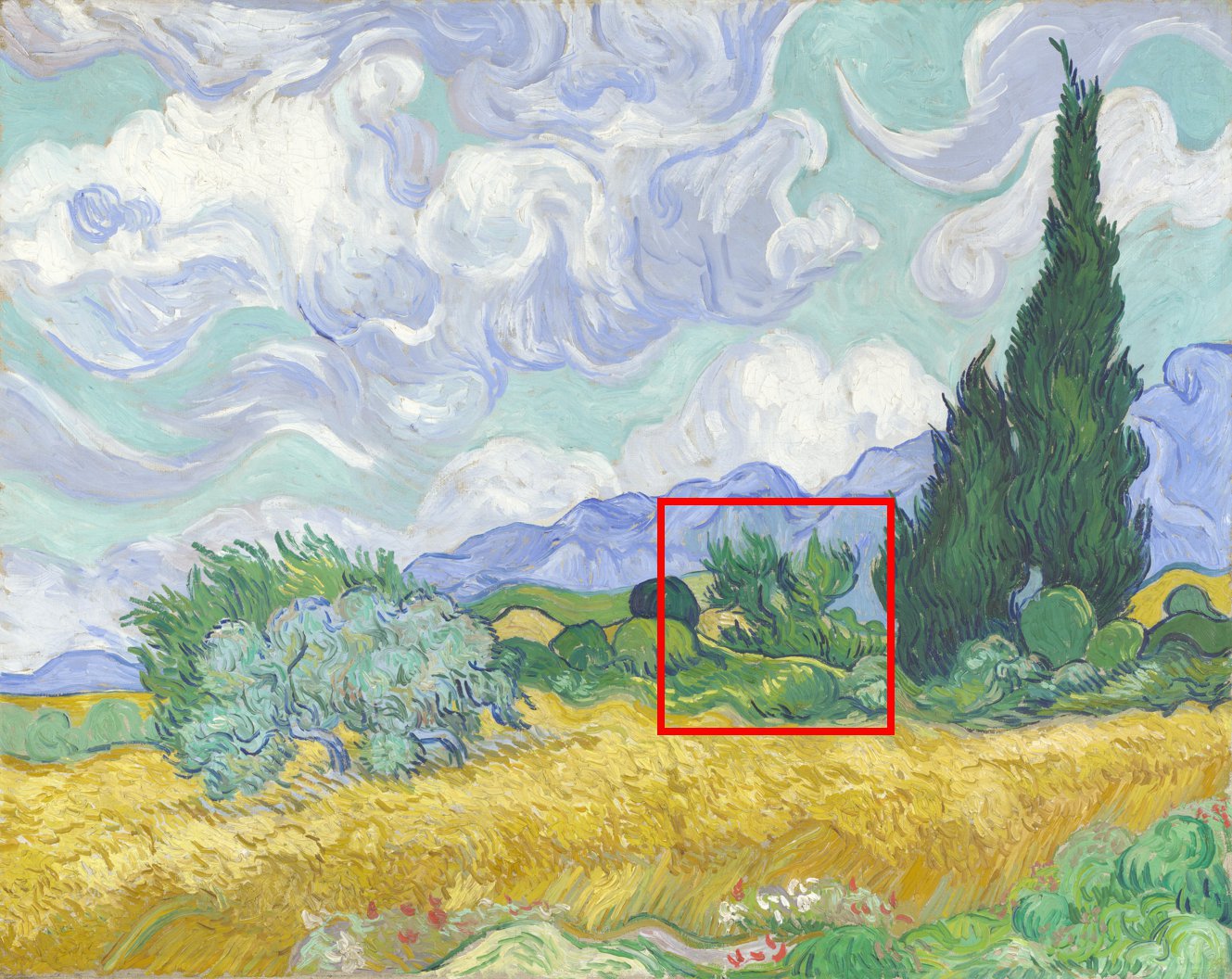}};
    \node[anchor=south west] (content) at
    (\resultwidth+0.02\linewidth,0){\includegraphics[width=\resultwidth]{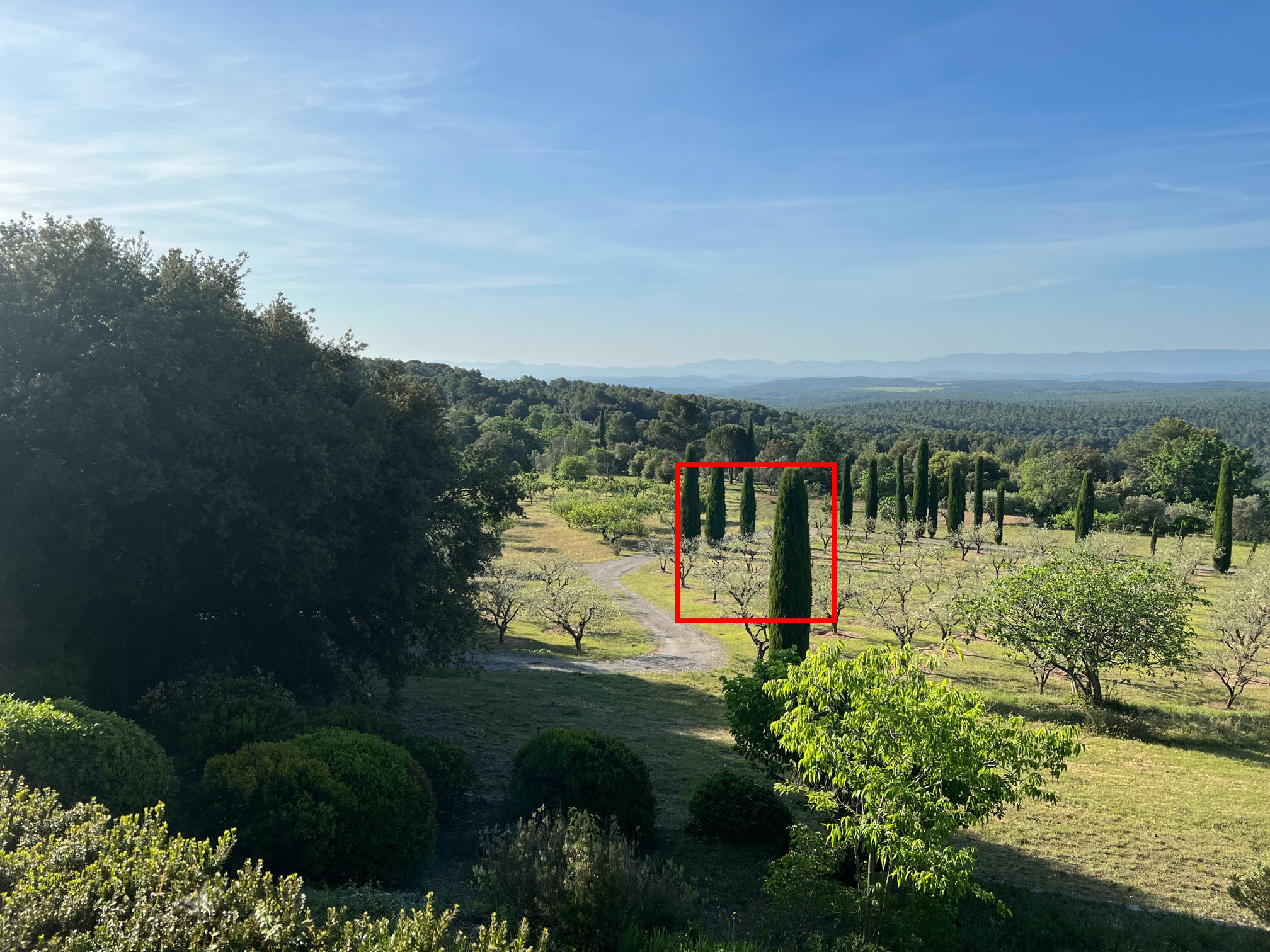}};;
    \node[anchor=south west]  (scale 1) at
    (2*\resultwidth+2*0.02\linewidth,0){\includegraphics[width=\resultwidth]{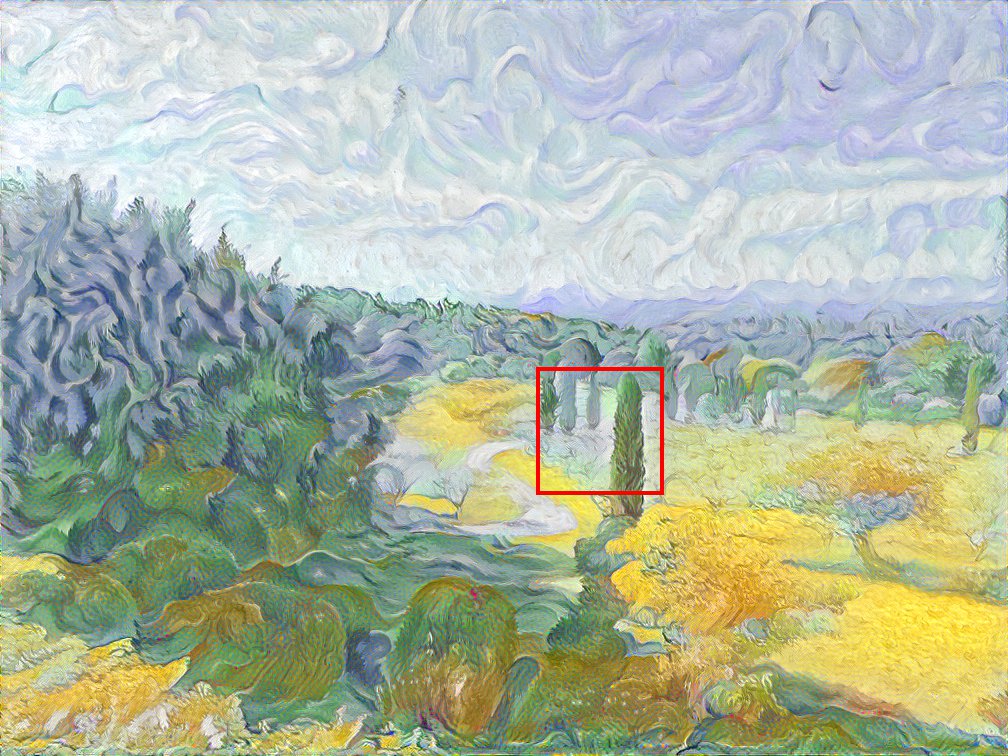}};
    \node[anchor=south west] (scale 2) at
    (3*\resultwidth+3*0.02\linewidth,0){\includegraphics[width=\resultwidth]{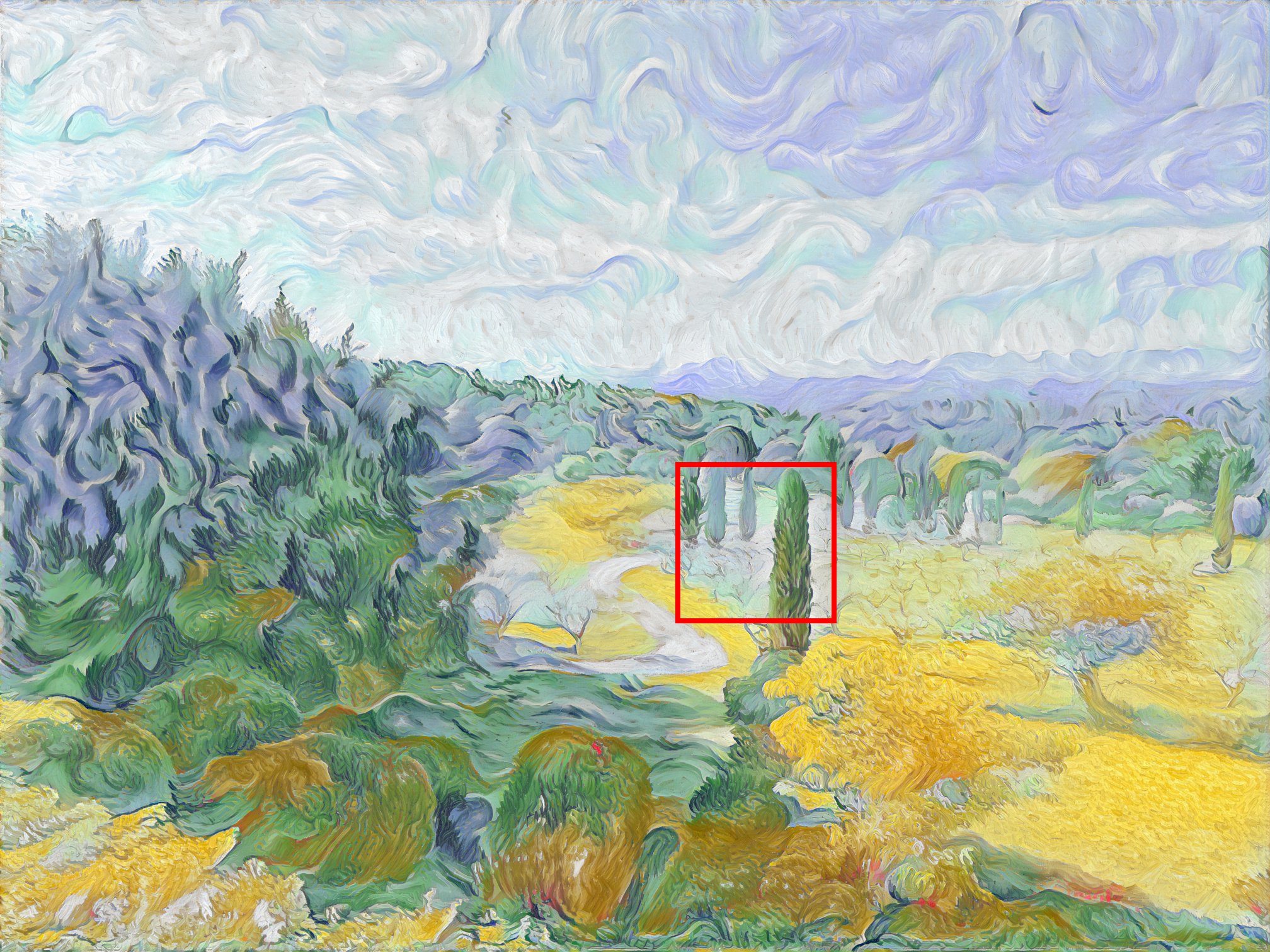}};
    \node[anchor = north west] at
    (0,-0.02\linewidth) {\includegraphics[width=\resultwidth]{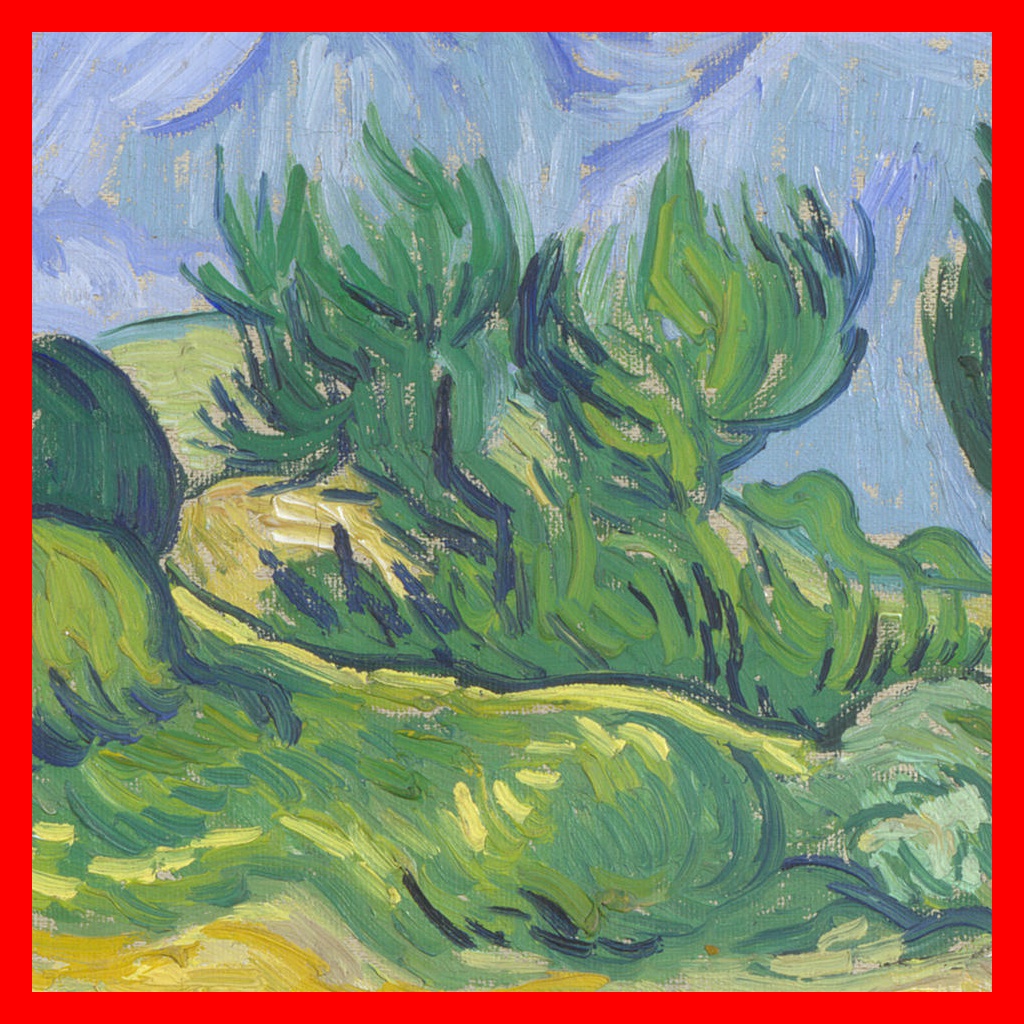}};
    \node[anchor = north west] at
    (\resultwidth+0.02\linewidth,-0.02\linewidth){\includegraphics[width=\resultwidth]{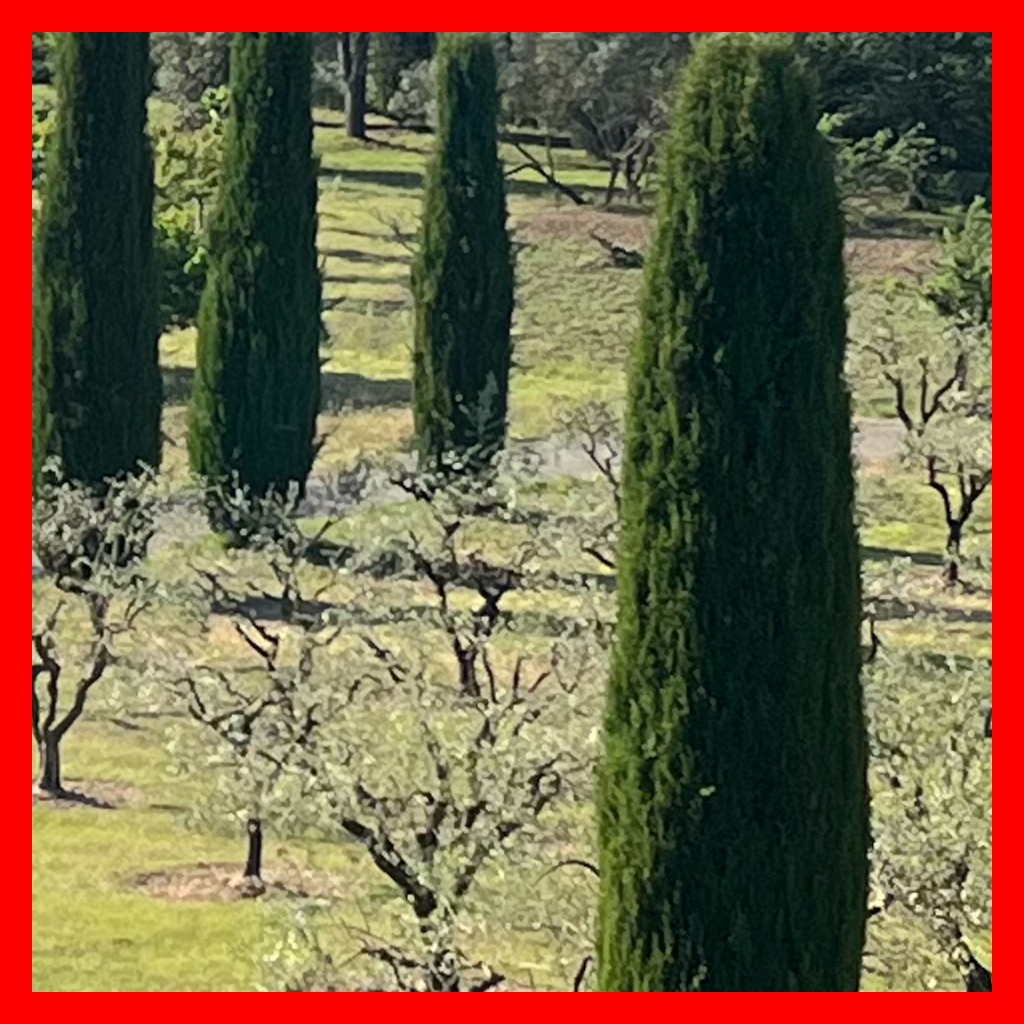}};
    \node[anchor = north west] at
    (2*\resultwidth+2*0.02\linewidth,-0.02\linewidth){\includegraphics[width=\resultwidth]{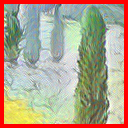}};
    \node[anchor = north west] at
    (3*\resultwidth+3*0.02\linewidth,-0.02\linewidth){\includegraphics[width=\resultwidth]{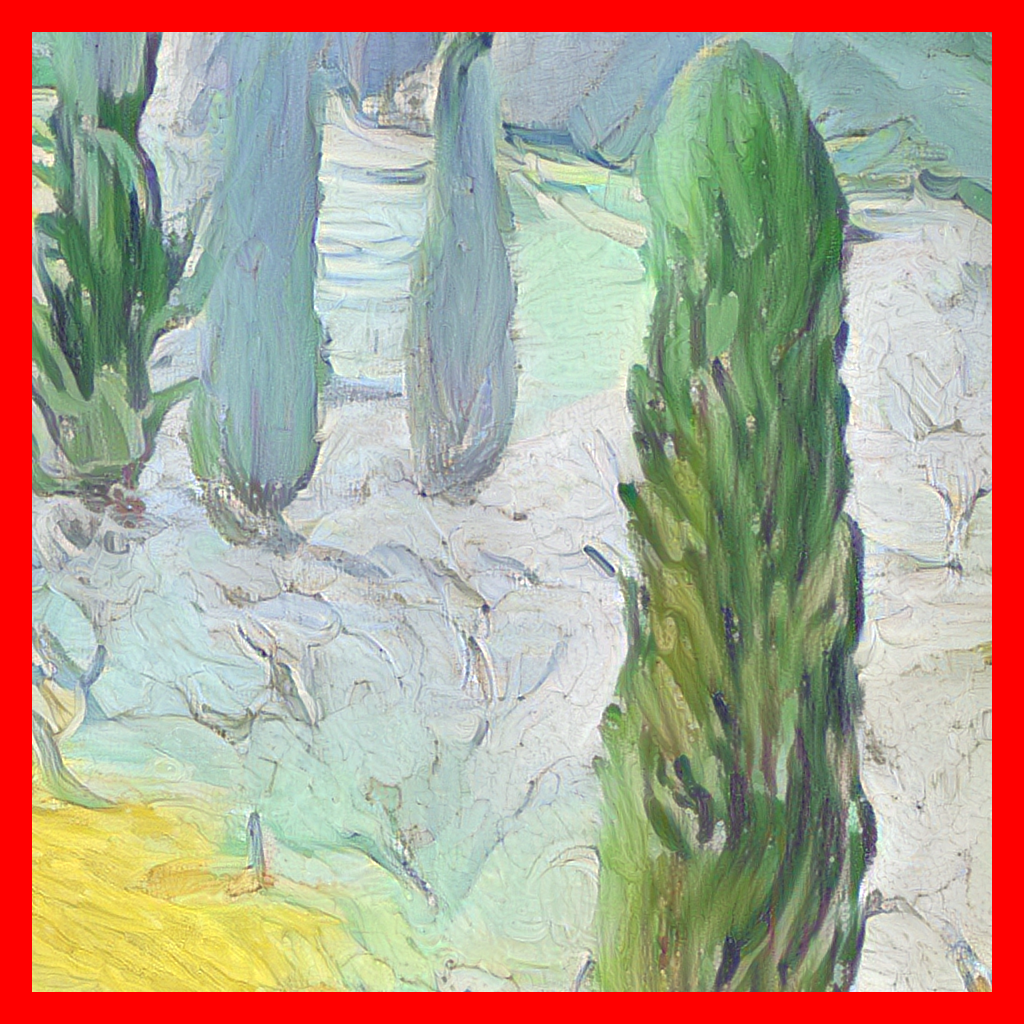}};
\end{tikzpicture}\caption{Ultra-high resolution multiscale style transfer. Top row from left to right: style (4226$\times$5319), content (6048$\times$8064), intermediary transfer at scale 1 (756$\times$1008) and final transfer at scale 4 (6048$\times$8064) (intermediary results at scale 2 (1512$\times$2016) and 3 (3024$\times$4032) are not shown).
Bottom row: zoomed in detail of size 1024$\times$1024 for the three UHR images and 128$\times$128 for the low resolution transfer at scale 1.
Our method produces a style transfer of unmatched quality for such high resolution.
It effectively conveys a pictorial aspect to the output images thanks to fine painting details such as brushstrokes, painting cracks, and canvas texture.
\label{fig:multiscale_style_transfer}}
}

\maketitle
\begin{abstract}
   Neural style transfer (NST) is a deep learning technique that produces an unprecedentedly rich style transfer from a style image to a content image. It is particularly impressive when it comes to transferring style from a painting to an image.
NST was originally achieved by solving an optimization problem to match the global statistics of the style image while preserving the local geometric features of the content image.
The two main drawbacks of this original approach is that it is computationally expensive and that the resolution of the output images is limited by high GPU memory requirements.
Many solutions have been proposed to both accelerate NST and produce images with larger size.
However, our investigation shows that these accelerated methods all compromise the quality of the produced images in the context of painting style transfer.
Indeed, transferring the style of a painting is a complex task involving features at different scales, from the color palette and compositional style to the fine brushstrokes and texture of the canvas.
This paper provides a solution to solve the original global optimization for ultra-high resolution (UHR) images, enabling multiscale NST at unprecedented image sizes.
This is achieved by spatially localizing the computation of each forward and backward passes through the VGG network.
Extensive qualitative and quantitative comparisons, as well as a \textcolor{coverletter}{perceptual study}, show that our method produces style transfer of unmatched quality for such high-resolution painting styles.
By a careful comparison, we show that state-of-the-art fast methods are still  prone to artifacts, thus suggesting that fast painting style transfer remains an open problem.


\end{abstract}


\section{Introduction}
\label{sec:intro}

\begin{figure*}[t]
\newlength{\resultwidthfigone}
\setlength{\resultwidthfigone}{0.66\linewidth}
\newlength{\stylewidthfigone}
\setlength{\stylewidthfigone}{0.98\linewidth}
\addtolength{\stylewidthfigone}{-\resultwidthfigone}
\newlength{\squarewidthfigone}
\setlength{\squarewidthfigone}{0.32\linewidth}

\begin{tikzpicture}[spy using outlines={rectangle, magnification=5,height=0.32\linewidth,width=0.32\linewidth}, every node/.style={inner sep=0,outer sep=0}]%
\node[anchor=north west](content) at
(0,0)
{\includegraphics[width=\stylewidthfigone]{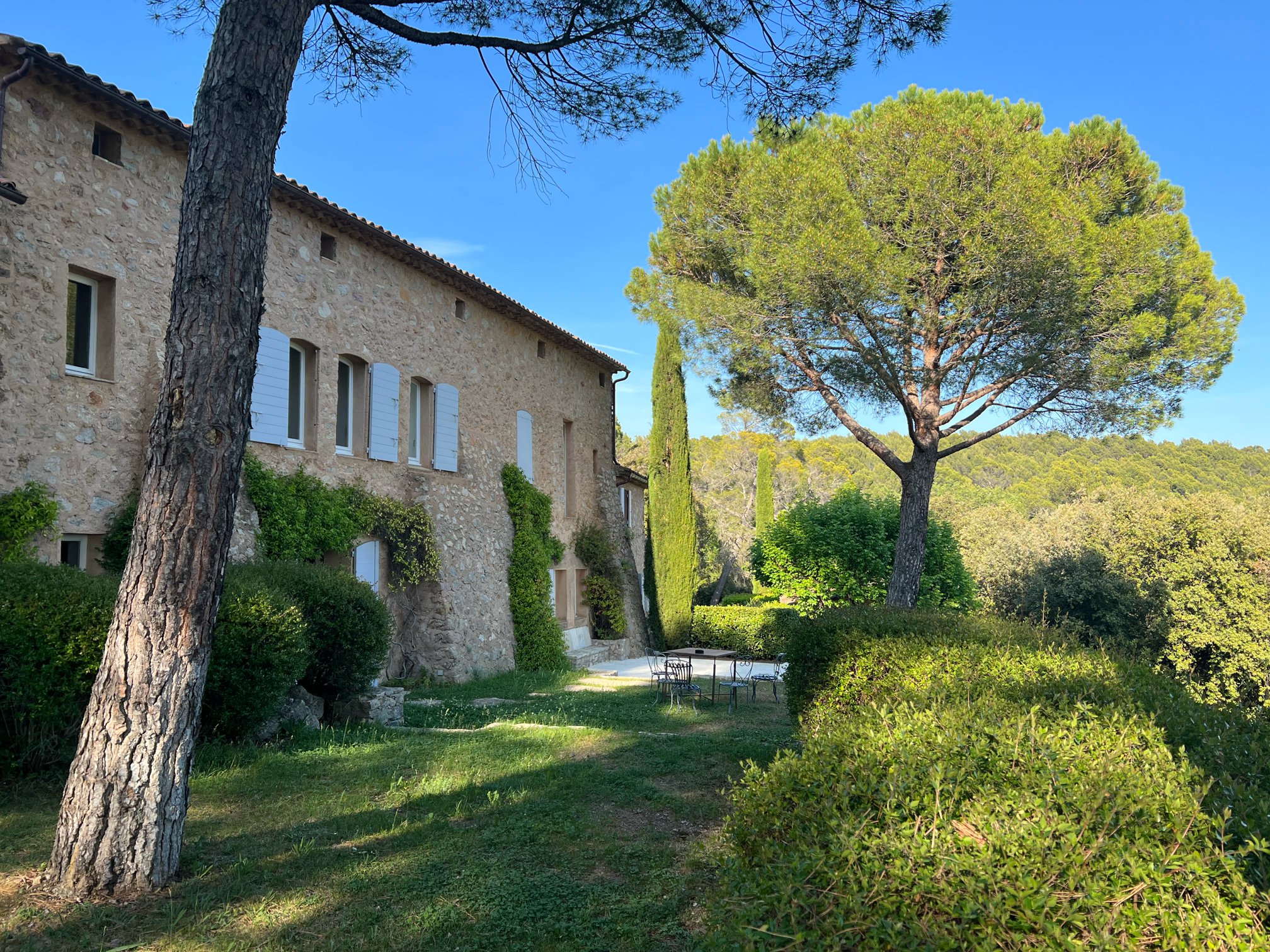}};
\node[anchor=south west](style) at
(0,-0.75*\resultwidthfigone)
{\includegraphics[width= 0.9813988095238095\stylewidthfigone]{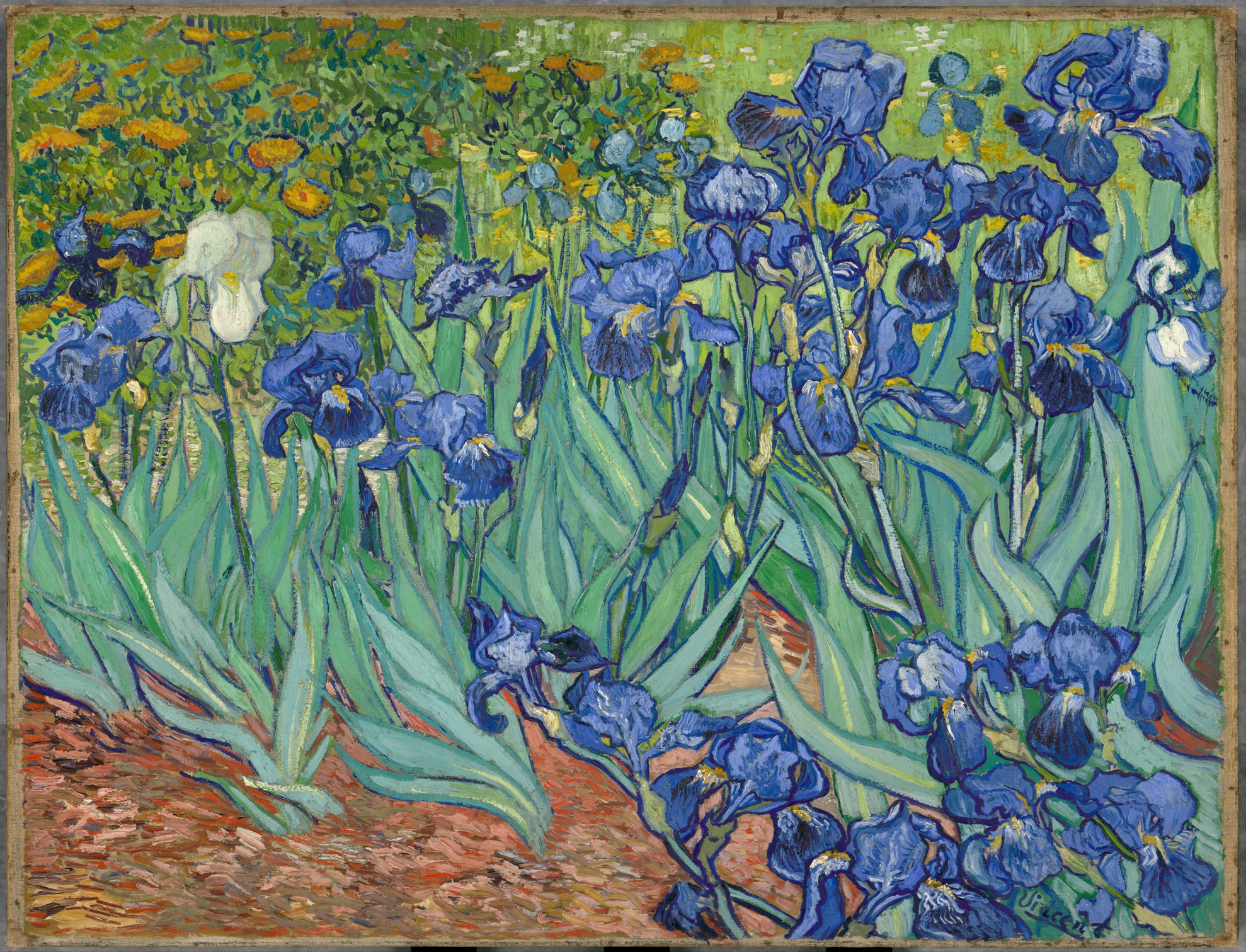}};
\node[anchor = north east] (result) at (\linewidth,0){\includegraphics[width=\resultwidthfigone]{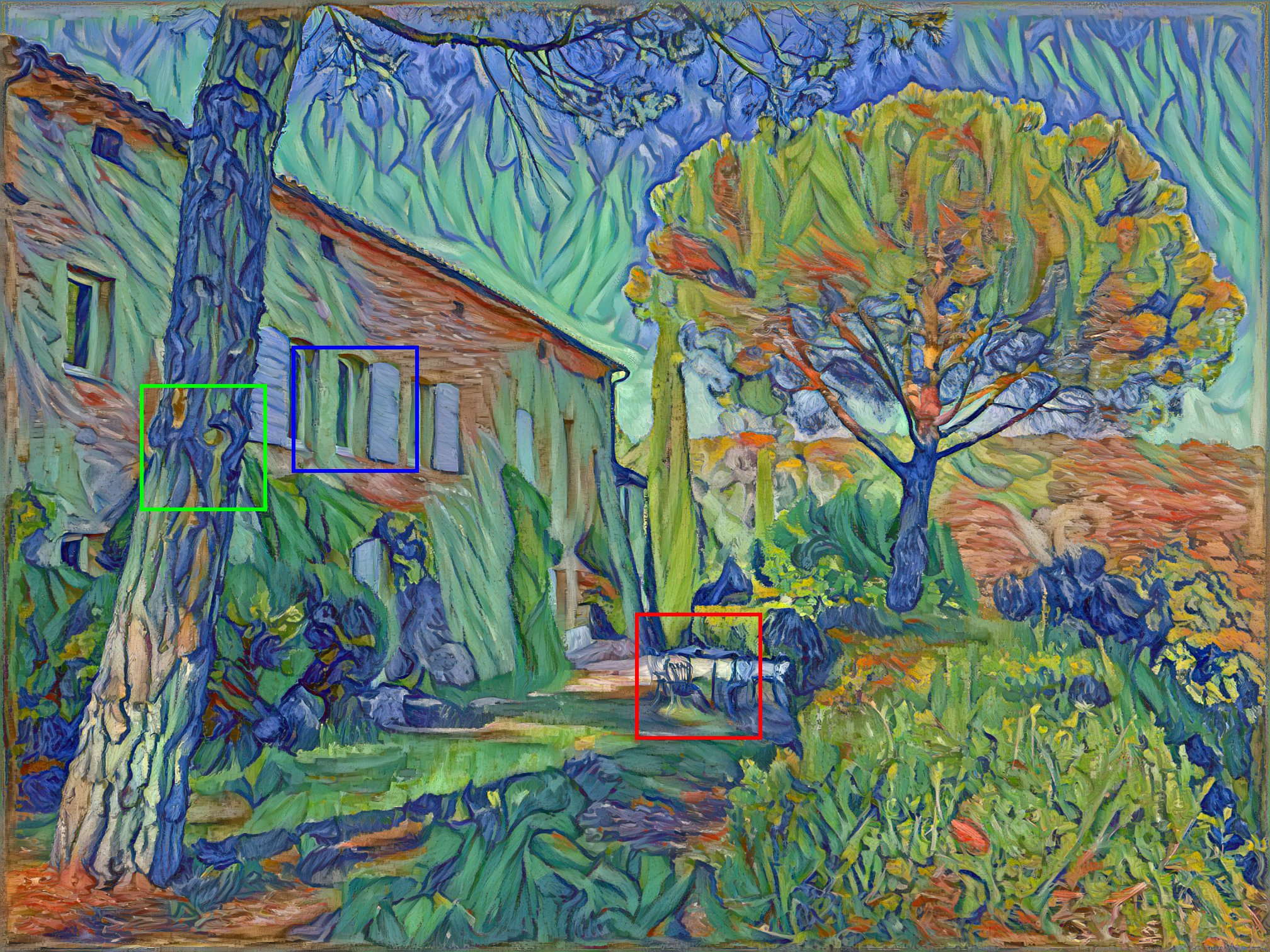}};
\node[anchor = north west] at
(0,-0.75\resultwidthfigone-0.02\linewidth) {\includegraphics[width=\squarewidthfigone]{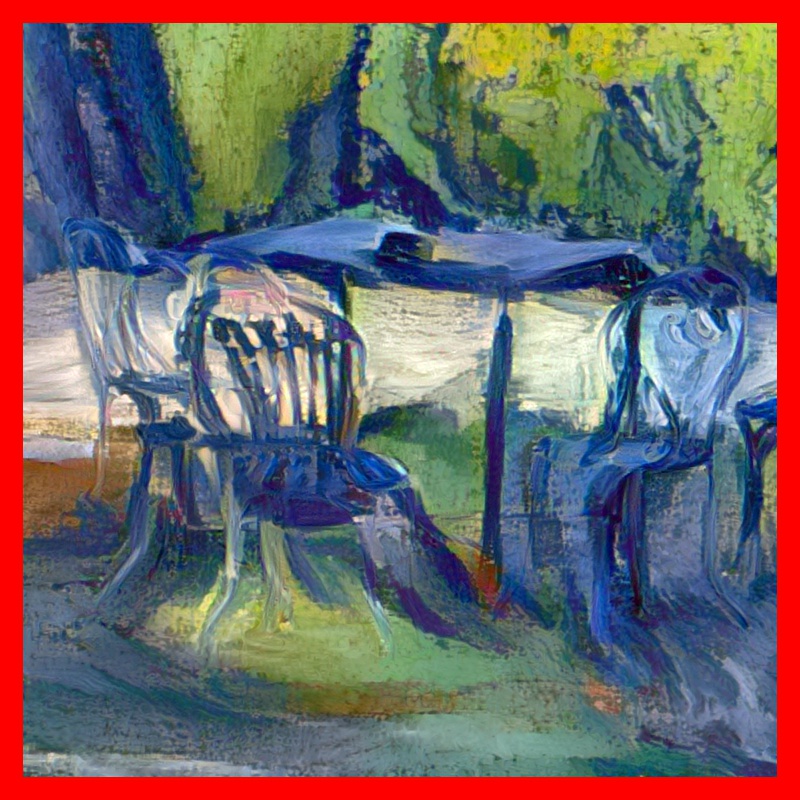}};
\node[anchor = north west] at (0.34\linewidth,-0.75\resultwidthfigone-0.02\linewidth){\includegraphics[width=\squarewidthfigone]{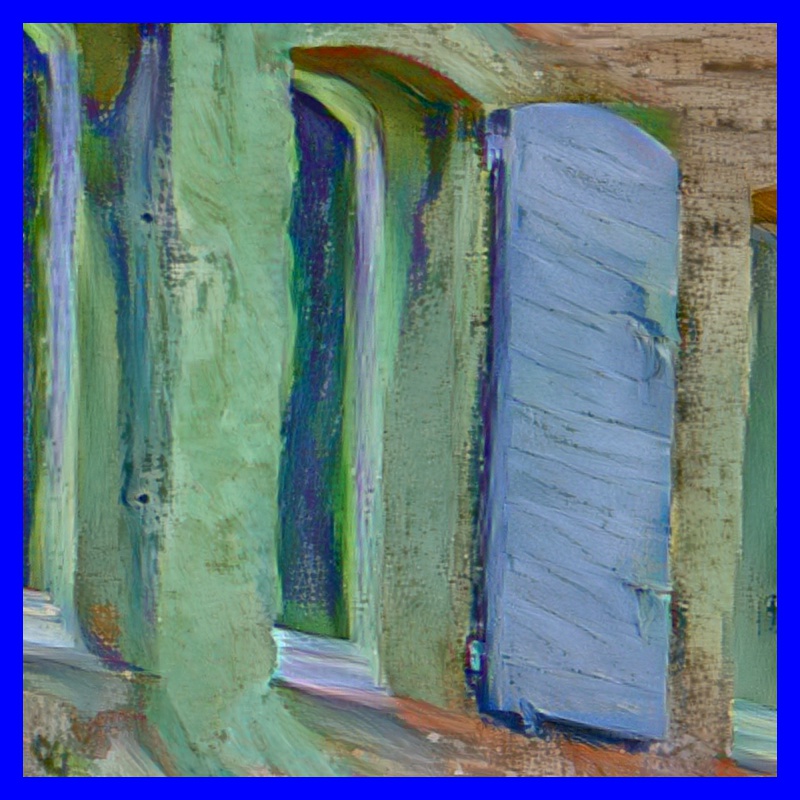}};
\node[anchor = north west] at (0.68\linewidth,-0.75\resultwidthfigone-0.02\linewidth){\includegraphics[width=\squarewidthfigone]{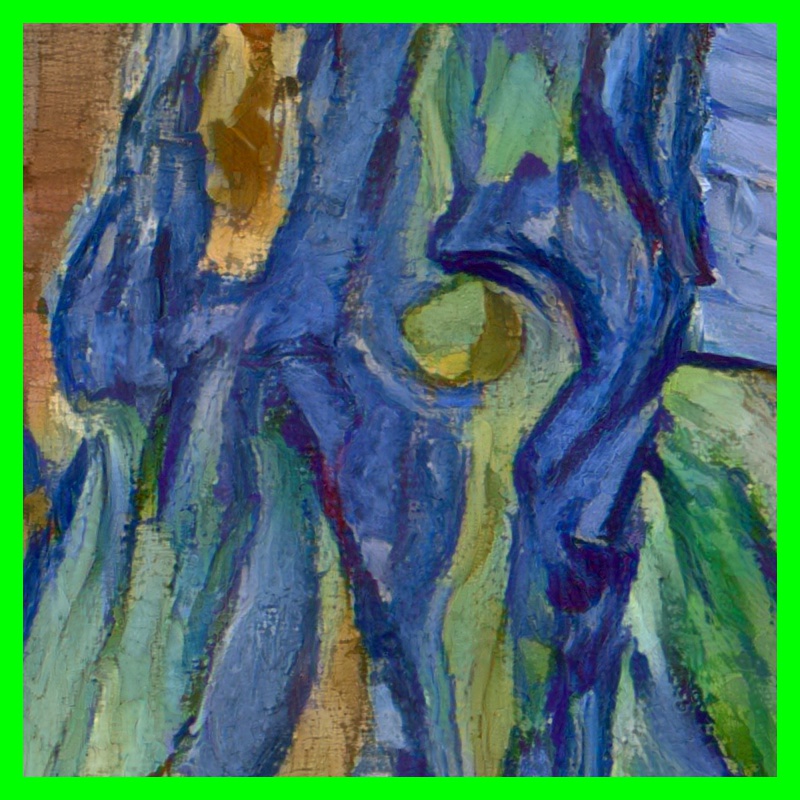}};
\end{tikzpicture}
\caption{UHR style transfer. Top row, content image (top-left, 6048$\times$8064), style image (bottom left, 6048$\times$7914),
result (right, 6048$\times$8064) (the three UHR images are downscaled $\times$4 for visualization).
Bottom row: three zoomed in details of the result image ($800^2$, true resolution). Observe how very fine details such as the chairs look as if painted.}
\label{fig:teaser}
\end{figure*}

Style transfer is an image editing strategy transferring an image style to a content image. Given style and content, the goal is to extract the style characteristics of the style and merge them to the geometric features of the content.
While this problem has a long history in computer vision and computer graphics (e.g.,~\cite{Hertzmann_etal_image_analogies_SIGGRAPH2001,Aubry_etal_fast_local_laplacian_filters_TOG2014}), it has seen a remarkable development since the seminal works of Gatys \emph{et al.} \cite{Gatys_et_al_texture_synthesis_using_CNN_2015,Gatys_et_al_image_style_transfer_cnn_cvpr2016} that introduced \emph{neural style transfer} (NST) \cite{Jing_etal_Neural_style_transfer_a_review_TVCG2020}.
These works demonstrate that the Gram matrices of the activation functions of a pre-trained VGG19 network~\cite{Simonyan_Zisserman_VGG_ICLR15} faithfully encode the perceptual style and textures of an input image.
NST is performed by optimizing a functional aiming at a compromise between fidelity to VGG19 features of the content image while reproducing the Gram matrices statistics of the style image.
Other global statistics have been proven effective for style transfer and texture synthesis~\cite{Lu_Zhu_Wu_Deepframe_AAAI2016,
Sendik_deep_correlations_texture_synthesis_SIGGRAPH2017,
Luan_etal_deep_photo_style_transfer_cvpr2017,
Vacher_etal_texture_interpolation_probing_visual_perception_NEURIPS2020,
Risser_etal_stable_and_controllable_neural_texture_synthesis_and_style_transfer_Arxiv2017,
Heitz_slices_Wassestein_loss_neural_texture_synthesis_CVPR2021,
DeBortoli_et_al_maximum_entropy_methods_texture_synthesis_SIMODS2021,
 gonthier2022high} and it has been shown that a coarse-to-fine multiscale approach allows one to reproduce different levels of style detail for images of moderate to high-resolution (HR)~\cite{Gatys_etal_Controlling_perceptual_factors_in_neural_style_transfer_CVPR2017, snelgrove2017high, gonthier2022high}.
The two major drawbacks of such optimization-based NST are the computation time and the limited resolution of images because of large GPU memory requirements.
The former limitation is more critical for the present work, since conveying the visual aspects of a painting requires multiple scales of visual detail.

Regarding computation time,
several methods have been proposed to generate new stylized images by
training feed-forward networks \cite{ulyanov2016texturenets,johnson2016Perceptual,li2016precomputed} or by training VGG encoder-decoder networks \cite{chen2016fast, Huang_arbitrary_style_transfer_real_time_ICCV2017,li2017universal,li2019learning,chiu2020iterative}.
These models tend to provide images with relatively low 
style transfer loss and can therefore be considered as approximate solutions to~\cite{Gatys_et_al_image_style_transfer_cnn_cvpr2016}. 
Despite a remarkable acceleration, these methods suffer from GPU memory limitations due to the large size of the models used for content and style characterization and are therefore limited in terms of resolution (generally limited to $1024^2$ pixels (px)).

This resolution limitation has received less attention but was recently tackled \cite{an2020real,Wang_2020_CVPR,Chen_Wang_Xie_Lu_Luo_towards_ultra_resolution_neural_style_transfer_thumbnail_instance_normalization_AAAI2022,Wang_etal_MicroAST_AAAI2023}.
{Nevertheless, although} generating UHR images (larger than 4k images), the approximate results are not able to correctly represent the style resolution. Indeed, {for some methods} to satisfy the GPU's memory limitations,  the transfer is performed locally on small patches of the content image with a zoomed out style image ($1024^2$ px)~\cite{Chen_Wang_Xie_Lu_Luo_towards_ultra_resolution_neural_style_transfer_thumbnail_instance_normalization_AAAI2022}. In other methods, the multiscale nature of the networks is not fully exploited~\cite{Wang_2020_CVPR}.

At the opposite of these machine learning based-approaches, we propose to solve the original NST optimization problem~\cite{Gatys_et_al_image_style_transfer_cnn_cvpr2016} for UHR images by introducing an exact localized algorithm.
As illustrated in Figure~\ref{fig:multiscale_style_transfer}, our UHR multiscale method manages to transfer the different levels of detail contained in the style image from the
color palette and compositional style to the fine brushstrokes and canvas texture.
The resulting UHR images look like authentic painting as can be seen in the UHR example of Figure~\ref{fig:teaser}. 

Comparative experiments show that the results of competing methods suffer from brushstroke styles that do not match those of the UHR style image, and that very fine textures are not  transferred well and are subject to local artifacts. To straighten this  visual comparison, we also introduce a qualitative and quantitative \textit{identity test} that highlights how well a given texture is being emulated.
A \textcolor{coverletter}{perceptual study} completes these experiments and confirm the superiority of our approach regarding painting style reproduction.

The main contributions of this work are summarized as follows:
\begin{itemize}
\setlength\itemsep{0em}
  \item We introduce a two-step algorithm to compute the style transfer loss gradient for UHR images that do not fit in GPU memory using localized neural feature calculation.
  \item We show that using this algorithm in a multi-scale procedure leads to a UHR style transfer for images up to 20k$^2$ px with details conveying a natural painting aspect at \textit{every scale}.
  \item Comparative experiments show that the visual quality of our UHR style transfer is by far richer and more faithful than state-of-the-art fast but approximate solutions, revealing that, in our opinion, fast UHR painting style transfer is still an open problem.
\end{itemize}

In particular, the superiority of our approach is confirmed by a blind \textcolor{coverletter}{perceptual study}.
This work provides a new reference method for high-quality  style transfer with unequaled multi-resolution depth.
It also naturally extends the state of the art for UHR texture synthesis.
The main drawback of our approach is that it remains computationally heavy, taking several minutes to produce an image.
Nevertheless, it us up to the users to define their speed vs. quality trade-off, and we believe that our algorithm can be viewed as a new gold standard for practitioners wishing to achieve the highest style transfer image quality.
Our public implementation (available at \url{https://github.com/bgalerne/scaling_painting_style_transfer}) will also allow future research on fast but approximate models to be compared with our method.

\section{Related work}
\label{sec:related-work}

\subsection{Style transfer by optimization}
As recalled in the introduction, the seminal work of Gatys \emph{et al.} formulated style transfer as an optimization problem minimizing the distances between Gram matrices of VGG features \cite{Gatys_et_al_image_style_transfer_cnn_cvpr2016}.
Other global statistics have been proven effective for style transfer and texture synthesis such as  deep correlations \cite{Sendik_deep_correlations_texture_synthesis_SIGGRAPH2017, gonthier2022high},
Bures metric~\cite{Vacher_etal_texture_interpolation_probing_visual_perception_NEURIPS2020},
spatial mean of features~\cite{Lu_Zhu_Wu_Deepframe_AAAI2016, DeBortoli_et_al_maximum_entropy_methods_texture_synthesis_SIMODS2021},
feature histograms~\cite{Risser_etal_stable_and_controllable_neural_texture_synthesis_and_style_transfer_Arxiv2017},
or even the full feature distributions~\cite{Heitz_slices_Wassestein_loss_neural_texture_synthesis_CVPR2021}. Specific cost function corrections have also been proposed for photorealistic style transfer~\cite{Luan_etal_deep_photo_style_transfer_cvpr2017}.
When dealing with HR images, a coarse-to-fine multiscale strategy has been proven efficient to capture the different levels of details present in style images~\cite{Gatys_etal_Controlling_perceptual_factors_in_neural_style_transfer_CVPR2017, snelgrove2017high, gonthier2022high}.
Style transfer by optimization has also been extended for video style transfer~\cite{ruder2016artistic} and style transfer for neural fields~\cite{Zhang_etal_arf_artistic_radiance_fields_ECCV2022}.

The original optimization approach \cite{Gatys_et_al_image_style_transfer_cnn_cvpr2016} was considered unfitted for UHR style transfer due to high memory requirements (limited to 1k$ ^2$ px images \cite{Texler_etal_arbitrary_style_transfer_using_neurally_guided_patch_synthesis_CG2020}).
This paper presents an algorithm that solves this very problem for UHR images.


\subsection{Universal style transfer (UST)}

Ulyanov \emph{et al.} \cite{ulyanov2016texturenets,Ulyanov_etal_improved_texture_networks_CVPR2017} and
Johnson \emph{et al.} \cite{johnson2016Perceptual} 
showed that feed-forward networks could be trained to approximately solve style transfer. Although these models produce a very fast style transfer, they require learning a new model for each style type, making them slower than the original optimization approach when training time is included.

Style limitation was addressed by training a VGG autoencoder that attempts to reverse VGG feature computations after normalizing them at the autoencoder bottleneck. Chen \emph{et al.} \cite{chen2016fast} introduce the encoder-decoder framework 
with a style swap layer replacing content features with the closest style features on overlapping patches. 
Huang \emph{et al.} \cite{Huang_arbitrary_style_transfer_real_time_ICCV2017} propose to use an Adaptive Instance Normalization (AdaIN) that adjusts the mean and variance of the content image features to match those of the style image.
Li \emph{et al.} \cite{li2017universal} match the covariance matrices of the content image features to those of the style image by applying whitening and coloring transforms. 
These operations are performed layer by layer and involve specific reconstruction decoders at each step.
Sheng \emph{et al.} \cite{sheng2018avatar} use one encoder-decoder block combining the transformations of two previous work~\cite{li2017universal,chen2016fast}.
Park and Lee \cite{park2019arbitrary} introduce an attention-based transformation module to integrate the local style patterns according to the spatial distribution of the content image.
Li \emph{et al.} \cite{li2019learning} train 
a symmetric encoder-decoder image reconstruction module and a transformation learning module. 
Chiu and Gurari~\cite{chiu2020iterative} extend the UST approach of Li \emph{et al.}~\cite{li2017universal} by embedding a new transformation that iteratively updates features in a cascade of four autoencoder modules.
Despite the many improvements in fast UST strategies, we remark that: (a) they rely on matching VGG statistics as introduced in~\cite{Gatys_et_al_image_style_transfer_cnn_cvpr2016}
(b) they are limited in resolution due to GPU memory required for large size models.

\subsection{UST for high-resolution images}
Some methods attempt to reduce the size of the network in order to perform high resolution style transfer. ArtNet \cite{an2020real} is a channel-wise pruned version of GoogLeNet~\cite{Szegedy_2015_CVPR}. Wang \emph{et al.} \cite{Wang_2020_CVPR} propose a collaborative distillation approach in order to compress the model by transferring the knowledge of a large network (VGG19) to a smaller one, hence reducing the number of convolutional filters used for UST~\cite{li2017universal, Huang_arbitrary_style_transfer_real_time_ICCV2017}. 
Chen \emph{et al.} \cite{Chen_Wang_Xie_Lu_Luo_towards_ultra_resolution_neural_style_transfer_thumbnail_instance_normalization_AAAI2022} proposed an UHR style transfer framework 
where the content image is divided into patches and a patch-wise style transfer is performed from a zoomed out version of the style image of size $1024^2$ px. Wang \emph{et al.} \cite{Wang_etal_MicroAST_AAAI2023} recently proposed to avoid using pre-trained convolutional deep neural networks for inference and instead train three very lightweight models, a content encoder, a style encoder, and a decoder, resulting in a  ultra-high resolution UST with very low inference time.
However, as will be shown below, the UHR style transfer results generally suffer from visual artifacts and do not faithfully convey the complexity of the style painting at all scales.

Texler \emph{et al.} \cite{Texler_etal_arbitrary_style_transfer_using_neurally_guided_patch_synthesis_CG2020} present a hybrid approach that combines neural networks and patch-based synthesis.
They first perform NST between the low-resolution versions of the content and the style images, then refine the style details using patch-based transfer at a medium resolution followed by an upscaling. By design, this approach only consider a low-resolution version of the content image and suffers from a loss of details in comparison to our method (see supp. mat.).
\textcolor{coverletter}{
In addition to style transfer, other works have addressed HR image synthesis using generative adversarial networks such as
 HR texture synthesis by tiling features in the latent space of a generative adversarial network~\cite{fruhstuck2019tilegan} as well as HR image generation using a bi-level approach~\cite{lin2021infinitygan} based on StyleGAN2~\cite{karras2020analyzing}.}


\section{Global optimization for neural style transfer}

\subsection{Single scale style transfer}
Let us recall the algorithm of Gatys \emph{et al.}~\cite{Gatys_et_al_image_style_transfer_cnn_cvpr2016}.
It solely relies on optimizing some VGG19 second-order statistics for changing the image style while maintaining some VGG19 features to preserve the content image's geometric features.
Style is encoded through Gram matrices of several VGG19 layers, namely the set
$\mathcal{L}_\mathrm{s} = \{\mathtt{ReLU\_k\_1},~k\in\{1,2,3,4,5\}\}$
while the content is encoded with a single feature layer
$L_\mathrm{c} = \mathtt{ReLU\_4\_2}$.

Given a content image $u$ and a style image $v$,
one optimizes the loss function
\begin{equation}
E_{\mathrm{transfer}}(x;(u,v))
=
E_{\mathrm{content}}(x;u)
+ E_{\mathrm{style}}(x;v)
\label{eq:gatys_loss_texture_transfer}
\end{equation}
where
$E_{\mathrm{content}}(x;u) = \lambda_\mathrm{c} \left\| V^{L_\mathrm{c}}(x) -  V^{L_\mathrm{c}}(u)\right\|^2$, with $\lambda_\mathrm{c}>0$, and
\begin{equation}
E_{\mathrm{style}}(x;v) = \sum_{L \in \mathcal{L}_\mathrm{s}} E_{\mathrm{style}}^L(x;v)
\label{eq:style_loss}.
\end{equation}
The style loss for a layer $L \in \mathcal{L}_\mathrm{s}$ is the Gram loss
\begin{equation}
E_{\mathrm{style}}^L(x;v) = w_L\left\| G^L(x) - G^L(v) \right\|^2_\mathrm{F}, \quad w_L>0,
\label{eq:gatys_gram_loss}
\end{equation}
where $\|\cdot\|_\mathrm{F}$ is the Frobenius norm, and, for an image $w$ and a layer index $L$, $G^L(w)$ denotes the Gram matrix of the VGG19 features at layer $L$:
if $V^L(w)$ is the feature response of $w$ at layer $L$ that has spatial size $\nh^L\times \nw^L$ and $\nc^L$ channels, one first reshapes $V^L(w)$ as a matrix of size $\np^L\times \nc^L$ with $\np^L=\nh^L \nw^L$ the number of feature pixels, its associated Gram matrix is the $\nc^L\times \nc^L$ matrix
\begin{equation}
G^L(w) = \frac{1}{\np^L} V^L(w)^\top V^L(w) = \frac{1}{\np^L}\sum_{k=0}^{\np^L} V^L(w)_k (V^L(w)_k)^\top,
\end{equation}
where $V^L(w)_k \in \mathbb{R}^{\nc^L}$ is the column vector corresponding to the $k$-th line of $V^L(w)$.
$E^L_{\mathrm{style}}(x;v)$ is a fourth-degree polynomial and non convex  with respect to (wrt) the VGG features $V^L(x)$.
Gatys \emph{et al.} \cite{Gatys_et_al_texture_synthesis_using_CNN_2015} propose to use the L-BFGS algorithm~\cite{Nocedal_updating_Quasi-Newton_matrices_with_limited_storage_1980} to minimize this loss, after initializing $x$ with the content image $u$.
L-BFGS is an iterative quasi-Newton procedure that approximates the inverse of the Hessian using a fixed size history of the gradient vectors computed during the last iterations.

\subsection{Gram loss correction}
\label{subsec:gram_loss_correction}

Previous works~\cite{Sendik_deep_correlations_texture_synthesis_SIGGRAPH2017,  Risser_etal_stable_and_controllable_neural_texture_synthesis_and_style_transfer_Arxiv2017, Heitz_slices_Wassestein_loss_neural_texture_synthesis_CVPR2021} have shown that optimizing the Gram loss alone may introduce some loss of contrast artifacts.
A proposed explanation is that Gram matrices encompass information regarding both the mean values and correlation of features.
While is has been shown that reproducing the full histogram of the features~\cite{Risser_etal_stable_and_controllable_neural_texture_synthesis_and_style_transfer_Arxiv2017, Heitz_slices_Wassestein_loss_neural_texture_synthesis_CVPR2021} permits to avoid this artefact, we found that simply correcting for the mean and standard deviation (std) of each feature produced visually satisfying results and is computationally simpler.

Given some (reshaped) features $V\in \R^{\np\times \nc}$, define $\mean(V)$ and $\std(V)\in \R^{\nc}$
as the spatial mean and standard deviation vectors of each feature channel.
Throughout the paper, the Gram loss $w_L\left\| G^L(u) - G^L(v) \right\|^2_\mathrm{F}$ of Eq.~\eqref{eq:gatys_gram_loss} is replaced by the following augmented style loss
\begin{equation}
\begin{aligned}
\tilde{E}_{\mathrm{style}}^L(x;v) =
& w_L\left\| G^L(u) - G^L(v) \right\|^2_\mathrm{F}\\
& + w'_L \| \mean(V^L(x)) - \mean(V^L(v)) \|^2 \\
&  + w''_L \| \std(V^L(x)) - \std(V^L(v)) \|^2
\end{aligned}
\label{eq:augmented_gatys_loss}
\end{equation}
for a better reproduction of the feature distribution.
Note that using the ``mean plus std loss'' alone was proposed in \cite{Li_etal_demystifying_neural_style_transfer_IJCAI2017} as an alternate loss for NST (see also \cite{Huang_arbitrary_style_transfer_real_time_ICCV2017}).
The values of all the weights $\lambda_\mathrm{c}$, $w_L$, $w_L'$, $w_L''$, $L \in \mathcal{L}_\mathrm{s}$,  have been fixed for all images (see the provided source code for the exact values).

Limiting our style loss $\tilde{E}_{\mathrm{style}}^L(x;v)$ to second-order statistics is capital for our localized algorithm described in Section~\ref{sec:localized_neural_transfer}.
Indeed, using more involved techniques such as slice Wasserstein distance minimization \cite{Heitz_slices_Wassestein_loss_neural_texture_synthesis_CVPR2021} is not feasible for UHR images due to prohibitive memory requirement.
The visual improvement when replacing ${E}_{\mathrm{style}}^L$ by $\tilde{E}_{\mathrm{style}}^L$ is illustrated in the supp. mat.

\subsection{Limited resolution}
Unfortunately, applying this Gatys \emph{et al.} algorithm off-the-shelf with UHR images is not possible in practice for images of size larger than 4000 px, even with a high-end GPU.
The main limitation comes from the fact that differentiating the loss $E_{\mathrm{transfer}}(x;(u,v))$ wrt $x$ requires fitting into memory $x$ and all its intermediate VGG19 features. While this requires a moderate 2.61 GB for a $1024^2$ px image, it requires $10.2$ GB for a $2048^2$ while scaling up to $4096^2$ is not feasible with a 40 GB GPU.
In the next section we describe a practical solution to overcome this limitation.

\section{Localized style transfer loss gradient}
\label{sec:localized_neural_transfer}

As mentioned in the introduction, our main contribution is to emulate the computation of
\begin{equation}
\nabla_x E_{\mathrm{transfer}}(x;(u,v))
\end{equation}
even for images larger than $4000^2$ px for which evaluation and automatic differentiation of the loss is not feasible due to large memory requirements.

We first discuss how one can compute neural features in a localized way and straightforwardly compute the style transfer loss using a spatial partition of the image.
Then, we demonstrate that this approach allows for the exact computation of the loss gradient using a two-pass procedure.

\subsection{Localized computation of neural features}

First suppose one wants to compute the feature maps $V^L(x)$, $L \in \mathcal{L}_\mathrm{s}\cup\{L_\mathrm{c}\}$, of an UHR image $x$.
The natural idea developed here is to compute the feature maps piece by piece, by partitioning the input image $x$ into small images of size $512^2$, that we will call blocks.
This approach will work up to boundary issues.
Indeed, to compute exactly the feature maps of $x$ one needs the complete receptive field centered at the pixel of interest.
Hence, each block of the partition must be extracted with a margin area, except on the sides that are actual borders for the image $x$.
In all our experiments we use a margin of width $256$ px in the image domain.

This localized way to compute features allows one to compute global feature statistics such as Gram matrices and means and stds vectors. Indeed, these statistics are all spatial averages that can be aggregated block by block by adding sequentially the contribution of each block.
Hence, this easy to implement procedure allows one to compute the value of the loss $E_{\mathrm{transfer}}(x;(u,v))$ (see Equation~\eqref{eq:gatys_loss_texture_transfer}).
However, in contrast with standard practice,
it is \emph{not} possible to automatically differentiate this loss wrt $x$,
because the computation graph linking back to $x$ has been lost.

\subsection{Localized gradient given global statistics}

\begin{figure*}
\includegraphics[width=\linewidth]{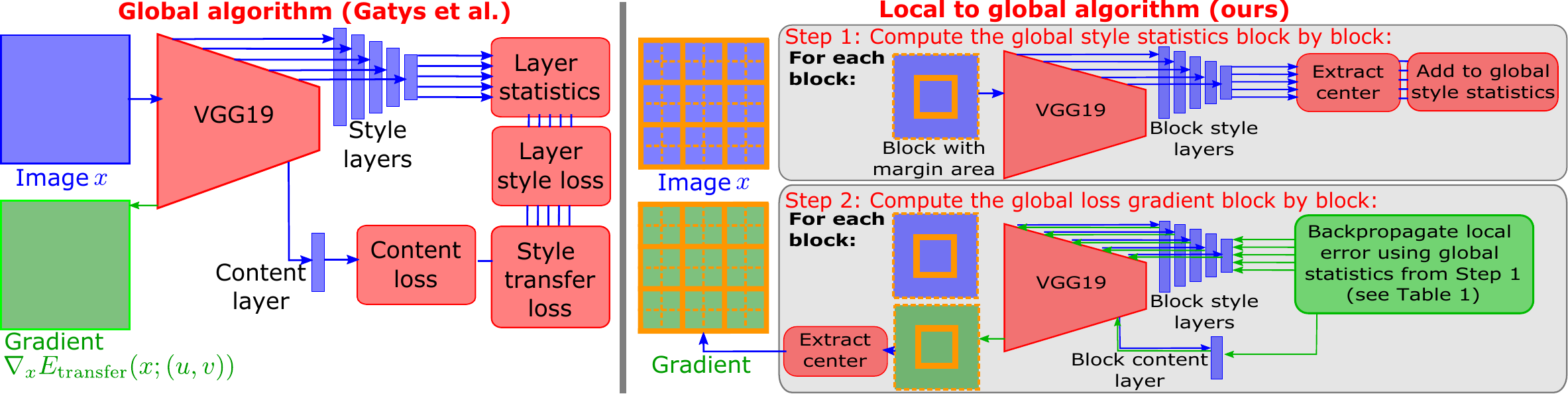}
\caption{Algorithm overview: Our localized algorithm (right part) allows to compute the global style transfer loss and its gradient wrt $x$ for images that are too large for the original algorithm of \cite{Gatys_et_al_image_style_transfer_cnn_cvpr2016} (left part). See Algorithm~\ref{alg:localized_style_transfer_loss_gradient} for the fully detailed procedure.}
\label{fig:algorithm_overview}
\end{figure*}


\begin{table*}
\begin{tabular}{@{}p{0.24\linewidth}p{0.32\linewidth}p{0.38\linewidth}@{}}
\toprule
Global statistics & Feature loss  Expression & Gradient wrt the feature \\
\midrule
Raw features $V$
&
MSE: $E(V) = \|V- V_\mathrm{ref}\|^2$
&
$\nabla_V E(V) = 2(V-V_\mathrm{ref})$
\\[0.45em]

Gram matrix: \textcolor{coverletter}{$G=\frac{1}{\np}V^\top V$}
&
Gram loss:
$E(V) = \left\| G - G_\mathrm{ref} \right\|^2_\mathrm{F}$
&
$\nabla_V E(V) = \frac{4}{\np} V (G - G_\mathrm{ref})$
\\[0.45em]

Feature mean: $\mean(V)$
&
Mean loss:
$E(V) = \| \mean(V)  - \mu_\mathrm{ref}\|^2$
&
$(\nabla_V E(V))_k = \frac{2}{\np} (\mean(V)  - \mu_\mathrm{ref})$
\\[0.45em]

Feature std: $\std(V)$
&
Std loss:
$E(V) = \| \std(V)  - \sigma_\mathrm{ref}\|^2$
&
$\nabla_V E(V)_{k,j} = \frac{2}{\np} (V_{k,j}-(\mean(V))_j) \frac{(\std(V))_j-\sigma_{\mathrm{ref},j}}{(\std(V))_j}$
\\
\bottomrule
\end{tabular}%
\caption{Expression of the feature loss gradient wrt a generic feature $V$ having $\np$ pixels and $\nc$ channels (matrix size $\np\times \nc$).}%
\label{table:feature_losses_gradients}%
\end{table*}

A close inspection of the different style losses wrt the neural features shows that they all have the same form: For each style layer $L\in\mathcal{L}_\mathrm{s}$,
the gradient of the layer style loss $\tilde{E}_{\mathrm{style}}^L(x;v)$ wrt the layer feature $V^L(x)_k \in \mathbb{R}^{\nc^L}$ at some pixel location $k$ only depends on the local value $V^L(x)_k$ and on some difference between the global statistics (Gram matrix, spatial mean, std) of $V^L(x)$ and the corresponding ones from the style layer $V^L(v)$.

\begin{proposition}[Locality of style loss gradient]
Given the layer global statistics values,
the gradient of the layer style loss $\tilde{E}_{\mathrm{style}}^L(x;v)$ wrt the layer feature $V^L(x)\in \mathbb{R}^{\nc^L}$ is local: The gradient value at location $k$ only depends on the feature $V^L(x)_k$ at the same location $k$.
\end{proposition}

\begin{proof}
Recall from Equation~\eqref{eq:augmented_gatys_loss} that
$\tilde{E}_{\mathrm{style}}^L(x;v)$ is a linear combination of the Gram, mean, and std losses.
As shown in Table~\ref{table:feature_losses_gradients}, given the global statistics, each of these losses satisfies the local property.
\end{proof}

\begin{figure*}[t]
\newlength{\gradbbheight}
\setlength{\gradbbheight}{0.123\linewidth}
\begin{tabular}{@{}cccccc@{}}
Style & Content & Global gradient & Localized gradient &  Blocks without margin & \\
\includegraphics[height=\gradbbheight]{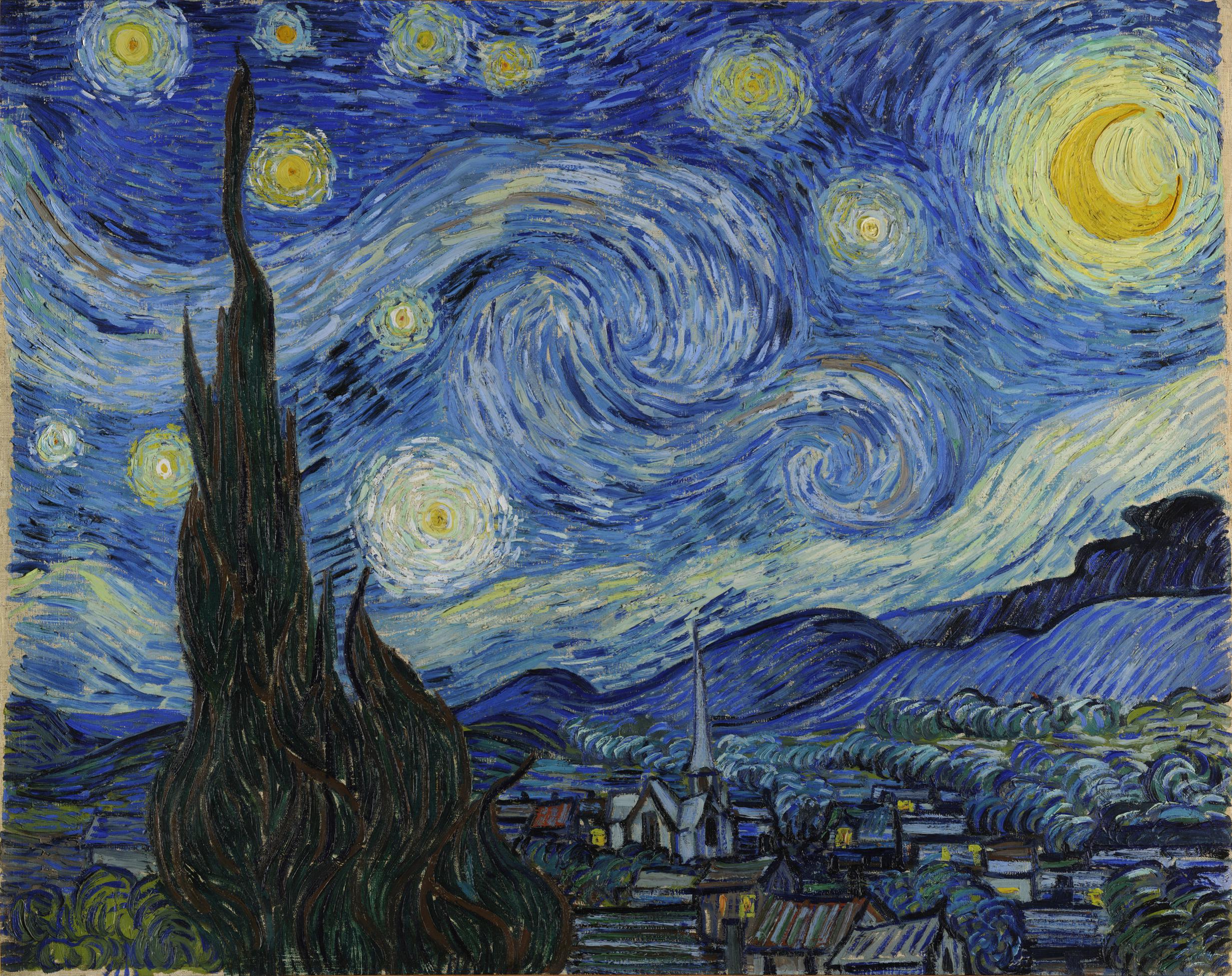}
&
\includegraphics[height=\gradbbheight]{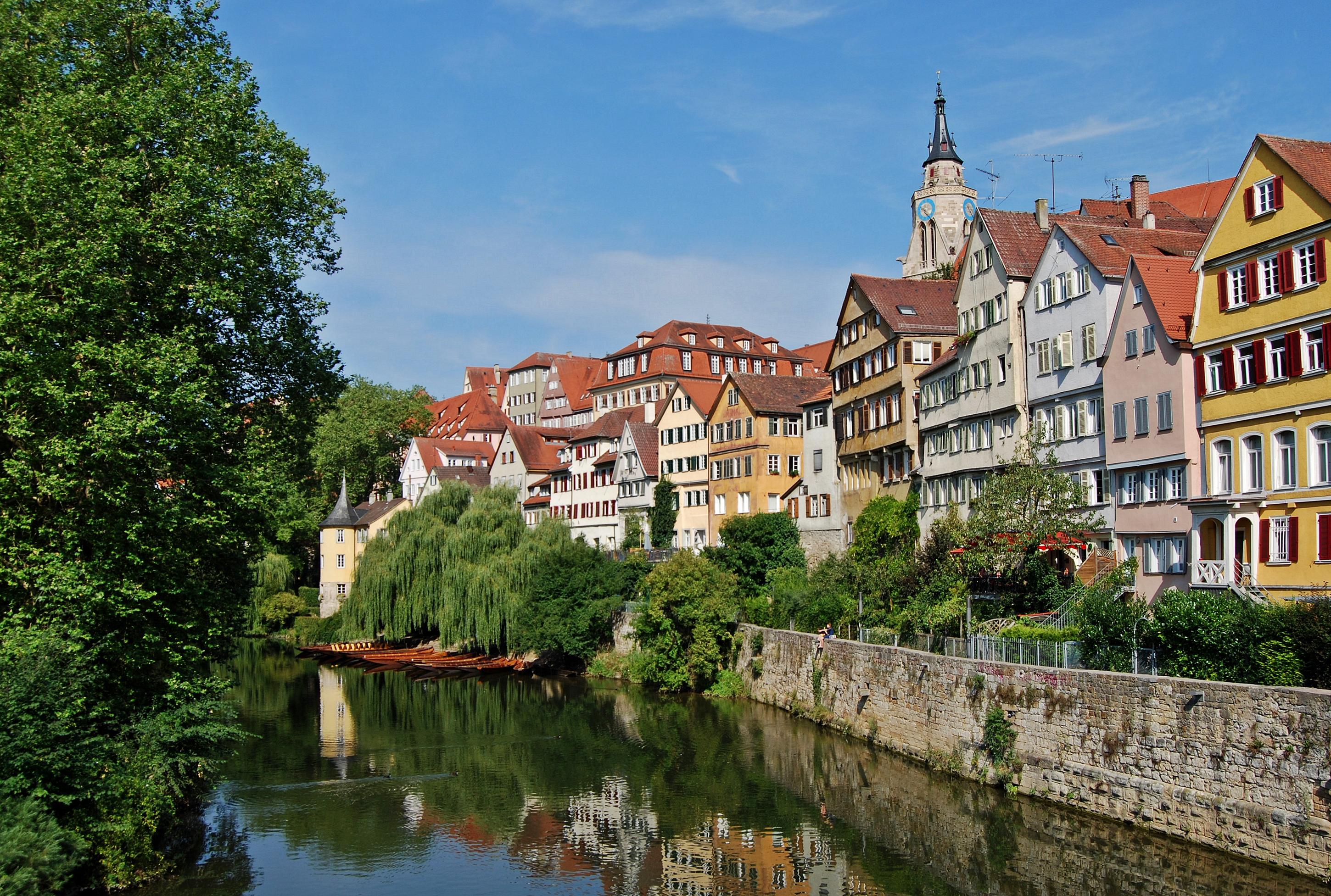}
&
\includegraphics[height=\gradbbheight]{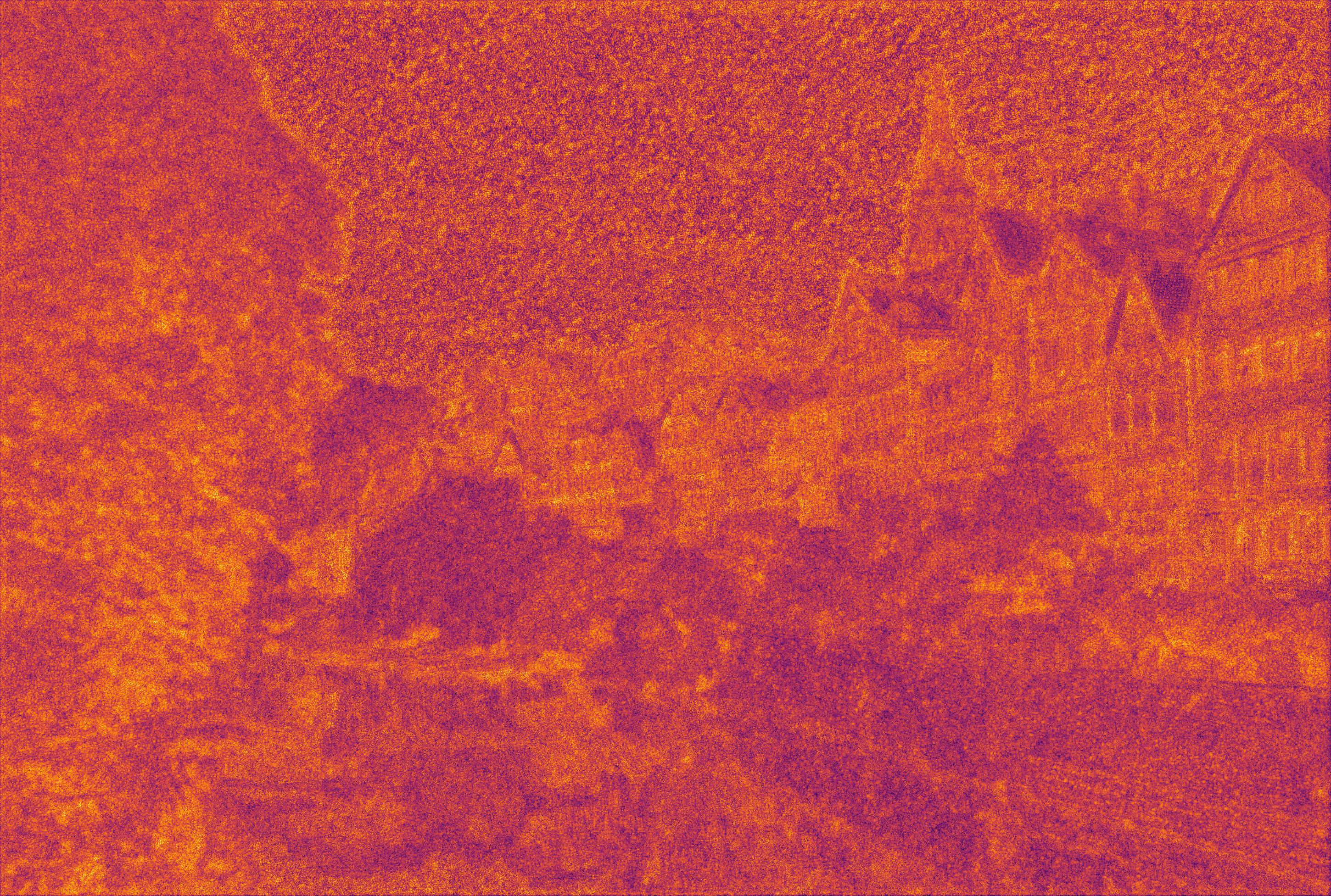}
&
\includegraphics[height=\gradbbheight]{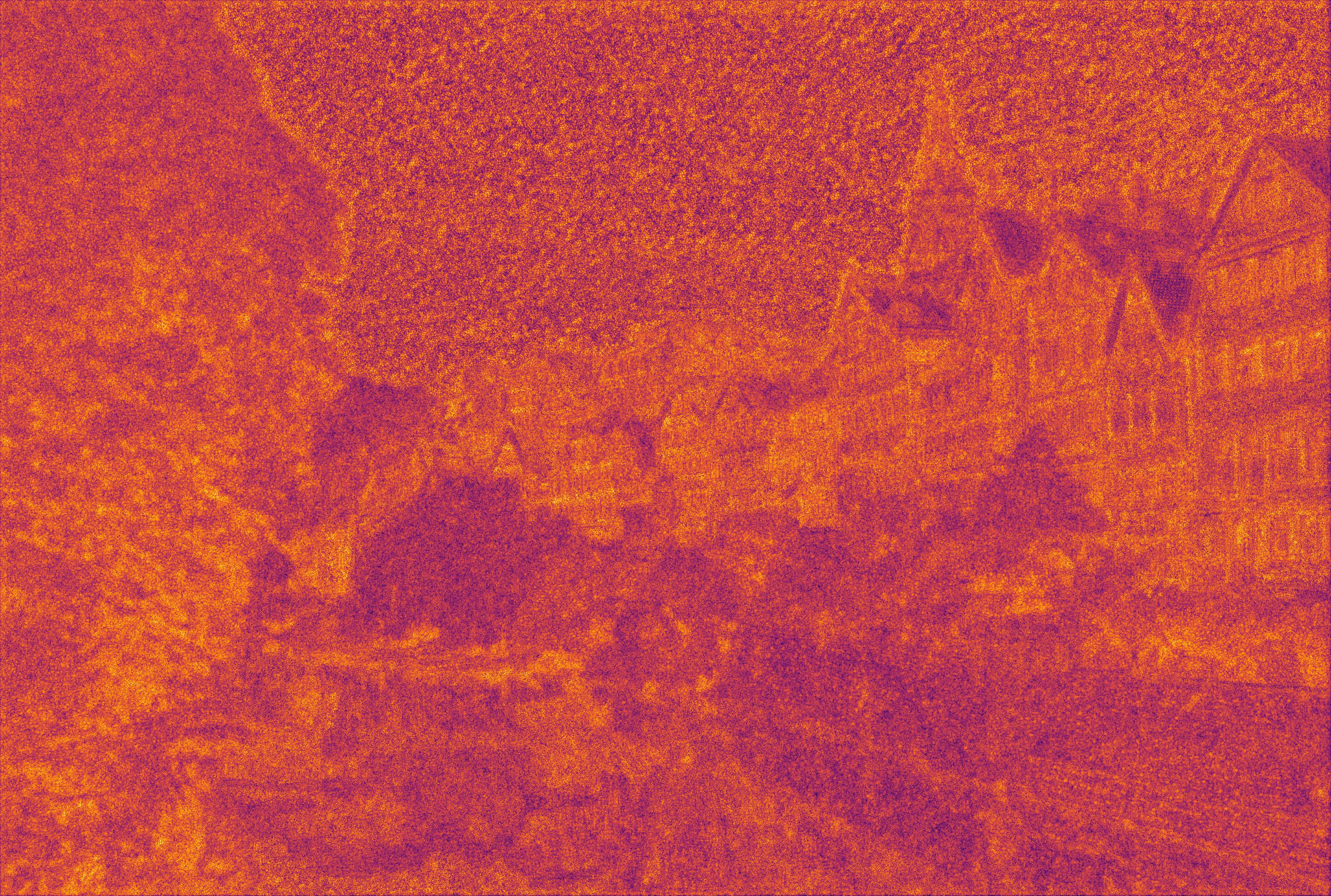}
&
\includegraphics[height=\gradbbheight]{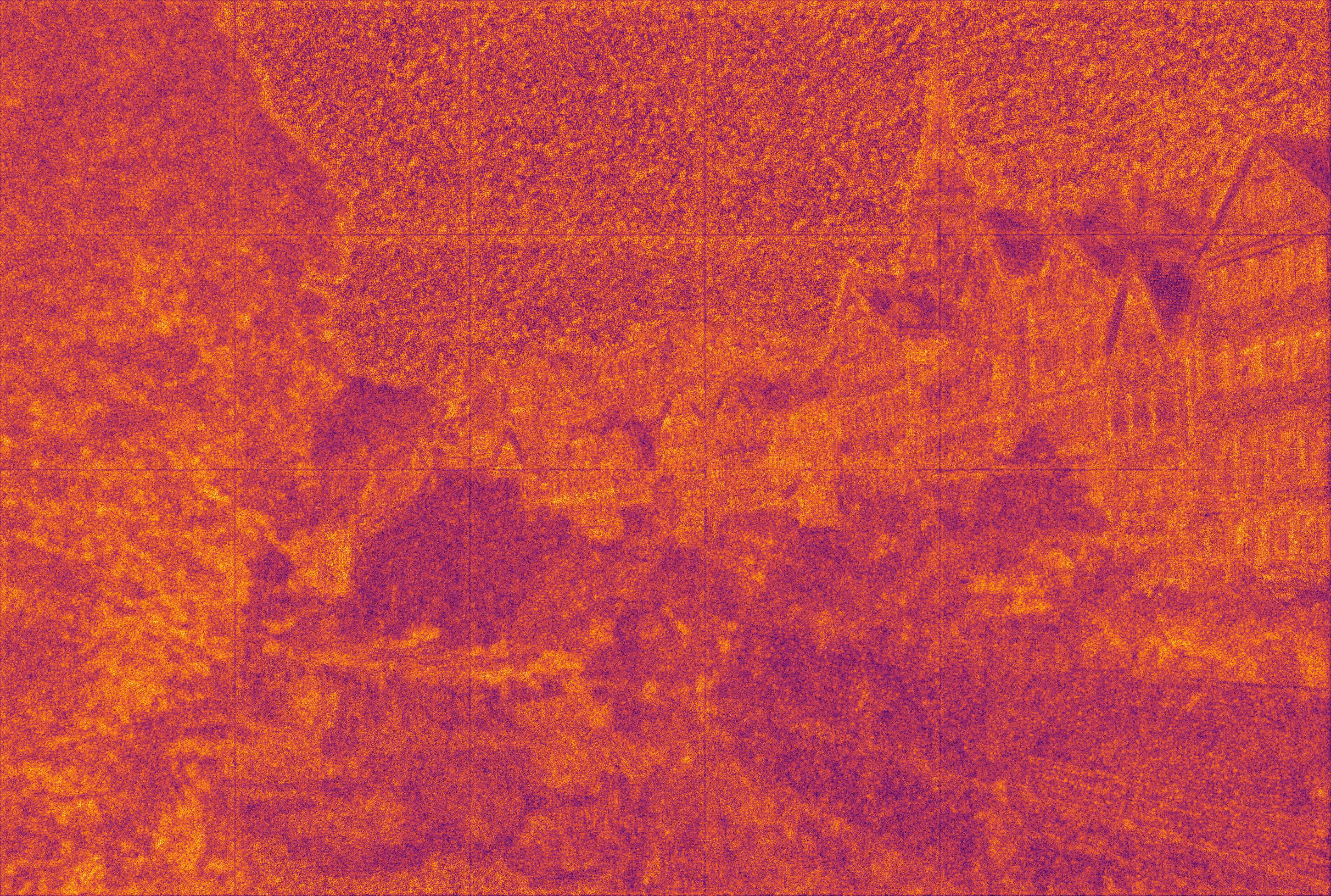}
&
\includegraphics[height=\gradbbheight]{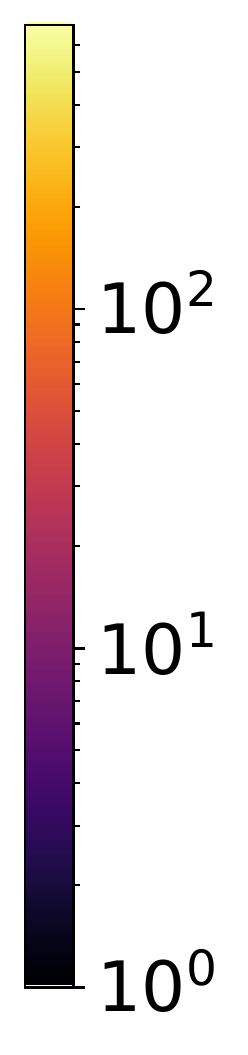}
\end{tabular}
\caption{Localized gradient computation: From left to right: Style image (size 1953$\times$2466), content image (size 1953$\times$2900), norm of the RGB gradient at each pixel computed with three different approaches: reference global gradient~\cite{Gatys_et_al_image_style_transfer_cnn_cvpr2016}, localized gradient using Algorithm~\ref{alg:localized_style_transfer_loss_gradient} (blocks of size 512$\times$512, block margin is $256$), and localized gradient with block margin set to zero.
Algorithm~\ref{alg:localized_style_transfer_loss_gradient} allows for the exact computation of the gradient up to numerical errors (relative error is 1.03$\textrm{e}$-2).
Using a block margin of size zero instead of 256 produces a gradient with visible seams at block boundaries (relative error is 5.23$\textrm{e}$-1).
 \label{fig:gradient_by_blocks_margin}}
\end{figure*}

Exploiting this locality of the gradient, it is also possible to \textit{exactly} compute the gradient vector $\nabla_x E_{\mathrm{transfer}}(x;(u,v))$ block by block using a two-pass procedure:
The first pass is used to compute the global VGG19 statistics  of each style layer and the second pass is used to locally backpropagate the gradient wrt the local neural features.
The whole procedure is described by Algorithm~\ref{alg:localized_style_transfer_loss_gradient} and illustrated by Figure~\ref{fig:algorithm_overview}.
As illustrated by Figure~\ref{fig:gradient_by_blocks_margin}, Algorithm~\ref{alg:localized_style_transfer_loss_gradient} enables to exactly compute the global gradient of the loss in a localized way. The used block margin of size 256 is necessary to avoid visual discontinuities at block boundaries (see Figure~\ref{fig:gradient_by_blocks_margin}).

\begin{algorithm}
\caption{Localized computation of the style transfer loss and its gradient wrt $x$}\label{alg:localized_style_transfer_loss_gradient}
\begin{algorithmic}
\Require Current image $x$, content image layer $V^{L_\mathrm{c}}(u)$, and list of feature statistics of $v$ $\{(G^L(v), \mean(V^L(v)), \std(V^L(v))), ~ L\in\mathcal{L}_\mathrm{s}\}$ (computed block by block)
\Ensure $E_{\mathrm{transfer}}(x;(u,v))$ and $\nabla_x E_{\mathrm{transfer}}(x;(u,v))$
\State \textbf{Step 1: Compute the global style statistics of $x$ block by block:}
\For{each block in the partition of $x$}
  \State Extract the block $b$ with margin and compute $\mathrm{VGG}(b)$ without computation graph
  \State For each style layer $L\in\mathcal{L}_\mathrm{s}$: Extract the features of the block by properly removing the margin and add their contribution to $G^L(x)$ and $\mean(V^L(x))$.
\EndFor
\State \textcolor{coverletter}{For each style layer $L\in\mathcal{L}_\mathrm{s}$: Compute \\$\std(V^L(x))
= (\operatorname{diag}(G^L(x)) - \mean(V^L(x))^2)^{\frac{1}{2}}$.}
\State \textbf{Step 2: Compute the transfer loss and its gradient wrt $x$ block by block:}
\State Initialize the loss and its gradient: $E_{\mathrm{transfer}}(x;(u,v)) \gets  \tilde{E}_{\mathrm{style}}(x;v)$; $\nabla_x E_{\mathrm{transfer}}(x;(u,v)) \gets 0$
\For{each block in the partition of $x$}
  \State Extract the block $b$ with margin and compute $\mathrm{VGG}(b)$ with computation graph
  \State For each style layer $L\in\mathcal{L}_\mathrm{s}$: Compute the gradient of the style loss wrt the local features using the global statistics of $x$ from Step 1 and the style statistics of $v$ as reference (Table~\ref{table:feature_losses_gradients})
  \State For the content layer $L_\mathrm{c}$, add the contribution of $V^{L_\mathrm{c}}(b)$ to the loss $E_{\mathrm{transfer}}(x;(u,v))$ and compute the gradient of the content loss wrt the local features (first row of Table~\ref{table:feature_losses_gradients})
  \State Use automatic differentiation to backpropagate all the feature gradients to the level of the input block image $b$.
  \State Populate the corresponding block of $\nabla_x E_{\mathrm{transfer}}(x;(u,v))$ with the inner part of the gradient obtained by backpropagation.
\EndFor
\end{algorithmic}
\end{algorithm}


\section{Multiscale high-resolution painting style transfer}

\subsection{Coarse-to-fine style transfer}

\begin{algorithm}
\caption{Multiscale style transfer}
\label{alg:multiscale_style_transfer}
\begin{algorithmic}
\Require Content image $u$, a style image $v$,
number of scales $n_\mathrm{scales}$
\Ensure Style transfered image $x$
\For{scale $s=1$ to $n_\mathrm{scales}$}
\State \textbf{Downscale} $u$ and $v$ by a factor $2^{n_\mathrm{scales}-s}$ to obtain the low-resolution couple $(u^{\downarrow},v^{\downarrow})$
\State \textbf{Initialization:} If $s=1$ let $x=u^{\downarrow}$, otherwise upscale current $x$ by a factor 2
\State \textbf{Style transfer at current scale:} \\ $x^{\downarrow} \leftarrow \operatorname{StyleTransfer}((u^{\downarrow},v^{\downarrow}), x^{\downarrow})$ using $n^s_{\textrm{it}}$ iterations of L-BFGS with gradient computed with Algorithm~\ref{alg:localized_style_transfer_loss_gradient}
\EndFor
\end{algorithmic}
\end{algorithm}

Thanks to Algorithm~\ref{alg:localized_style_transfer_loss_gradient}, we can apply style transfer to unprecedented scales.
However, applying a direct style transfer to UHR images generally does not produce the desired effects due to the fixed size of VGG19 receptive fields.
For images larger than 500$^2$ px, visually richer results are obtained by adopting a multiscale approach~\cite{Gatys_etal_Controlling_perceptual_factors_in_neural_style_transfer_CVPR2017} corresponding to the standard coarse-to-fine texture synthesis~\cite{Wei_Levoy_fast_texture_synthesis_2000} that we recall in Algorithm~\ref{alg:multiscale_style_transfer}.

Our two step localized computation approach allows to apply style transfer through up to 6 scales (e.g., from 512$^2$ px to 16384$^2$ px).
Except for the first step, all subsequent style transfers are well-initialized, allowing for a faster optimization~\cite{Gatys_etal_Controlling_perceptual_factors_in_neural_style_transfer_CVPR2017}.
For our baseline implementation, we use L-BFGS with 600 iterations for the first scale and 300 iterations for the subsequent scales.
Due to the large memory needed to store UHR images, the L-BFGS history is limited to the 10 last gradients for all scales except the first one that uses the standard history size of 100.

Finally, in order to avoid GPU memory saturation, for very large images we perform the L-BFGS update procedure and gradient history storing on the CPU for the last scale. This allows to increase the maximal number of pixels by 190\% (+70\% in square image side), as reported in the left column of Table~\ref{tab:max_resoluton_computation_time}. In particular this allows to apply style transfer on images with the unprecedented size of 20k$^2$ using a GPU with 80 GB of memory.

The coarse-to-fine procedure is revealed to be essential to convey the visual complexity of UHR digital photograph of a painting: the first scale encompasses color and large strokes while subsequent scales refine the stroke details up to the painting texture, bristle brushes, fine painting cracks and canvas texture, as illustrated in Figure~\ref{fig:multiscale_style_transfer}.
Surprisingly, fast methods for universal style transfer are not based on a coarse-to-fine approach, which is probably the main reason for their lack of fidelity to fine details (see Section~\ref{subsec:comparison_with_fast_alternatives}).



\begin{figure*}[t]
\setlength{\resultwidth}{0.228\linewidth}
\begin{tikzpicture}[spy using outlines={rectangle, magnification=8,height=\resultwidth,width=\resultwidth, every spy on node/.append style={thick}}, every node/.style={inner sep=0,outer sep=0}]%
    \node[anchor=south west] at
    (0*\resultwidth+0*0.02\linewidth,0){\includegraphics[width=\resultwidth]{img/cvpr_multiscale/square_result_sc_1_of_4--iters_600_300_300_300_cw_1.0_wmstd_1000.png}};
    \node[anchor=south west] at
    (1*\resultwidth+1*0.02\linewidth,0){\includegraphics[width=\resultwidth]{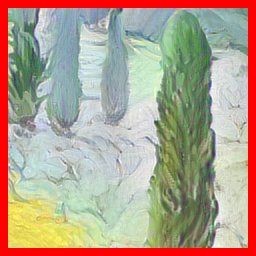}};
    \node[anchor=south west]  at
    (2*\resultwidth+2*0.02\linewidth,0){\includegraphics[width=\resultwidth]{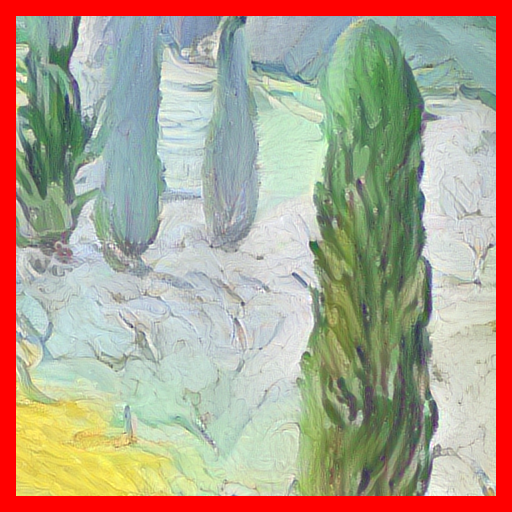}};
    \node[anchor=south west]  at
    (3*\resultwidth+3*0.02\linewidth,0){\includegraphics[width=\resultwidth]{img/cvpr_multiscale/square_result_sc_4_of_4--iters_600_300_300_300_cw_1.0_wmstd_1000.png}};
    \node[] at (0.5*\resultwidth,\resultwidth+0.01\linewidth) {{Scale 1 (128$\times$128)}};
    \node[] at (1.5*\resultwidth+1*0.02\linewidth,\resultwidth+0.01\linewidth) {{Scale 2 (256$\times$256)}};
    \node[] at (2.5*\resultwidth+2*0.02\linewidth,\resultwidth+0.01\linewidth) {{Scale 3 (512$\times$512)}};
    \node[] at (3.5*\resultwidth+3*0.02\linewidth,\resultwidth+0.01\linewidth) {{Scale 4 (1024$\times$1024)}};
    \node[rotate=90] at (-0.02\linewidth,0.5\resultwidth) {{SPST (baseline)}};

\end{tikzpicture}

\vspace{0.015\linewidth}

\begin{tikzpicture}[spy using outlines={rectangle, magnification=8,height=\resultwidth,width=\resultwidth, every spy on node/.append style={thick}}, every node/.style={inner sep=0,outer sep=0}]%
    \node[anchor=south west] at
    (0*\resultwidth+0*0.02\linewidth,0){\includegraphics[width=\resultwidth]{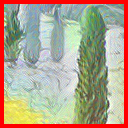}};
    \node[anchor=south west] at
    (1*\resultwidth+1*0.02\linewidth,0){\includegraphics[width=\resultwidth]{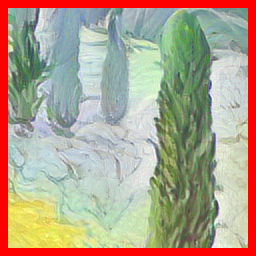}};
    \node[anchor=south west]  at
    (2*\resultwidth+2*0.02\linewidth,0){\includegraphics[width=\resultwidth]{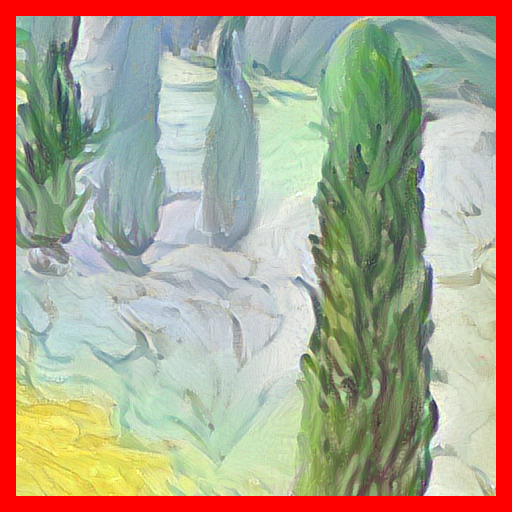}};
    \node[anchor=south west]  at
    (3*\resultwidth+3*0.02\linewidth,0){\includegraphics[width=\resultwidth]{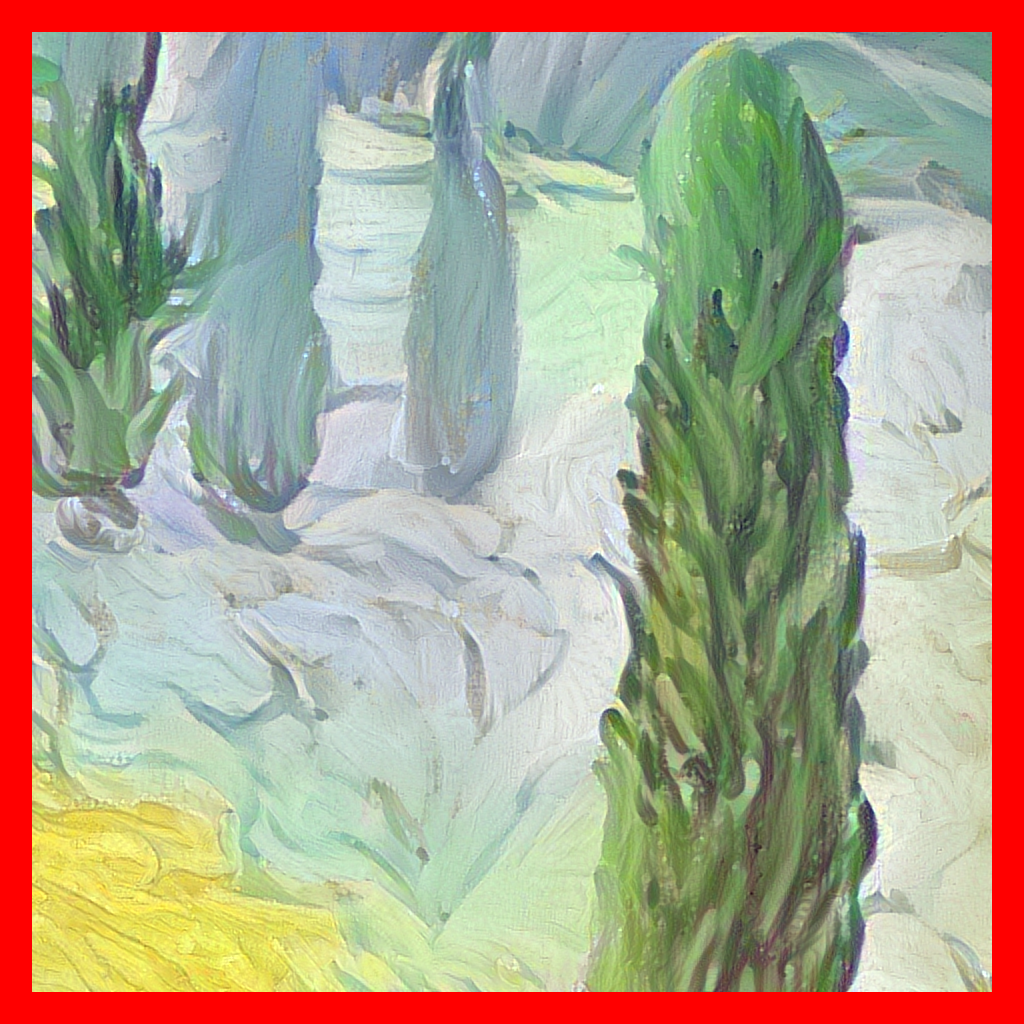}};
    \node[rotate=90] at (-0.02\linewidth,0.5\resultwidth) {{SPST-fast}};

\end{tikzpicture}
\caption{UHR multiscale style transfer: Stability of upscaling and SPST-fast.
The top row shows intermediary steps for the experiment of Figure~\ref{fig:multiscale_style_transfer} displaying details of size $128^2$ to $1024^2$.
While the transfer is globally stable from one scale to the next,
each upscaling enables the addition of fine pictorial details which give an authentic painting aspect to the final output image.
Bottom row: Same details for the output of the SPST-fast alternative that uses less and less L-BFGS iterations after each upscaling.
Computation times  for the full scale image outputs of size 6048$\times$8064 are 74 minutes for SPST and 13 minutes for SPST-fast ($\times$5.6 speed up).
Observe that both methods produce very close results but that the very fine details of the SPST-fast output are slightly less complex.
\label{fig:multiscale_style_transfer_details_fast}}
\end{figure*}

\subsection{Accelerated multiscale style transfer}

\begin{table}[t]
    \centering
    \begin{tabular}{@{}llcccc@{}}
    \toprule
    GPU (VRAM) & & \multicolumn{4}{c}{Computation time (in min.)}%
    \tabularnewline
    \multicolumn{2}{@{}l}{w/ max. res. GPU / GPU+CPU} &  2k & 4k & 8k & 16k \\
    \midrule
    RTX 2080 (11GB) &  SPST & 12.8 & 70.3$^\star$ & - & - \tabularnewline
    \smallskip max. res. 3k / 5k$^\star$ & SPST-fast & 4.3 & 13.1$^\star$ & - & -
    \tabularnewline

    A100 (40GB) & SPST & 4.0 & 20.6 & 96.6 & - \tabularnewline
    \smallskip max. res. 8k / 14k$^\star$ & SPST-fast & 1.7 & 4.1 & 15.2 & -
\tabularnewline

    A100 (80GB) & SPST & 3.8 & 19.0 & 89.7 & 406$^\star$ \tabularnewline
max. res. 12k / 20k$^\star$ & SPST-fast & 1.5 & 3.9 & 14.4 & 61.4$^\star$

    \tabularnewline
    \bottomrule
    \end{tabular}
    \caption{
    Resolution and computation time of SPST and SPST-fast depending on GPU hardware.
    Below each hardware name, we give the maximum resolution achievable with full computation on the GPU and the maximum resolution achievable when using the CPU for L-BGFS steps for the last scale (denoted by $^\star$).
    For the computation times, $^\star$ indicates that L-BFGS optimization had to be moved to the CPU to avoid GPU memory saturation.
    All images are square.}
    \label{tab:max_resoluton_computation_time}
\end{table}

The main drawback of our baseline approach is the computational cost.
Indeed, the complexity is linear in the number of pixels, making each upscaling step four times longer than the previous one.
Nevertheless, we experimentally observed that the style transfer is remarkably stable from one step to the next, as can be observed in the top row of Figure~\ref{fig:multiscale_style_transfer_details_fast}.
To the best of our knowledge, this property has never been reported, probably because style transfer involving several scales was not reachable without  our localized algorithm for gradient computation.

The role of the last steps is to refine local texture in accordance to the style image at the current resolution.
To allow for a faster alternative, we found that these last steps can be alleviated by reducing the number of iterations.
We thus propose an alternative procedure, called \emph{SPST-fast} in what follows, that reduces the number of iterations by a factor $3$ from one scale to the other, while ensuring a minimal number of 30 iterations, e.g., for 4 scales one uses $(n^s_{\textrm{it}})_{1\leq x\leq 4} = (600, 200, 66, 30)$ instead of $(n^s_{\textrm{it}})_{1\leq x\leq 4} = (600, 300, 300, 300)$ for the baseline implementation.
Computation times for both SPST and SPST-fast are reported in Table~\ref{tab:max_resoluton_computation_time} for three different GPU hardwares.
They show that SPST-fast is about five times faster than SPST.
Note that our algorithm allows for multiscale style transfer of UHR images up to 20k$^2$ px.
Even on a moderate GPU with 11 GB of memory, our algorithm can deal with images of size 5k$^2$ px, while the original implementation of \cite{Gatys_et_al_image_style_transfer_cnn_cvpr2016}
does not run on a 40GB GPU for an image of size 4k$^2$ px.
Our source code is available at \url{https://github.com/bgalerne/scaling_painting_style_transfer}.

As shown in Figure~\ref{fig:multiscale_style_transfer_details_fast}, SPST-fast produces visually satisfying results but with small texture details that are slightly less aligned with the UHR content image compared to the SPST baseline approach.

\section{Numerical Results}
\label{sec:experiments}

\subsection{Ultra-high resolution style transfer}

An example of UHR style transfer is displayed in Figure~\ref{fig:teaser} with several highlighted details.
Figure~\ref{fig:multiscale_style_transfer} illustrates intermediary steps of our high resolution multiscale algorithm.
The result for the first scale (third column) corresponds to the ones of the original paper~\cite{Gatys_et_al_image_style_transfer_cnn_cvpr2016} (except for our slightly modified style loss) and suffers from poor image resolution and grid artifacts. As already discussed with Figure~\ref{fig:multiscale_style_transfer_details_fast}, while progressing to the last scale, the texture of the painting gets refined and stroke details gain a natural aspect. This process is remarkably stable; the successive global style transfers results remain consistent with the one of the first scale.

\subsection{Ultra-high resolution texture synthesis}

Although we focus our discussion on style transfer,
our approach also allows for UHR texture synthesis.
Following the original paper on texture synthesis~\cite{Gatys_et_al_texture_synthesis_using_CNN_2015},
given a texture exemplar $v$, texture synthesis is performed by minimizing
$E_{\mathrm{style}}(x;v)$~\eqref{eq:style_loss},
starting from a random white noise image $x_0$.
From a practical point of view,
it consists in minimizing the style transfer loss with the three following differences:
a) The style image is replaced by the texture image.
b) There is no content image and no content loss (set $\lambda_\mathrm{c}=0$).
c) The image $x$ is initialized as a random white noise $x_0$.
We perform texture synthesis following the same multiscale approach and using the augmented style loss $\tilde{E}_{\mathrm{style}}^L(x;v)$ defined in Equation~\eqref{eq:augmented_gatys_loss}.

Our experiments show that for texture synthesis, one should use a number of scales as high as possible, that is, the multiscale process starts with images of moderate size (about 200 pixels).
To illustrate this point we show two different UHR texture synthesis in
Figures~\ref{fig:texture_synthesis_one}
(six additional results are displayed in the supp. mat.).
For each example, the synthesis using three scales (same setting as for style transfer) and five scales is shown.
Starting with a first scale with small size is critical for a satisfying synthesis quality.
Indeed, using only three scales yields textures that are spatially homogeneous due to the white noise initialization.

Let us recall that our approach enables to reach up to 20k$^2$ px (see Table~\ref{tab:max_resoluton_computation_time} for maximal resolutions), which pushes by far the maximal resolution for neural texture synthesis. Indeed, to the best of our knowledge the highest resolution reported in the neural texture synthesis literature was limited to $2048^2$ px~\cite{gonthier2022high} for the multiscale version of Gatys \emph{et al.} algorithm~\cite{Gatys_et_al_texture_synthesis_using_CNN_2015}.

\newlength{\textsynthtwidth}
\begin{figure*}
\setlength{\textsynthtwidth}{0.32\linewidth}
\begin{tikzpicture}[spy using outlines={rectangle, magnification=5,height=\resultwidth,width=\resultwidth, every spy on node/.append style={thick}}, every node/.style={inner sep=0,outer sep=0}]%
    \node[anchor=south west] (ori) at (0,0){\includegraphics[width=\textsynthtwidth]{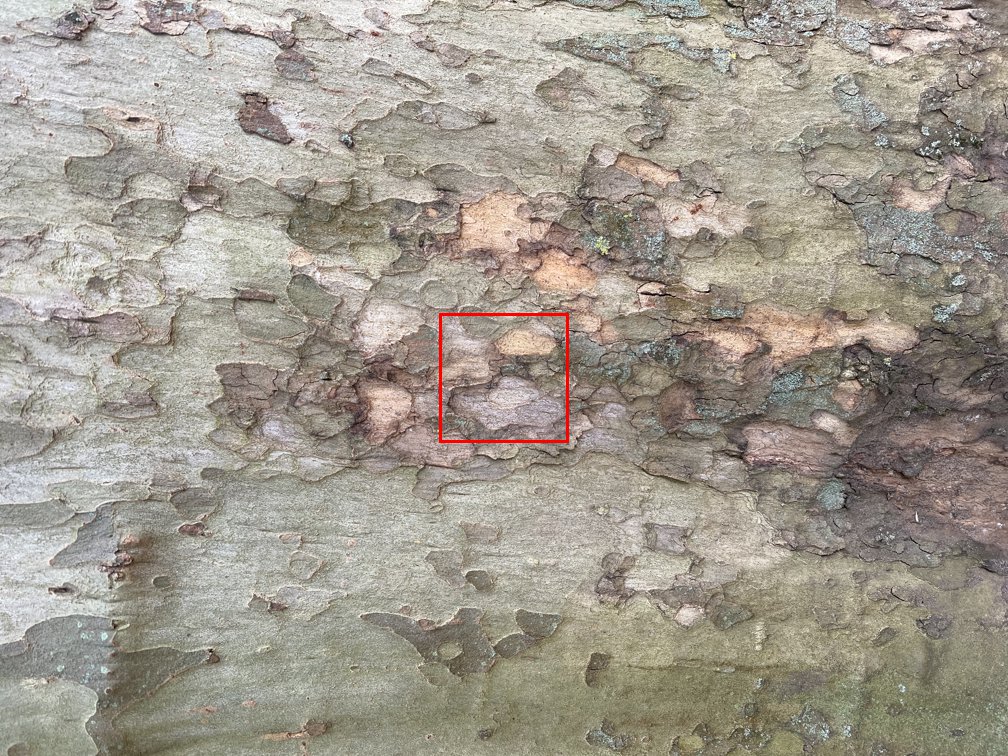}};
    \node[anchor=south west] (synth256) at
    (\textsynthtwidth+0.02\linewidth,0){\includegraphics[width=\textsynthtwidth]{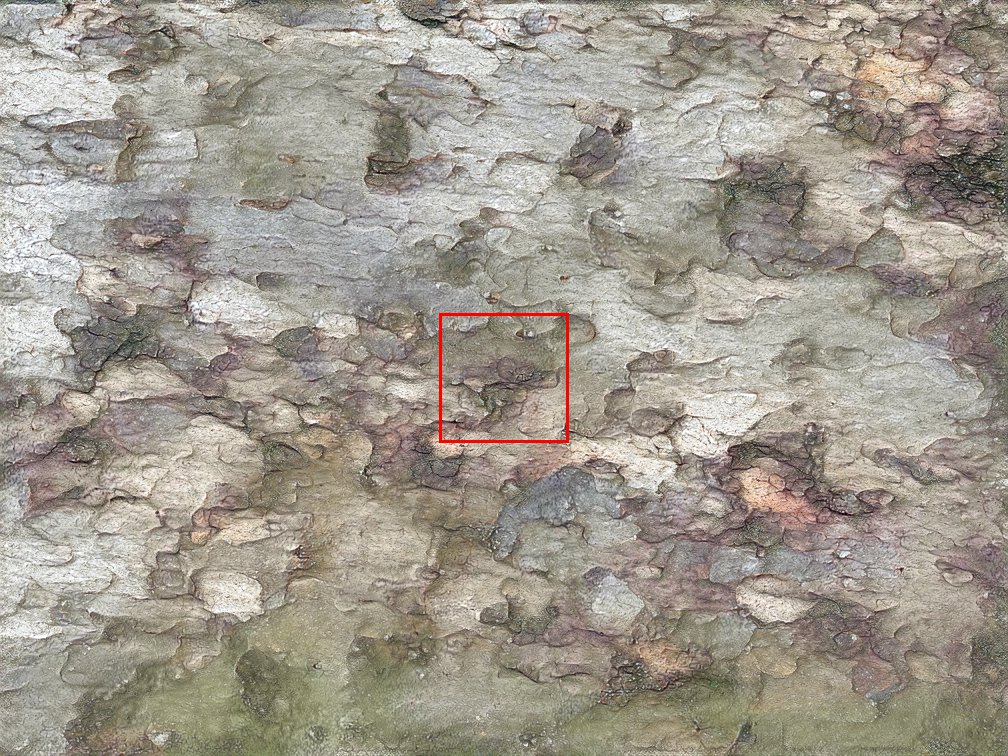}};
    \node[anchor=south west] (synth800) at
    (2*\textsynthtwidth+2*0.02\linewidth,0){\includegraphics[width=\textsynthtwidth]{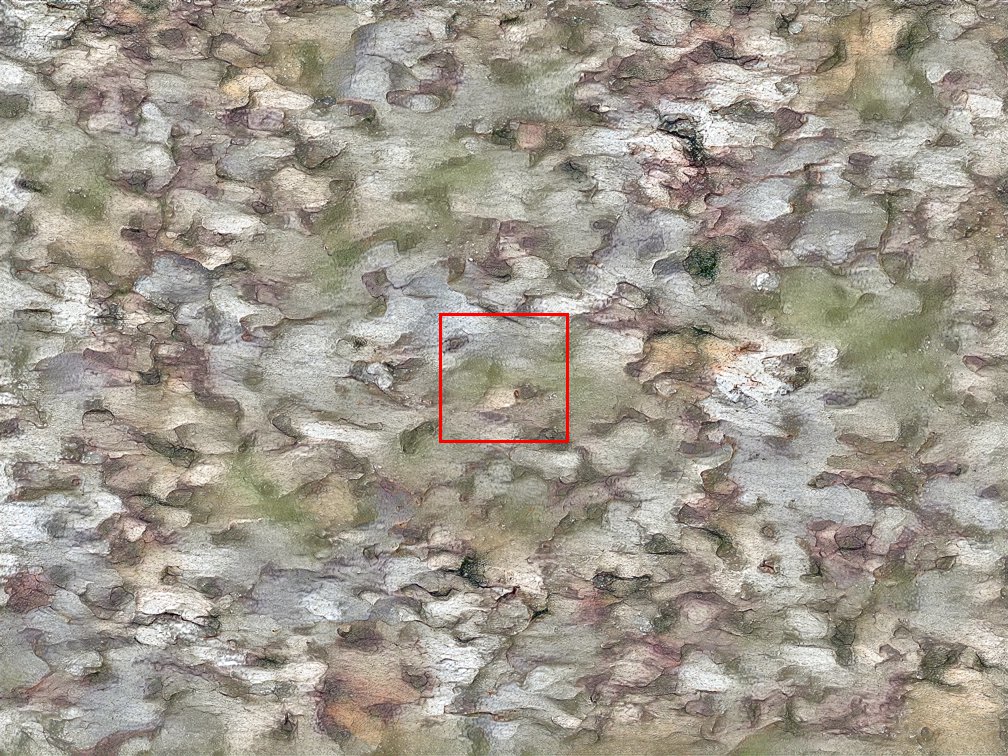}};
    \node[] (oritext) at (0.5*\textsynthtwidth,4.5) {{Input texture image}};
    \node[] (256res) at (1.5*\textsynthtwidth+0.02\linewidth,4.5) {Synthesis using 5 scales};
    \node[] (800res) at (2.5*\textsynthtwidth+2*0.02\linewidth,4.5) {Synthesis using 3 scales};
    \node[anchor = north west] at
    (0,-0.01\linewidth){\includegraphics[width=\textsynthtwidth]{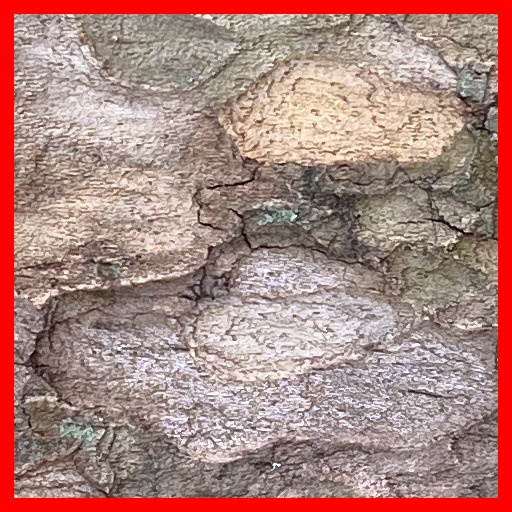}};
    \node[anchor = north west] at
    (\textsynthtwidth+0.02\linewidth,-0.01\linewidth){\includegraphics[width=\textsynthtwidth]{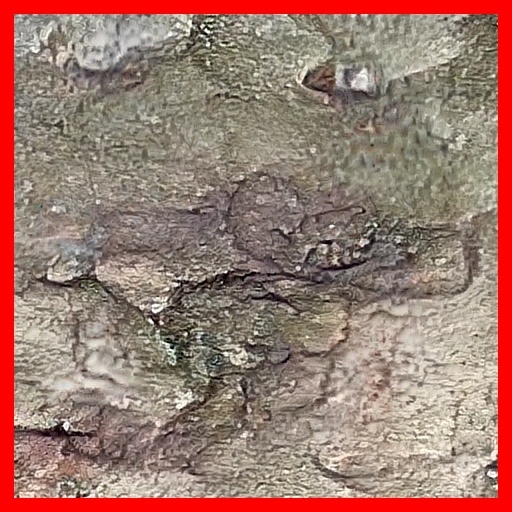}};
    \node[anchor = north west] at
    (2*\textsynthtwidth+2*0.02\linewidth,-0.01\linewidth){\includegraphics[width=\textsynthtwidth]{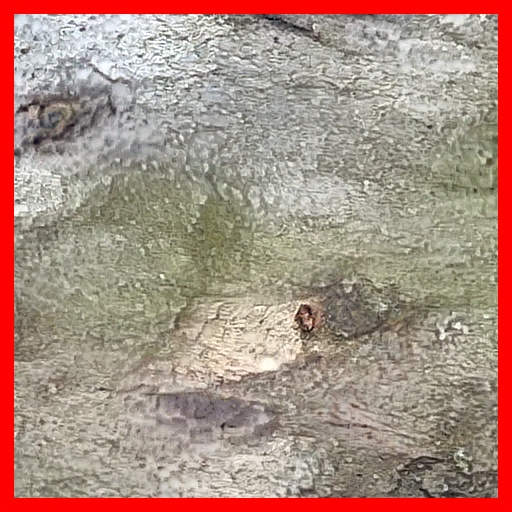}};
\end{tikzpicture}

\vspace{0.015\linewidth}

\begin{tikzpicture}[spy using outlines={rectangle, magnification=5,height=\resultwidth,width=\resultwidth, every spy on node/.append style={thick}}, every node/.style={inner sep=0,outer sep=0}]%
    \node[anchor=south west] (ori) at (0,0){\includegraphics[width=\textsynthtwidth]{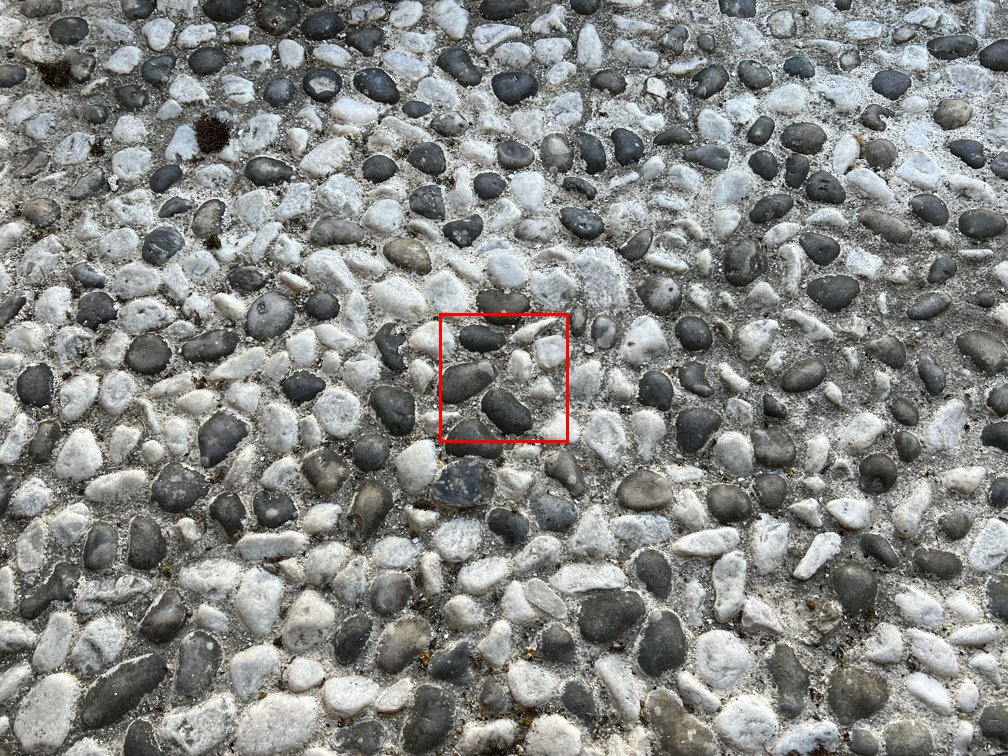}};
    \node[anchor=south west] (synth256) at
    (\textsynthtwidth+0.02\linewidth,0){\includegraphics[width=\textsynthtwidth]{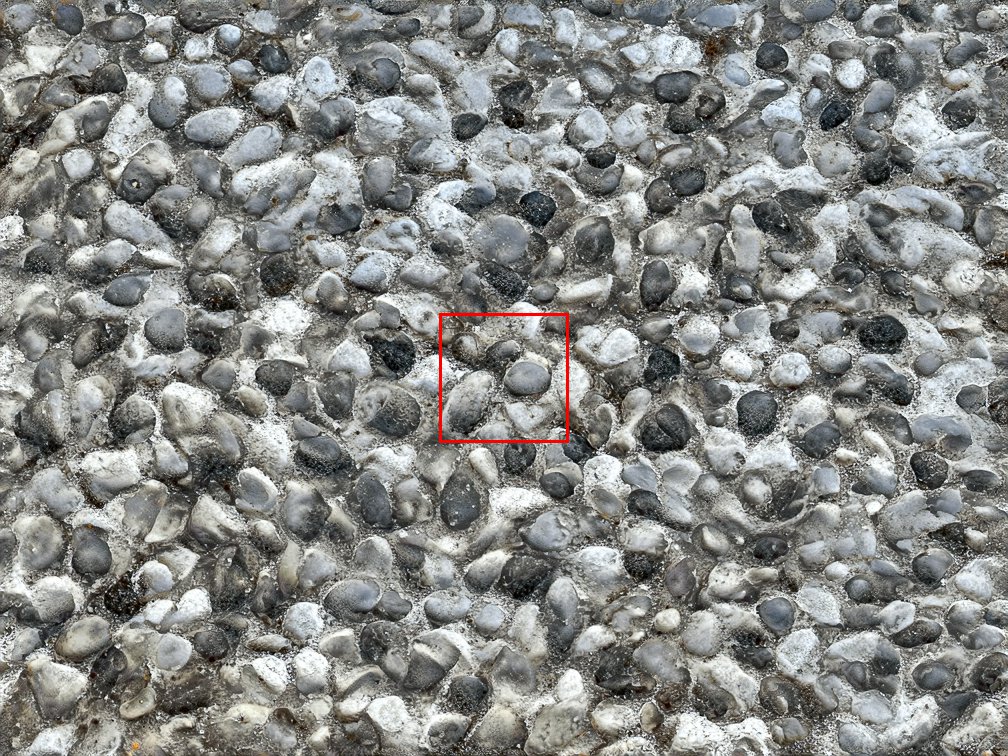}};
    \node[anchor=south west] (synth800) at
    (2*\textsynthtwidth+2*0.02\linewidth,0){\includegraphics[width=\textsynthtwidth]{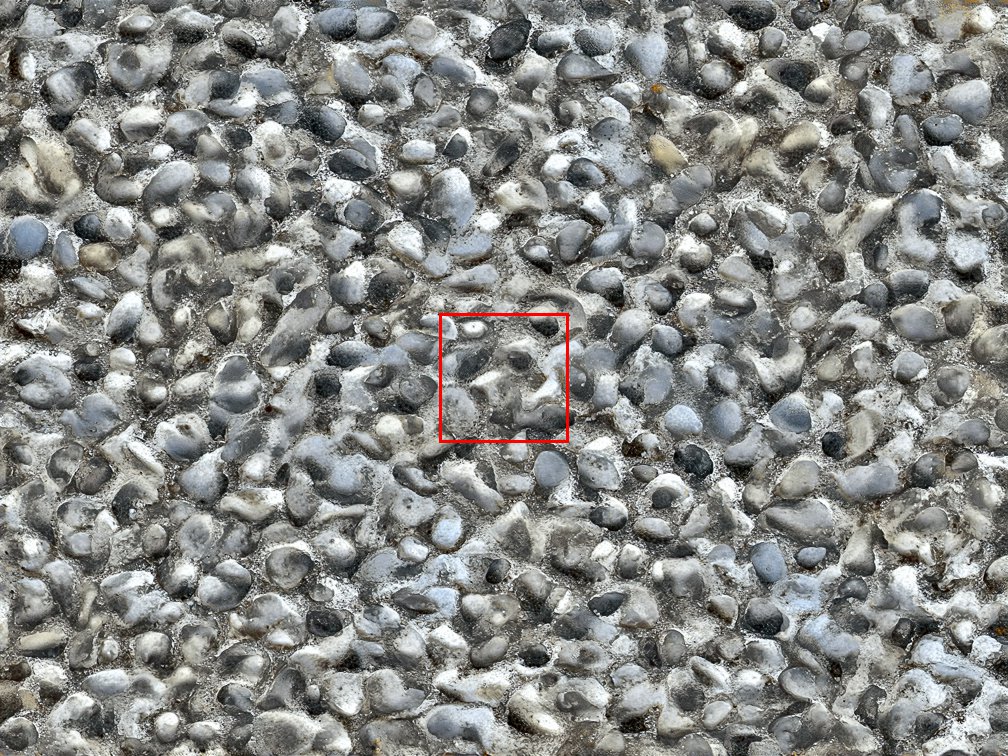}};
    \node[anchor = north west] at
    (0,-0.01\linewidth){\includegraphics[width=\textsynthtwidth]{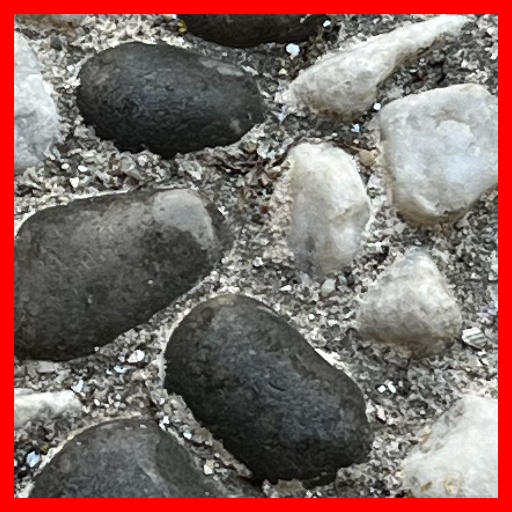}};
    \node[anchor = north west] at
    (\textsynthtwidth+0.02\linewidth,-0.01\linewidth){\includegraphics[width=\textsynthtwidth]{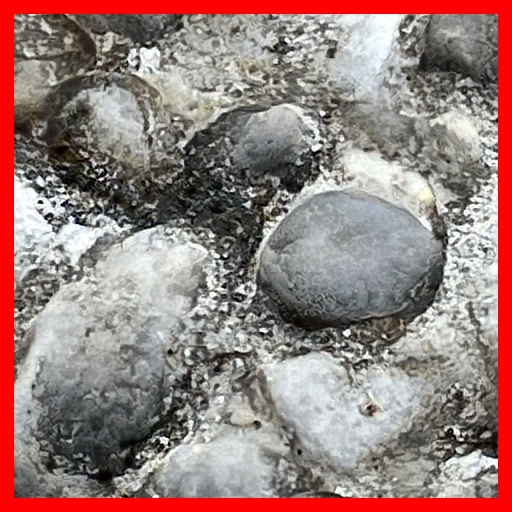}};
    \node[anchor = north west] at
    (2*\textsynthtwidth+2*0.02\linewidth,-0.01\linewidth){\includegraphics[width=\textsynthtwidth]{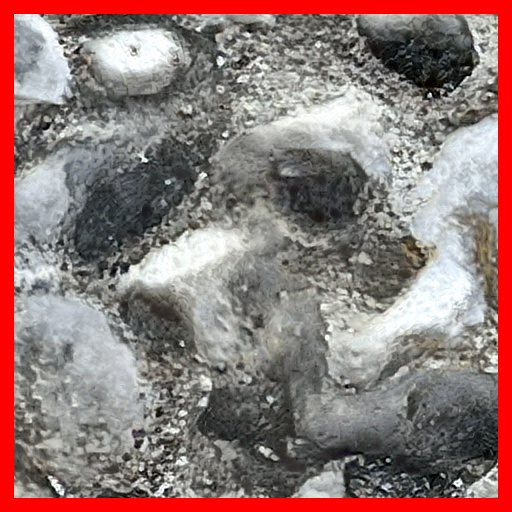}};
\end{tikzpicture}
\caption{UHR texture synthesis (same as Figure~\ref{fig:texture_synthesis_one}): From left to right: Input texture image, synthesis using 5 scales, synthesis using 3 scales. Image have size (3024$\times$4032) (downscaled by a factor 4 for inclusion in the .pdf) and true resolution details have size (512$\times$512).}
\label{fig:texture_synthesis_one}
\end{figure*}

\section{Comparison with very fast alternatives}
\label{sec:comparison}

\subsection{Visual comparison}
\label{subsec:comparison_with_fast_alternatives}
We compare our method with two fast alternatives for UHR style transfer, namely collaborative distillation (CD)~\cite{Wang_2020_CVPR} and URST~\cite{Chen_Wang_Xie_Lu_Luo_towards_ultra_resolution_neural_style_transfer_thumbnail_instance_normalization_AAAI2022} (based on~\cite{li2017universal}) using their official implementations. 
To improve readability of Figure~\ref{fig:style_transfer_comparison}, results for SPST-fast, which are really close to the ones of SPST but have slightly less details, are only reproduced in supp. mat.

As already discussed in Section~\ref{sec:related-work}, URST decreases the resolution of the style image to $1024^2$ px, so the style transfer is
not performed at the proper scale and fine details cannot be transferred (e.g., the algorithm is not aware of the brushstroke style).
As in UST methods, CD does not take into account details at different scales but simply proposes to reduce the number of filters in the auto-encoder network through collaborative distillation, to process larger images.
Unsurprisingly, one observes in Figure~\ref{fig:style_transfer_comparison} that our method is the only one capable of conveying the aspect of the painting strokes to the content image.
CD suffers from halos around objects (e.g., tress in the first example), saturated color, and high-frequency artifacts (see fourth column of Figure~\ref{fig:style_transfer_comparison}).
URST presents visible patch boundaries, a detail frequency mismatch due to improper scaling, loss of structure (e.g., buildings in the second example) and sometimes critical shrinking of the color palette (see fifth column of Figure~\ref{fig:style_transfer_comparison}).

All in all, even though CD and URST produce UHR images, one can argue that the effective resolution of the output does not match their size due to the many visual artifacts. In comparison, our iterative SPST algorithm produces images for which every image part is in accordance with UHR painting style, up to the pixel level.

\textcolor{coverletter}{Finally, let us observe that the style transfer results are in general better when the geometric content of the style image and the content image are close, regardless of the method (see Figure~\ref{fig:failure-cases} and supp. mat. for additional examples).
Gatys \emph{et al.}~\cite{Gatys_etal_Controlling_perceptual_factors_in_neural_style_transfer_CVPR2017} show that one can mitigate these failure cases by adding more control (e.g., segmentation, color transfer as preprocessing, careful rescaling of the style). All these solutions can be straightforwardly adapted to work on UHR images.}


\begin{figure*}[t]
\setlength{\gradbbheight}{0.1837\linewidth}
\begin{tabular}{@{}ccccc@{}}
Style image & Content image & SPST (ours) & CD &  URST \\
\includegraphics[width=0.1\textwidth]{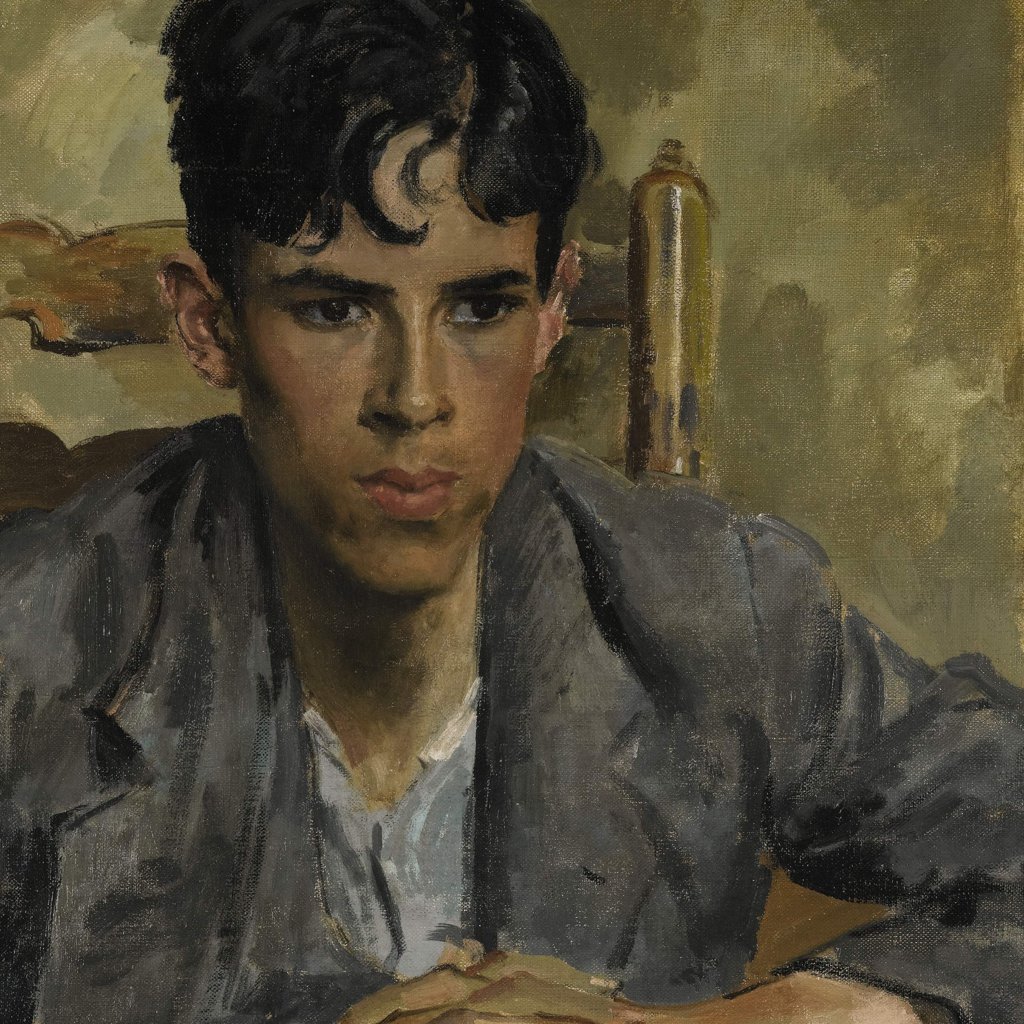}
&
\includegraphics[width=0.2\textwidth]{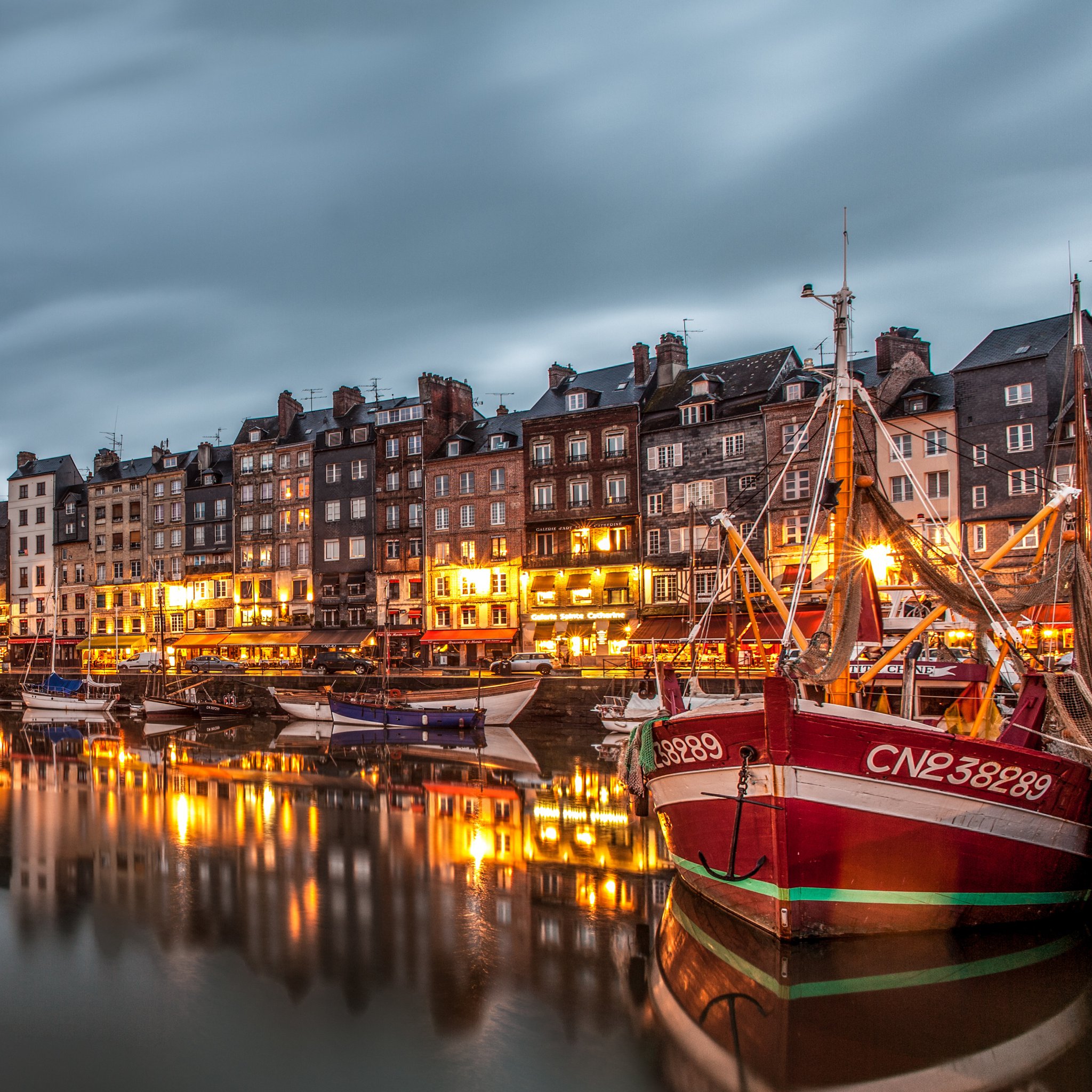}
&
\includegraphics[width=0.2\textwidth]{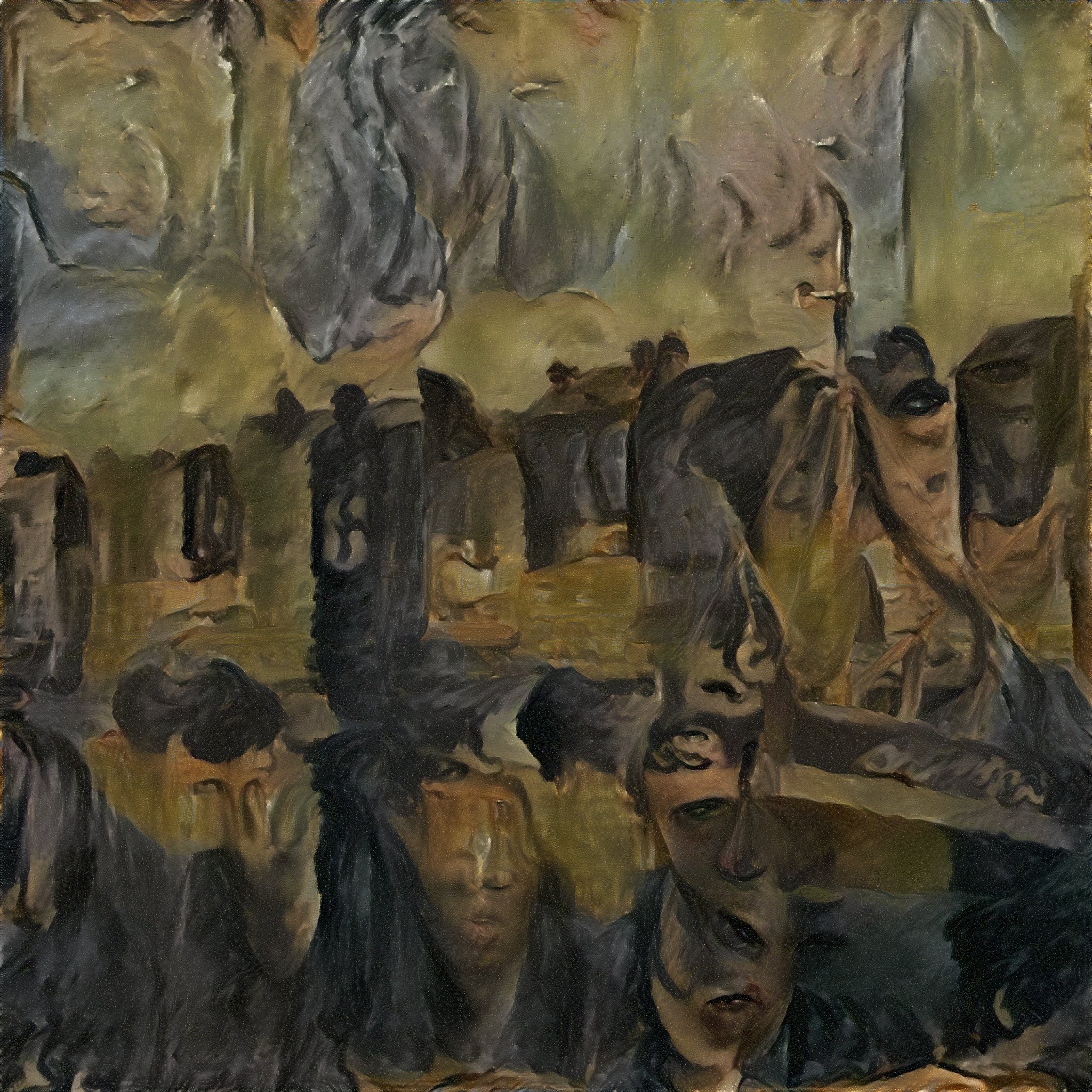}
&
\includegraphics[width=0.2\textwidth]{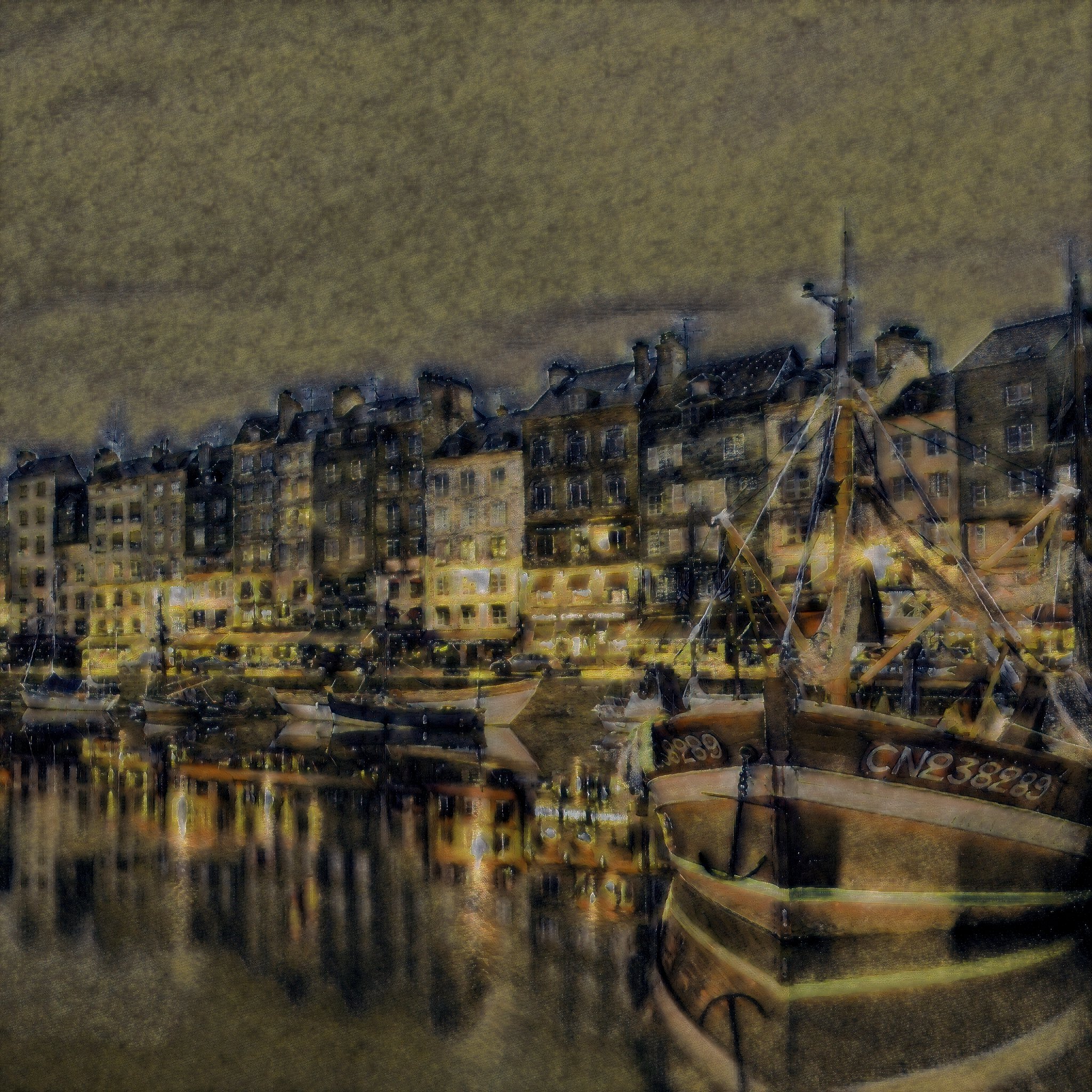}
&
\includegraphics[width=0.2\textwidth]{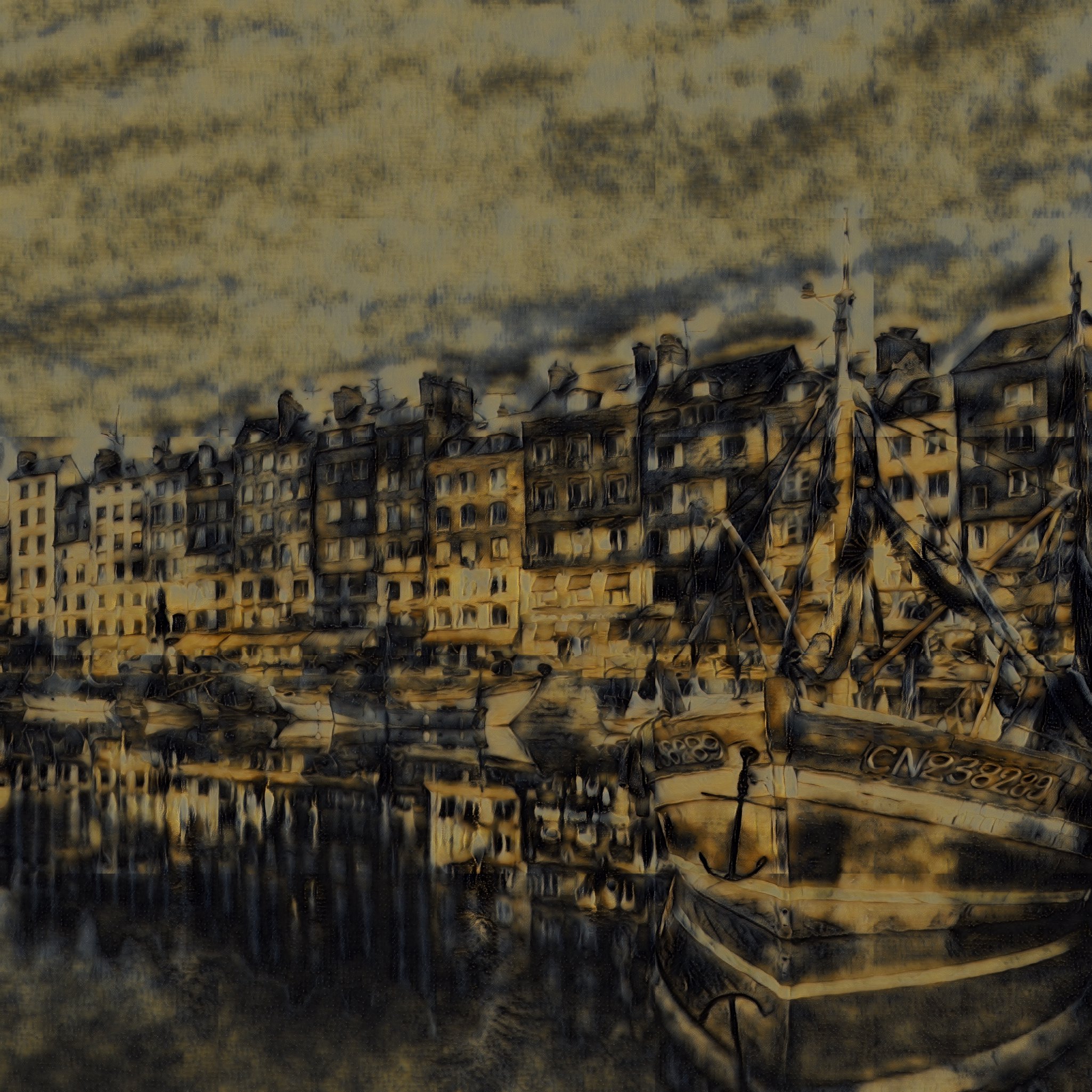}
\end{tabular}
\caption{\textcolor{coverletter}{Example of content-style mismatch. From left to right: style (2048$\times$2048), content (4096$\times$4096), our results (SPST),
CD~\cite{Wang_2020_CVPR} and URST~\cite{Chen_Wang_Xie_Lu_Luo_towards_ultra_resolution_neural_style_transfer_thumbnail_instance_normalization_AAAI2022}. We used four scale for our results.
The content-style mismatch results in the loss of details in the buildings and does not preserve the homogeneity of the sky.
}}
\label{fig:failure-cases}
\end{figure*}

\begin{figure*}
\setlength{\resultwidth}{0.1837\linewidth}
\begin{tikzpicture}[spy using outlines={rectangle, magnification=5,height=\resultwidth,width=\resultwidth, every spy on node/.append style={thick}}, every node/.style={inner sep=0,outer sep=0}]%
    \node[anchor=south west] (style) at (0,0){\includegraphics[width= 0.6683501683501684\resultwidth]{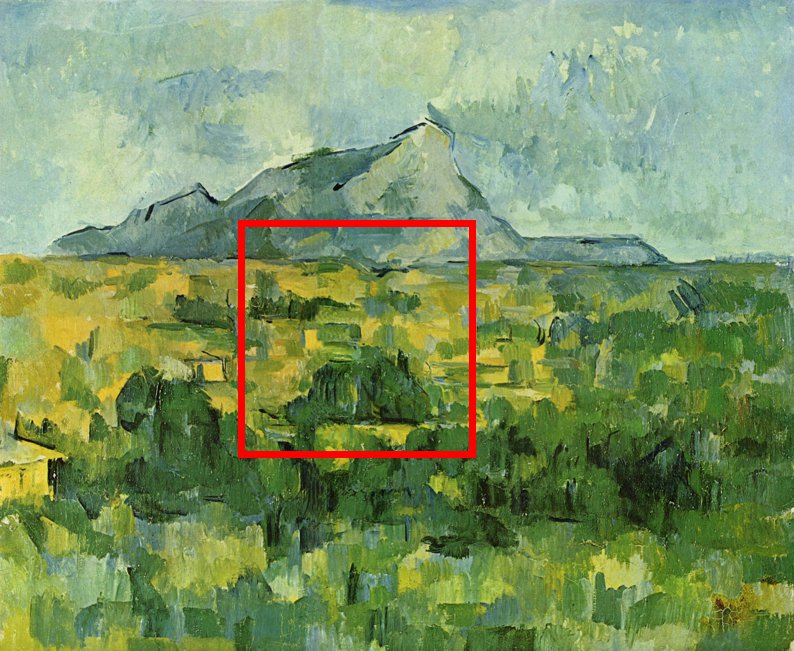}};
    \node[anchor=south west] (content) at
    (\resultwidth+0.02\linewidth,0){\includegraphics[width=\resultwidth]{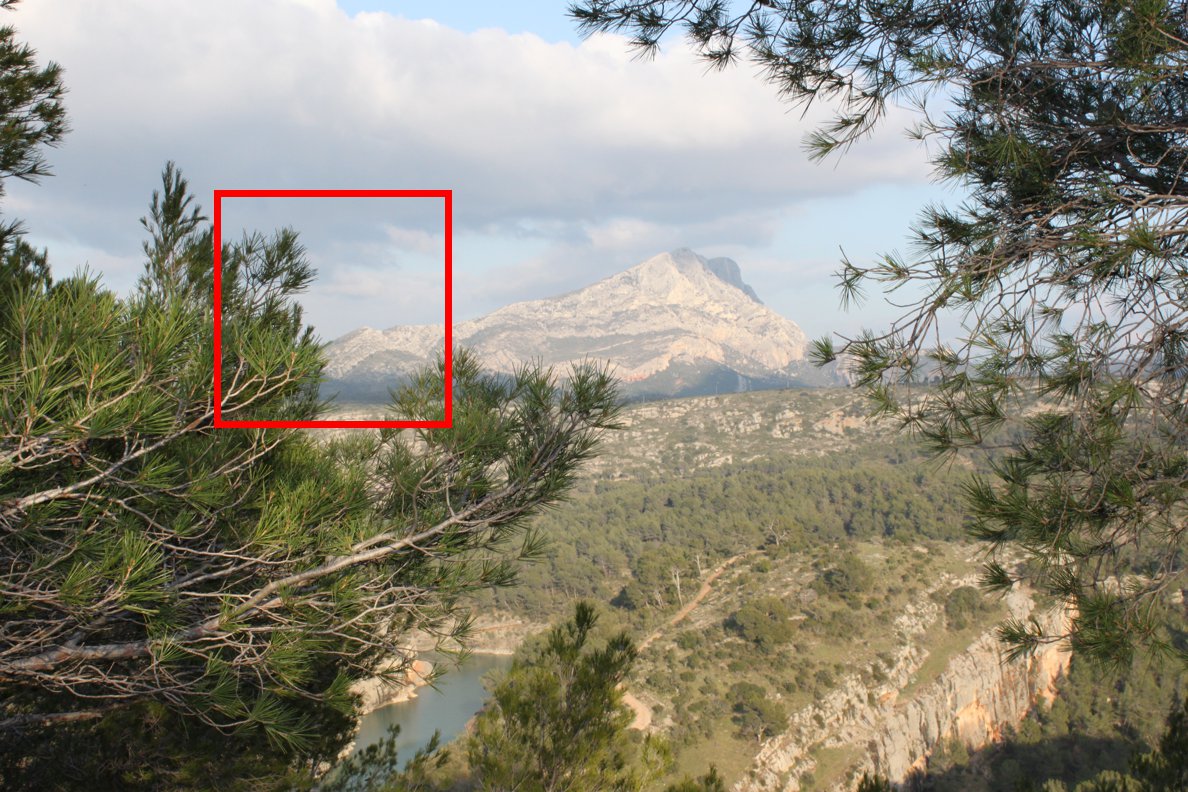}};;
    \node[anchor=south west]  (ours) at
    (2*\resultwidth+2*0.02\linewidth,0){\includegraphics[width=\resultwidth]{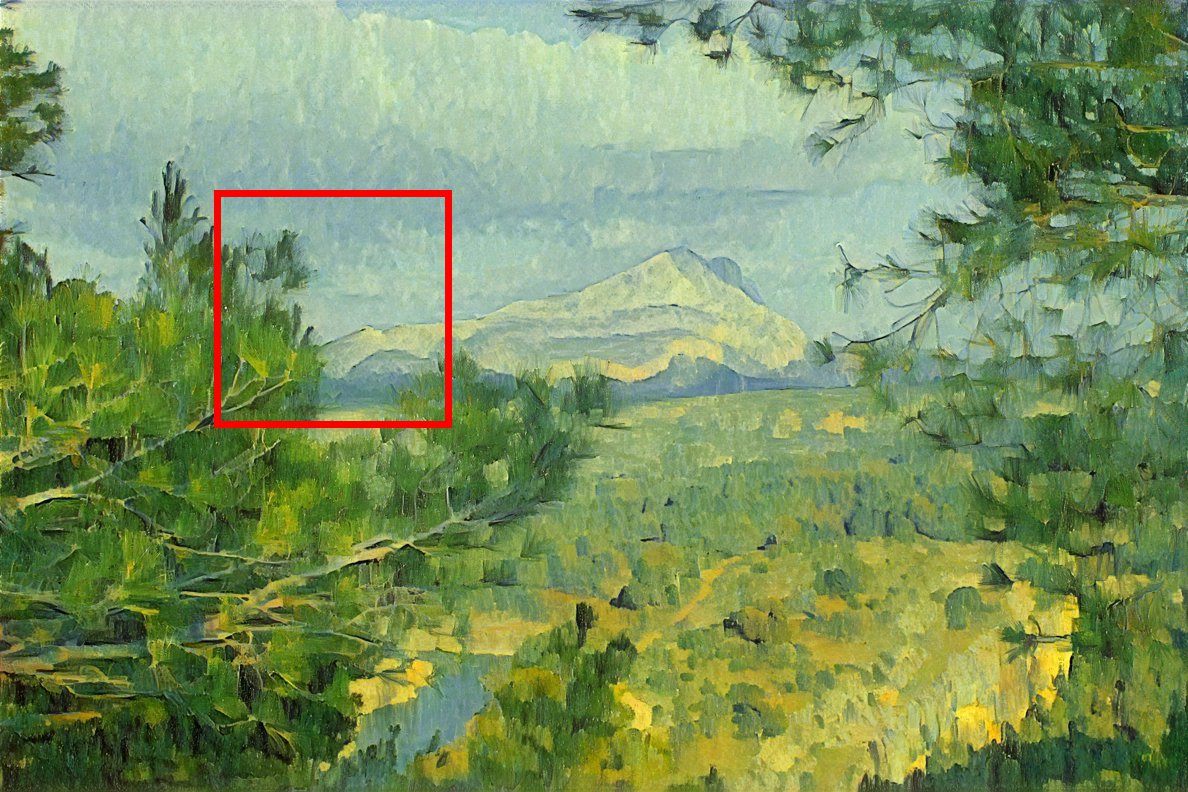}};
    \node[anchor=south west] (collab) at
    (3*\resultwidth+3*0.02\linewidth,0){\includegraphics[width=\resultwidth]{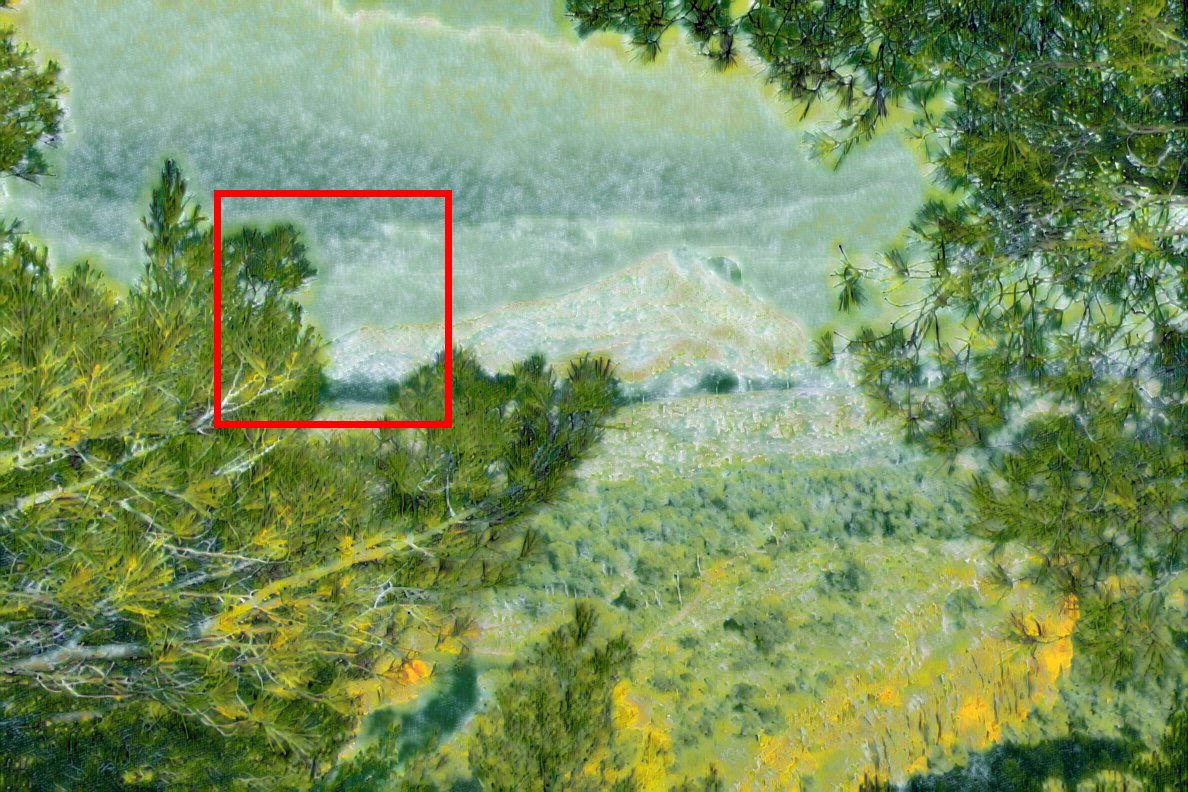}};
    \node[anchor=south west] (urst) at
    (4*\resultwidth+4*0.02\linewidth,0){\includegraphics[width=\resultwidth]{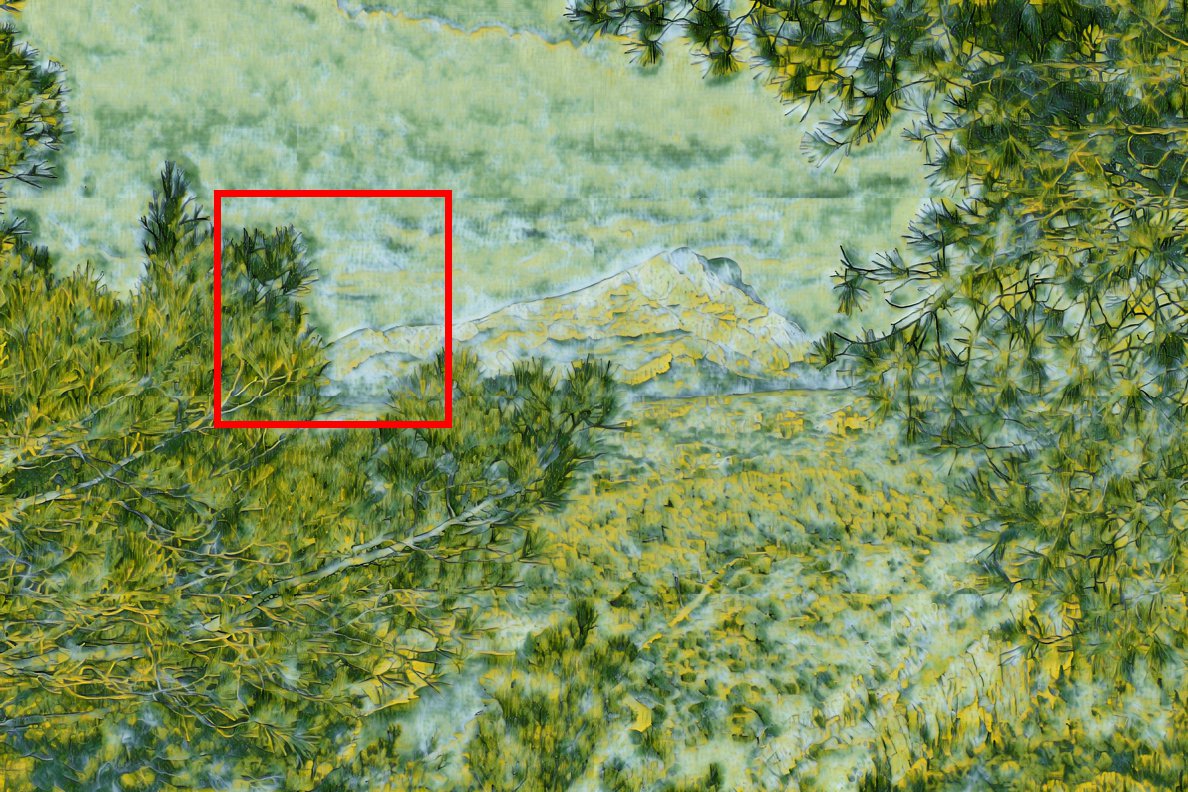}};
    \node[] (styleLabel) at (0.5*\resultwidth,2.35) {{Style image}};
    \node[] (contentLabel) at (1.5*\resultwidth+0.02\linewidth,2.35) {{Content image}};
    \node[] (spstLabel) at (2.5*\resultwidth+2*0.02\linewidth,2.35) {{SPST (ours)}};
    \node[] (cdLabel) at (3.5*\resultwidth+3*0.02\linewidth,2.35) {{CD}};
    \node[] (urstLabel) at (4.5*\resultwidth+4*0.02\linewidth,2.35) {{URST}};
    \node[anchor = north west] at
    (0,-0.02\linewidth) {\includegraphics[width=\resultwidth]{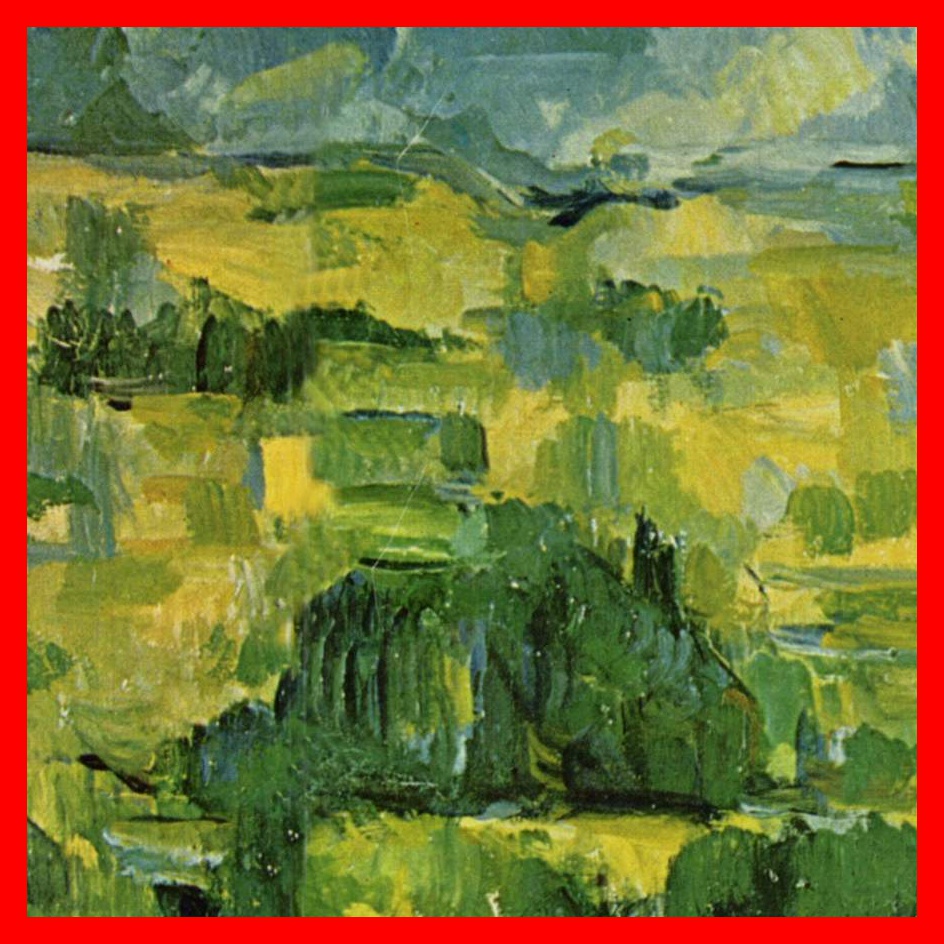}};
    \node[anchor = north west] at
    (\resultwidth+0.02\linewidth,-0.02\linewidth) {\includegraphics[width=\resultwidth]{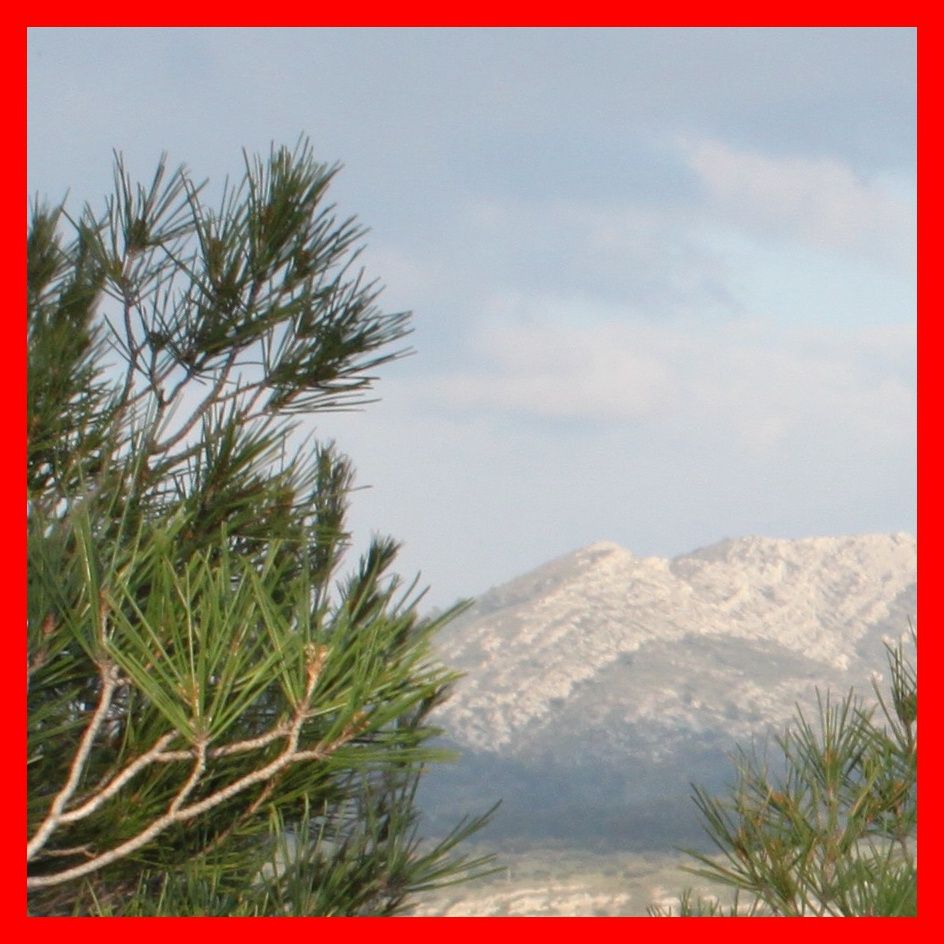}};
    \node[anchor = north west] at
    (2*\resultwidth+2*0.02\linewidth,-0.02\linewidth) {\includegraphics[width=\resultwidth]{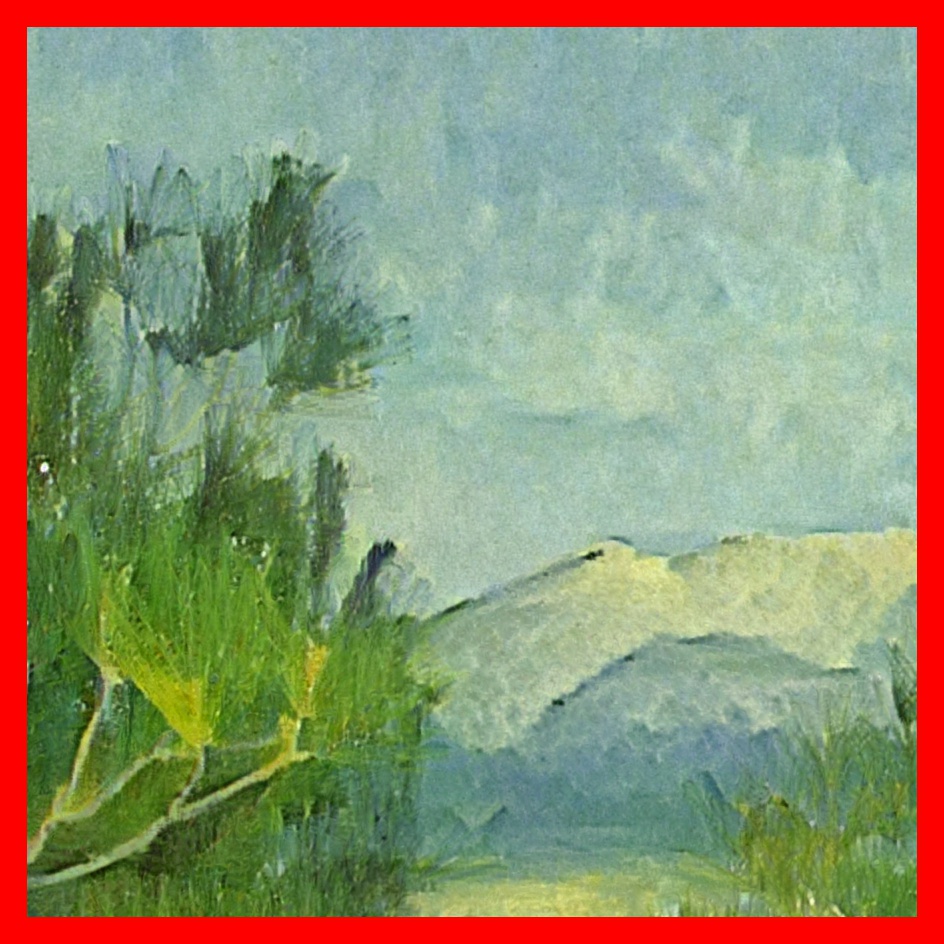}};
    \node[anchor = north west] at
    (3*\resultwidth+3*0.02\linewidth,-0.02\linewidth){\includegraphics[width=\resultwidth]{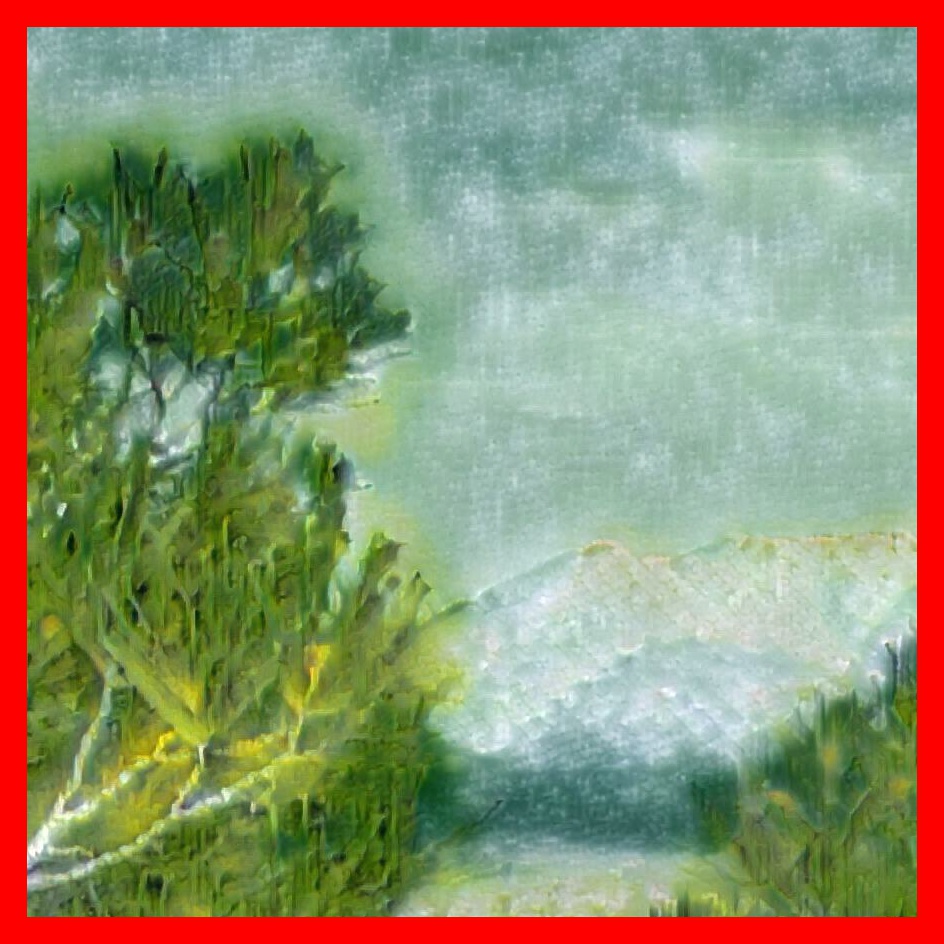}};
    \node[anchor = north west] at
    (4*\resultwidth+4*0.02\linewidth,-0.02\linewidth){\includegraphics[width=\resultwidth]{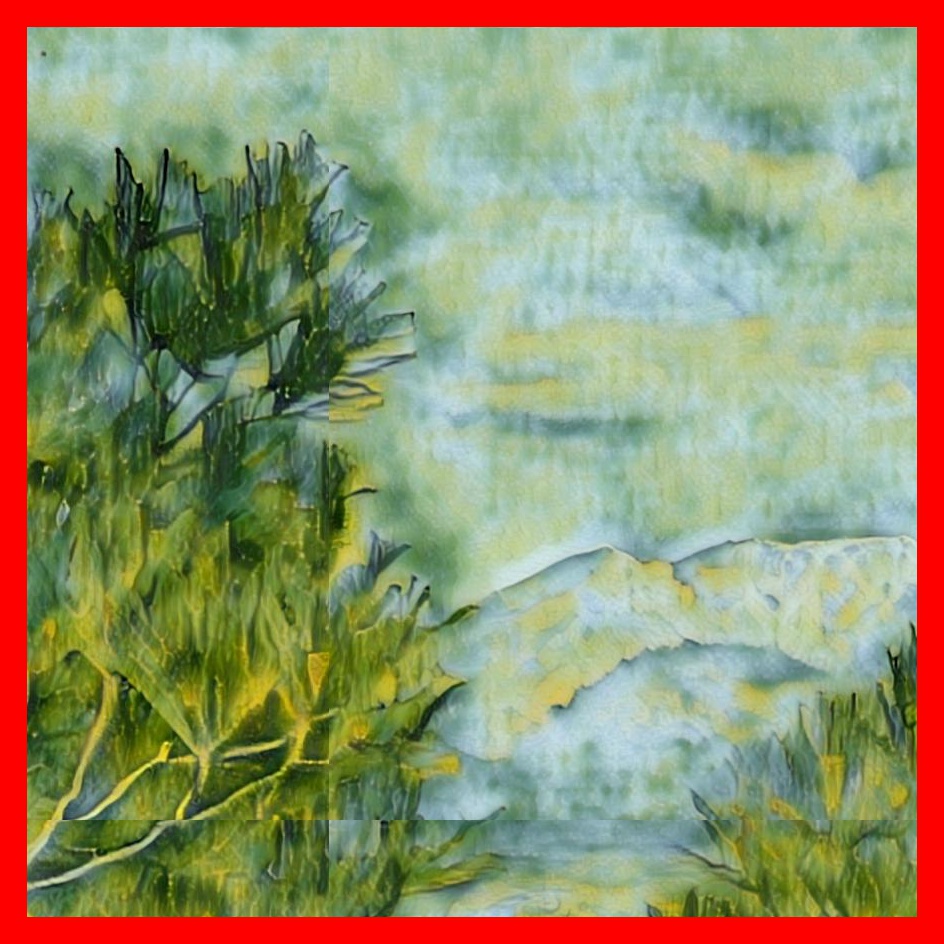}};
\end{tikzpicture}

\vspace{0.015\linewidth}

\begin{tikzpicture}[spy using outlines={rectangle, magnification=5,height=\resultwidth,width=\resultwidth, every spy on node/.append style={thick}}, every node/.style={inner sep=0,outer sep=0}]%
    \node[anchor=south west] (style) at (0,0){\includegraphics[width=0.9392361111111112\resultwidth]{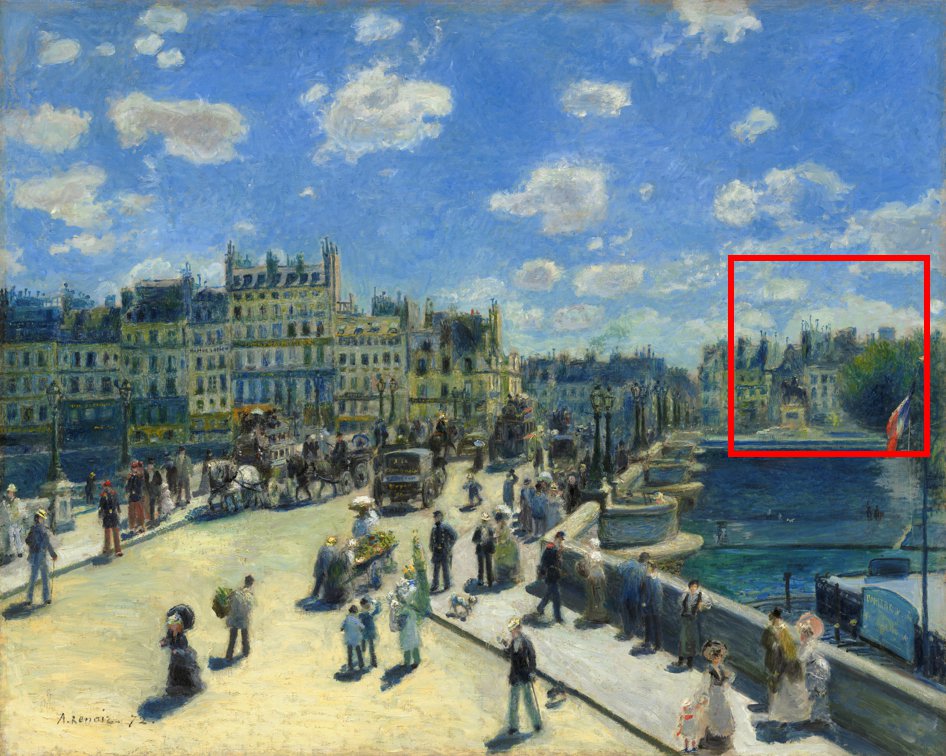}};
    \node[anchor=south west] (content) at
    (\resultwidth+0.02\linewidth,0){\includegraphics[width=\resultwidth]{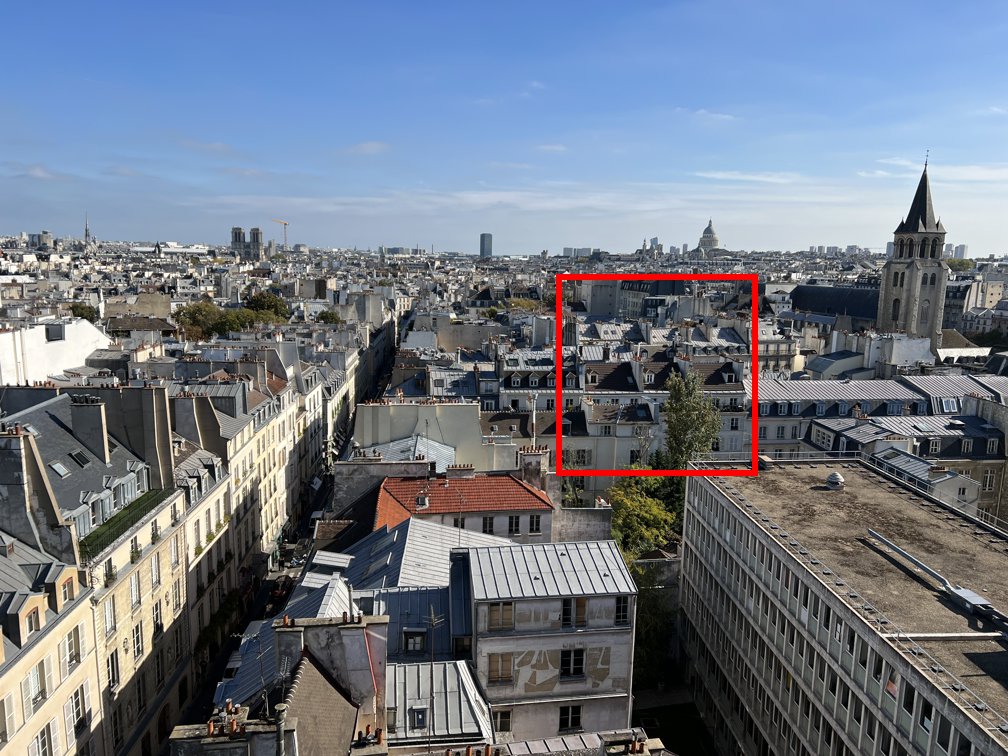}};;
    \node[anchor=south west]  (ours) at
    (2*\resultwidth+2*0.02\linewidth,0){\includegraphics[width=\resultwidth]{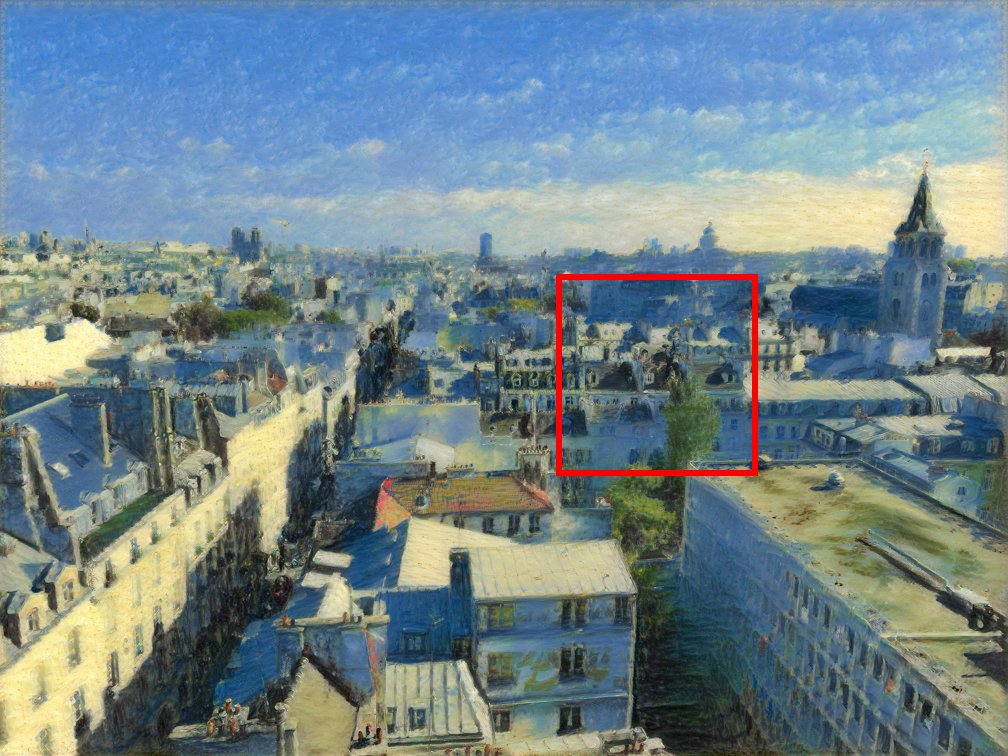}};
    \node[anchor=south west] (collab) at
    (3*\resultwidth+3*0.02\linewidth,0){\includegraphics[width=\resultwidth]{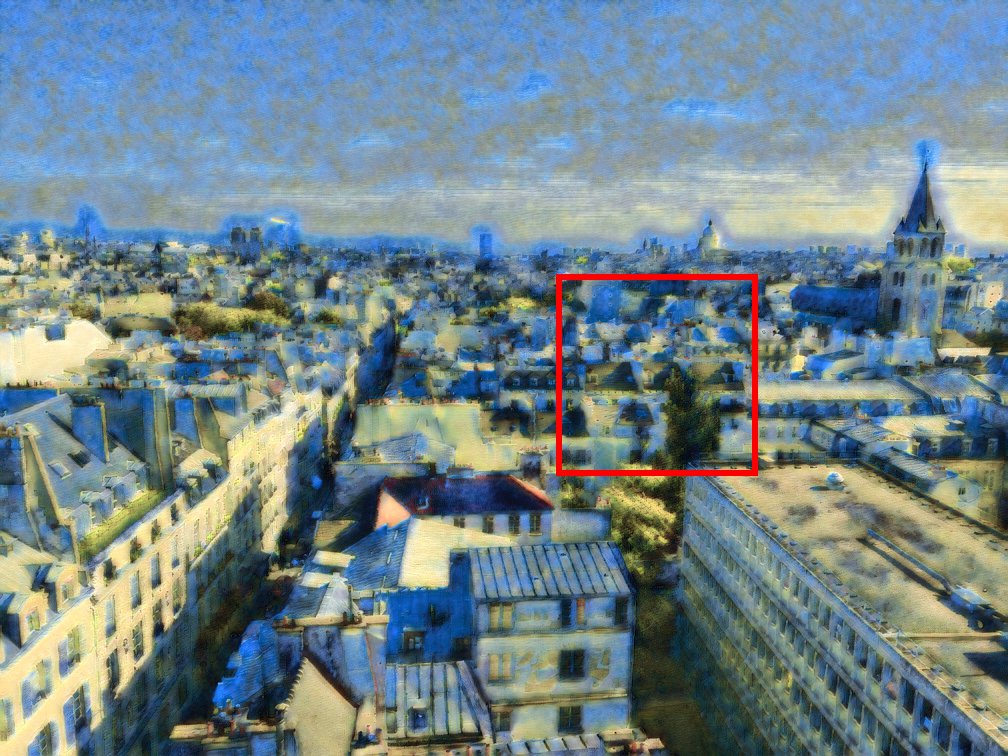}};
    \node[anchor=south west] (urst) at
    (4*\resultwidth+4*0.02\linewidth,0){\includegraphics[width=\resultwidth]{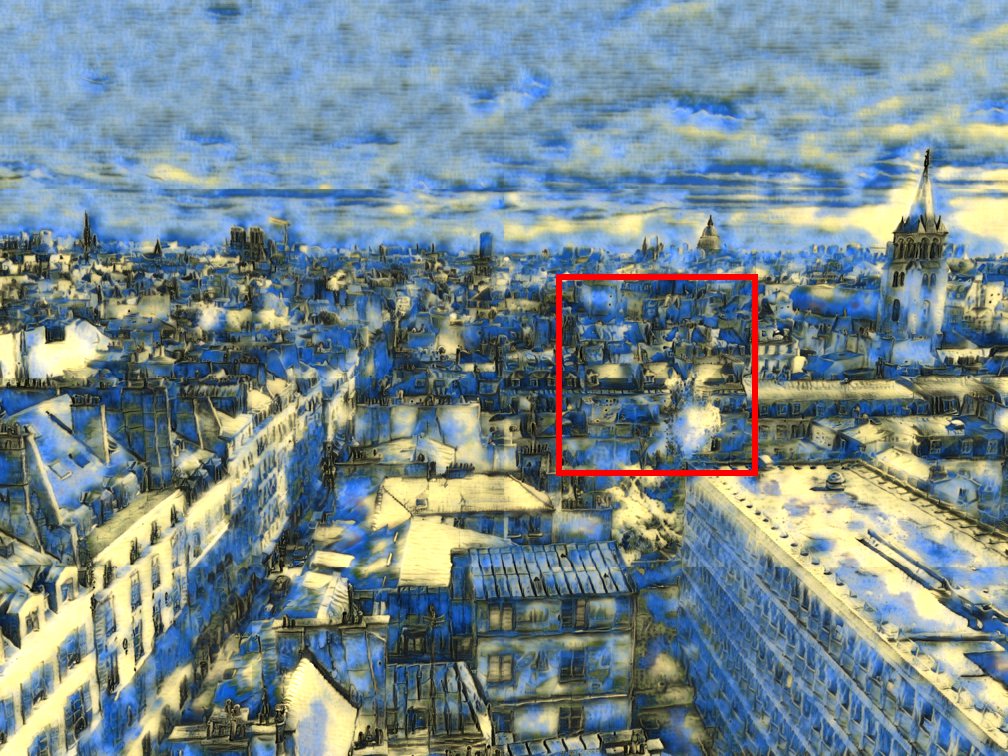}};
    \node[anchor = north west] at
    (0,-0.02\linewidth){\includegraphics[width=\resultwidth]{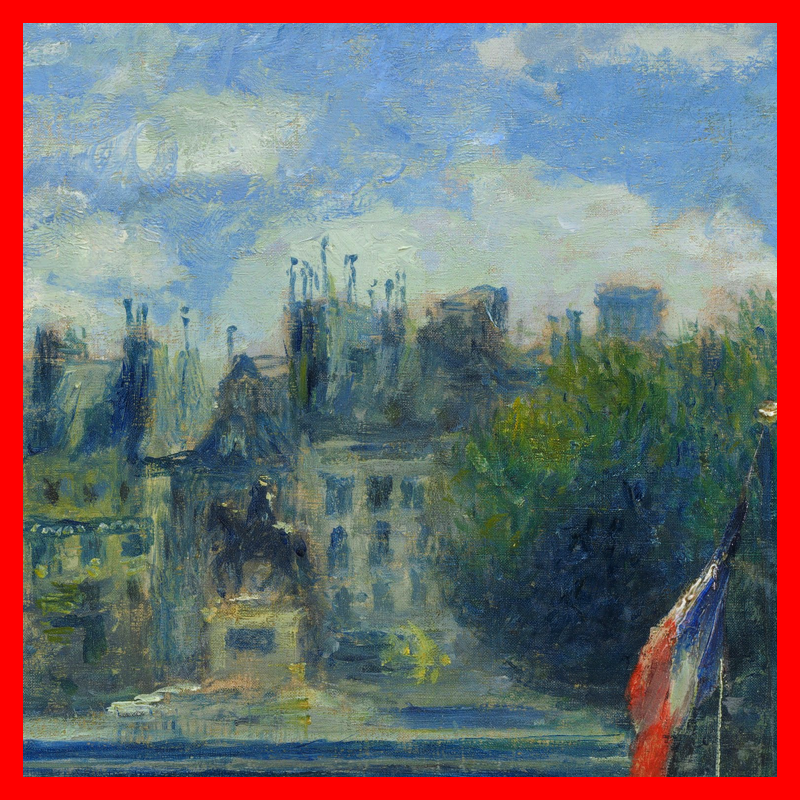}};
    \node[anchor = north west] at
    (\resultwidth+0.02\linewidth,-0.02\linewidth){\includegraphics[width=\resultwidth]{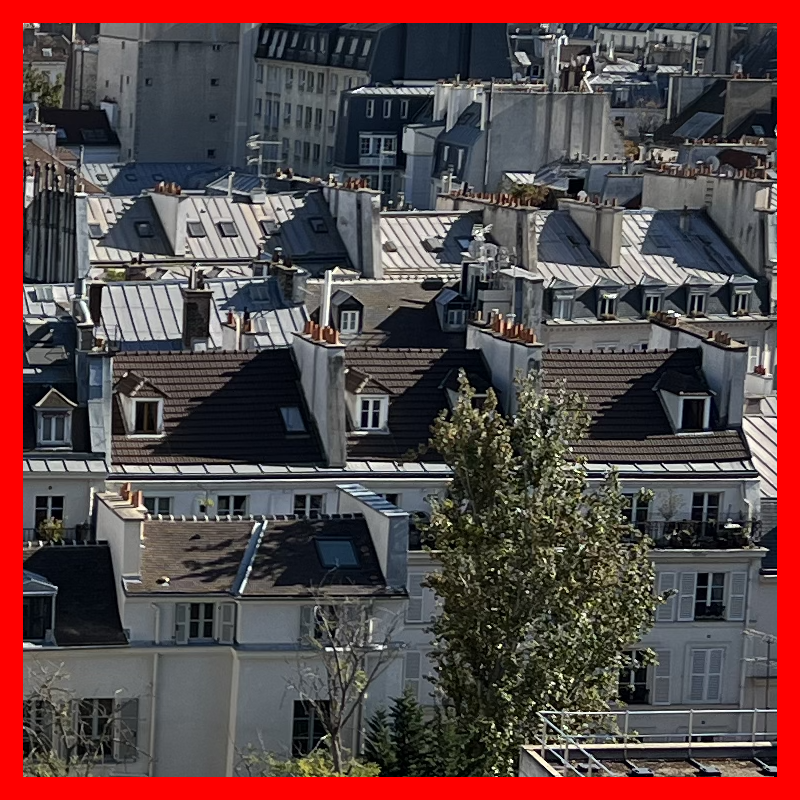}};
    \node[anchor = north west] at
    (2*\resultwidth+2*0.02\linewidth,-0.02\linewidth){\includegraphics[width=\resultwidth]{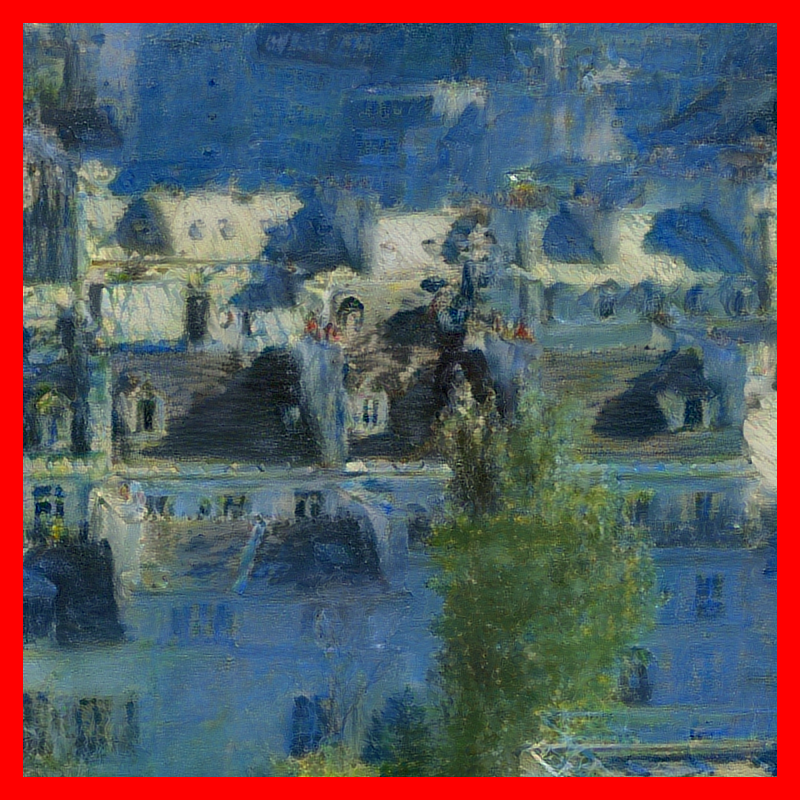}};
    \node[anchor = north west] at
    (3*\resultwidth+3*0.02\linewidth,-0.02\linewidth){\includegraphics[width=\resultwidth]{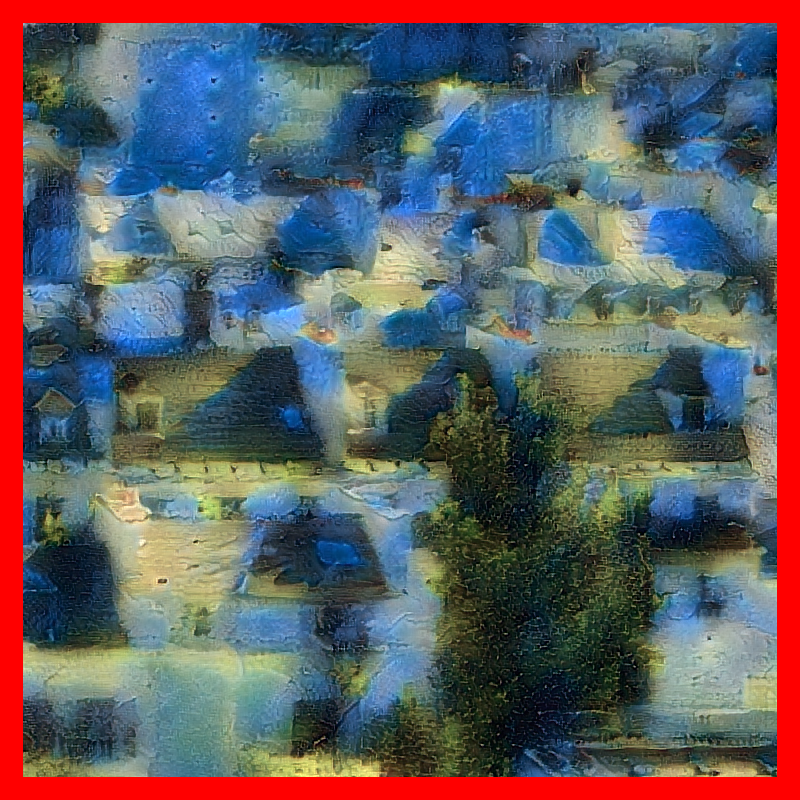}};
    \node[anchor = north west] at
    (4*\resultwidth+4*0.02\linewidth,-0.02\linewidth){\includegraphics[width=\resultwidth]{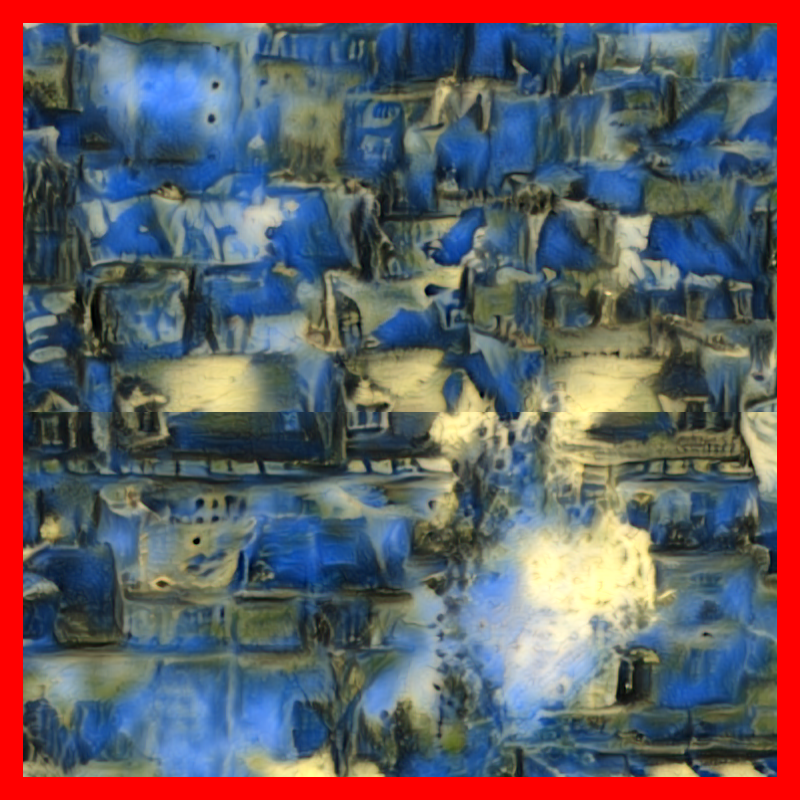}};
\end{tikzpicture}

\vspace{0.015\linewidth}

\begin{tikzpicture}[spy using outlines={rectangle, magnification=8,height=\resultwidth,width=\resultwidth, every spy on node/.append style={thick}}, every node/.style={inner sep=0,outer sep=0}]%
    \node[anchor=south west] (style) at (0,0){\includegraphics[width=0.723404255319149\resultwidth]{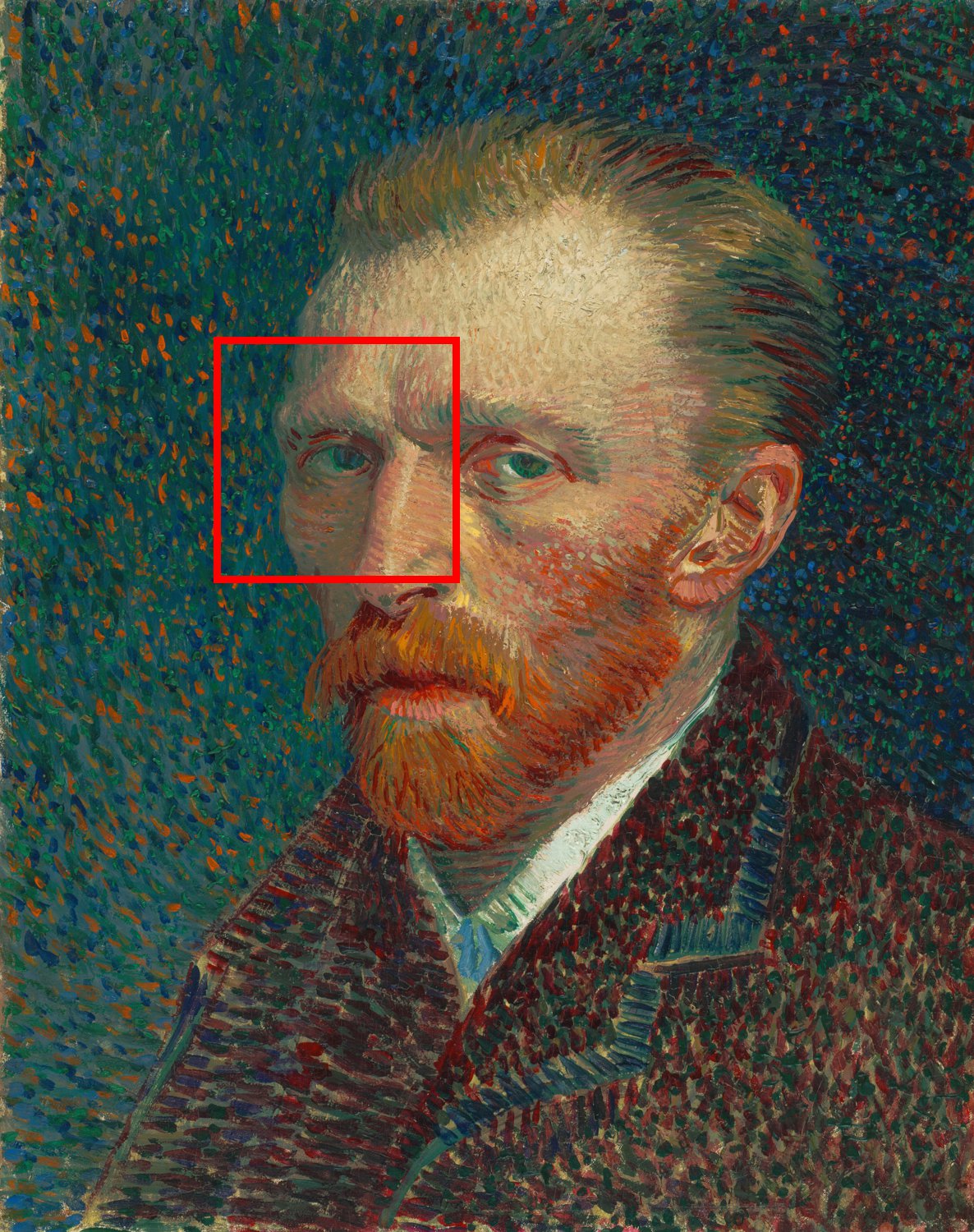}};
    \node[anchor=south west] (content) at
    (\resultwidth+0.02\linewidth,0){\includegraphics[width=\resultwidth]{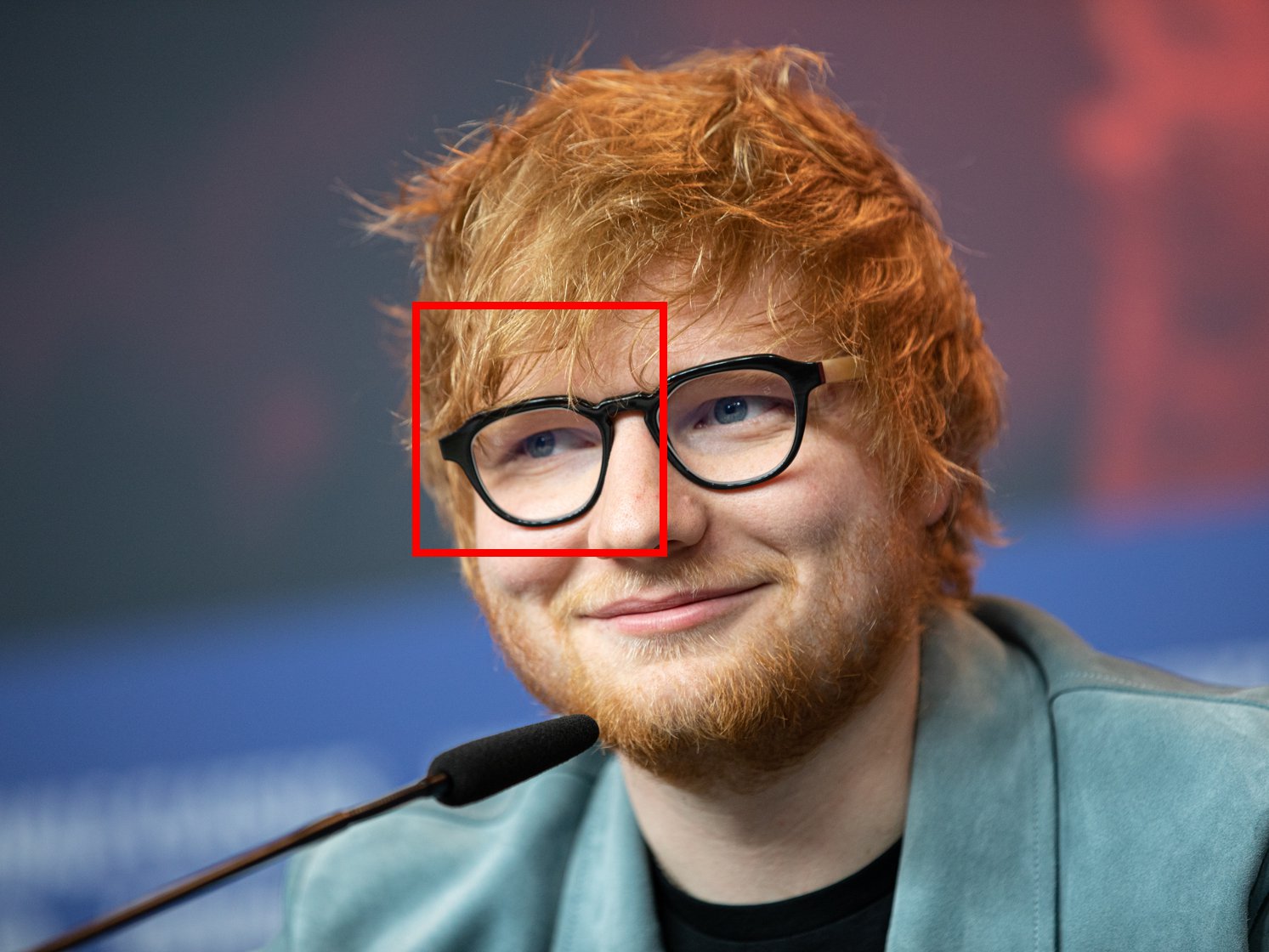}};;
    \node[anchor=south west]  (ours) at
    (2*\resultwidth+2*0.02\linewidth,0){\includegraphics[width=\resultwidth]{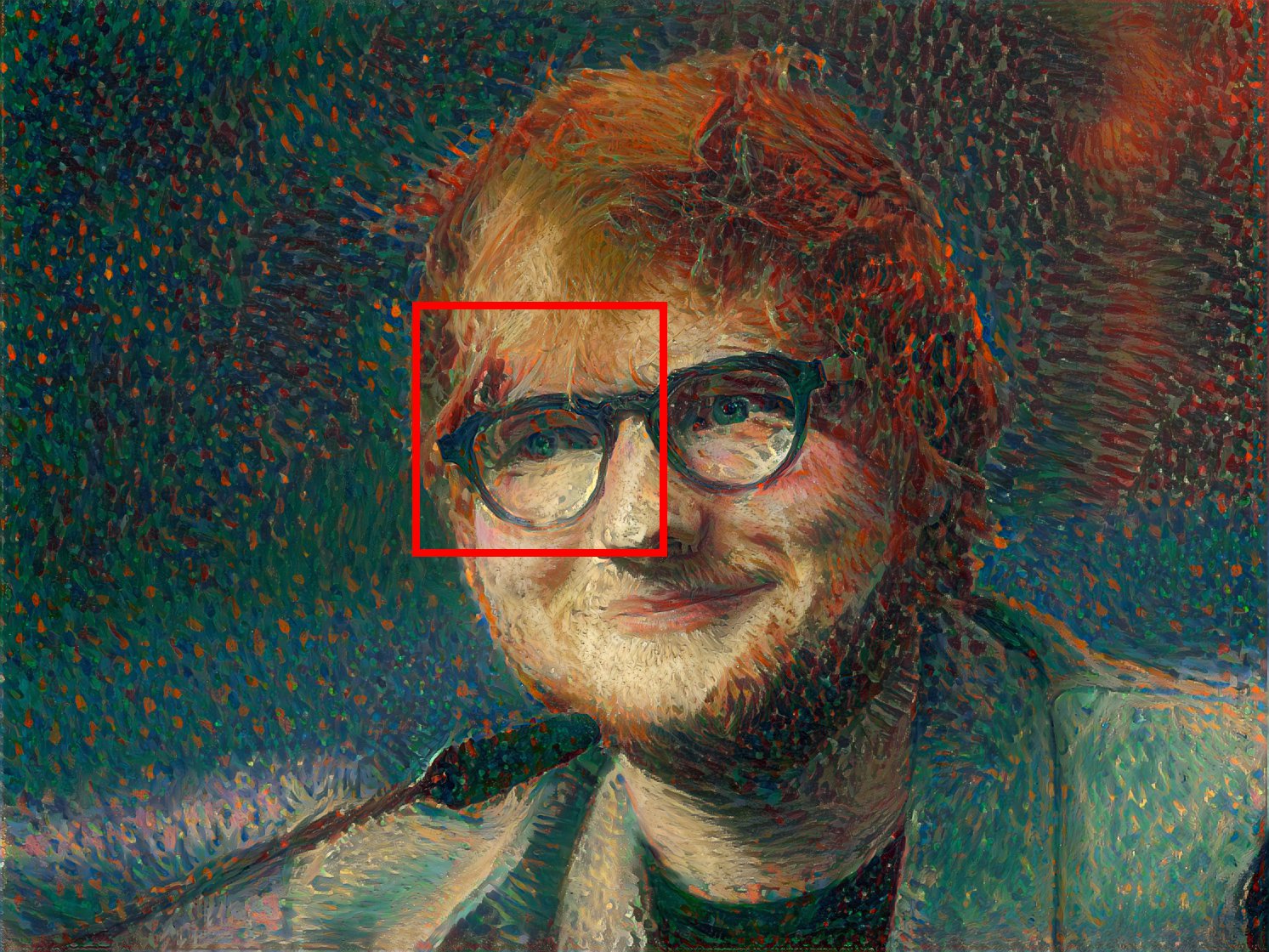}};
    \node[anchor=south west] (collab) at
    (3*\resultwidth+3*0.02\linewidth,0){\includegraphics[width=\resultwidth]{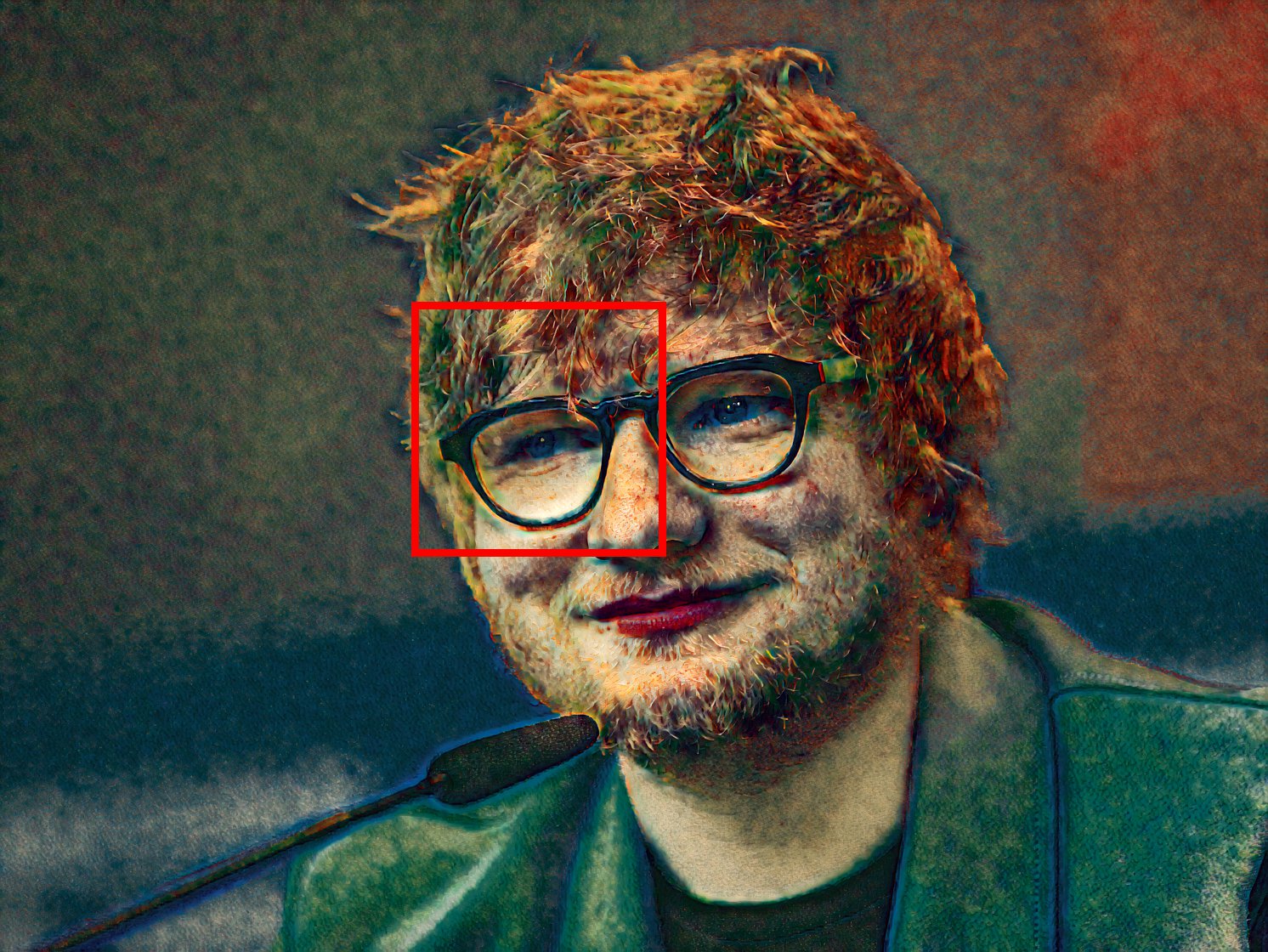}};
    \node[anchor=south west] (urst) at
    (4*\resultwidth+4*0.02\linewidth,0){\includegraphics[width=\resultwidth]{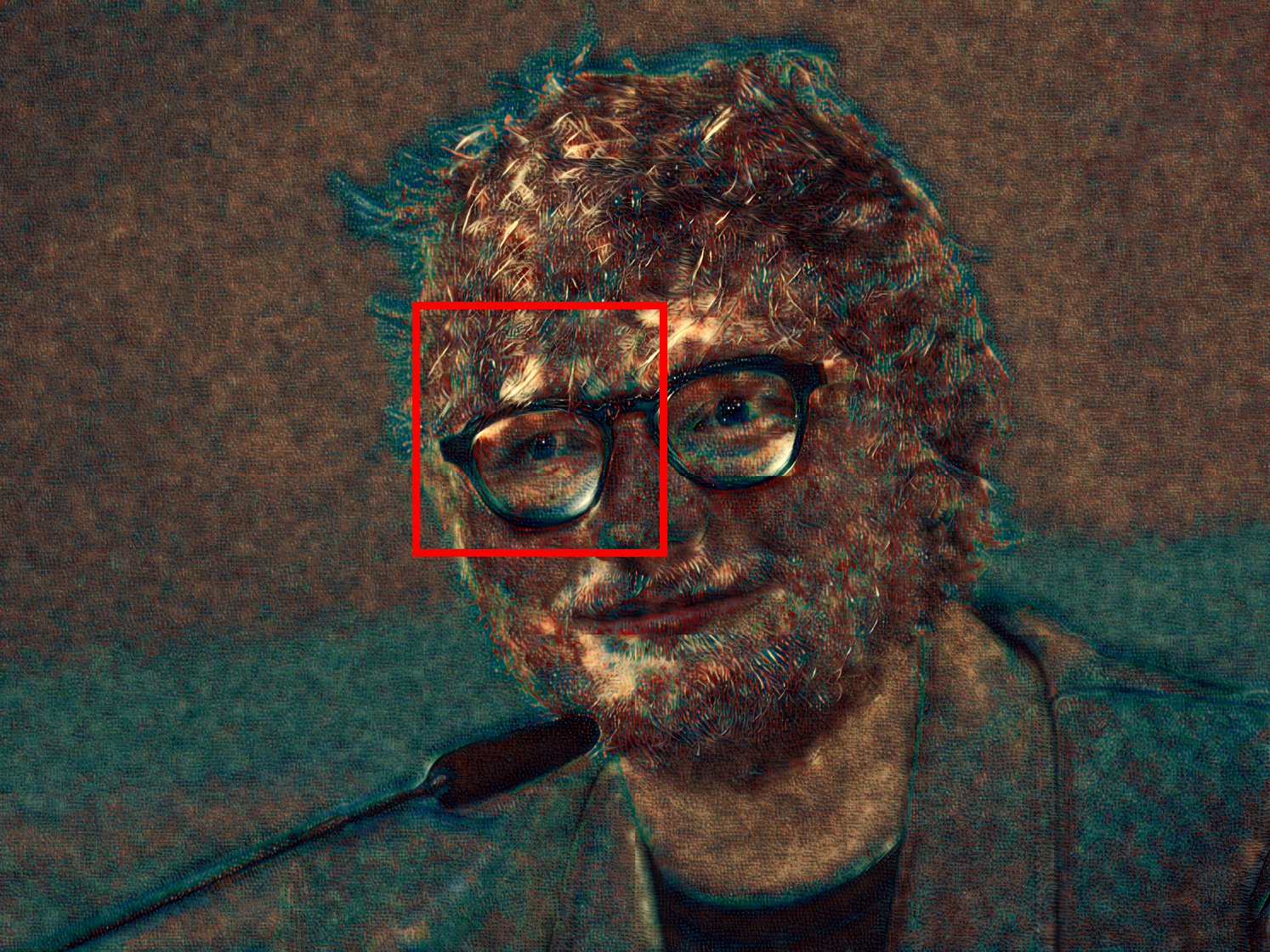}};
\node[anchor = north west] at
    (0,-0.02\linewidth){\includegraphics[width=\resultwidth]{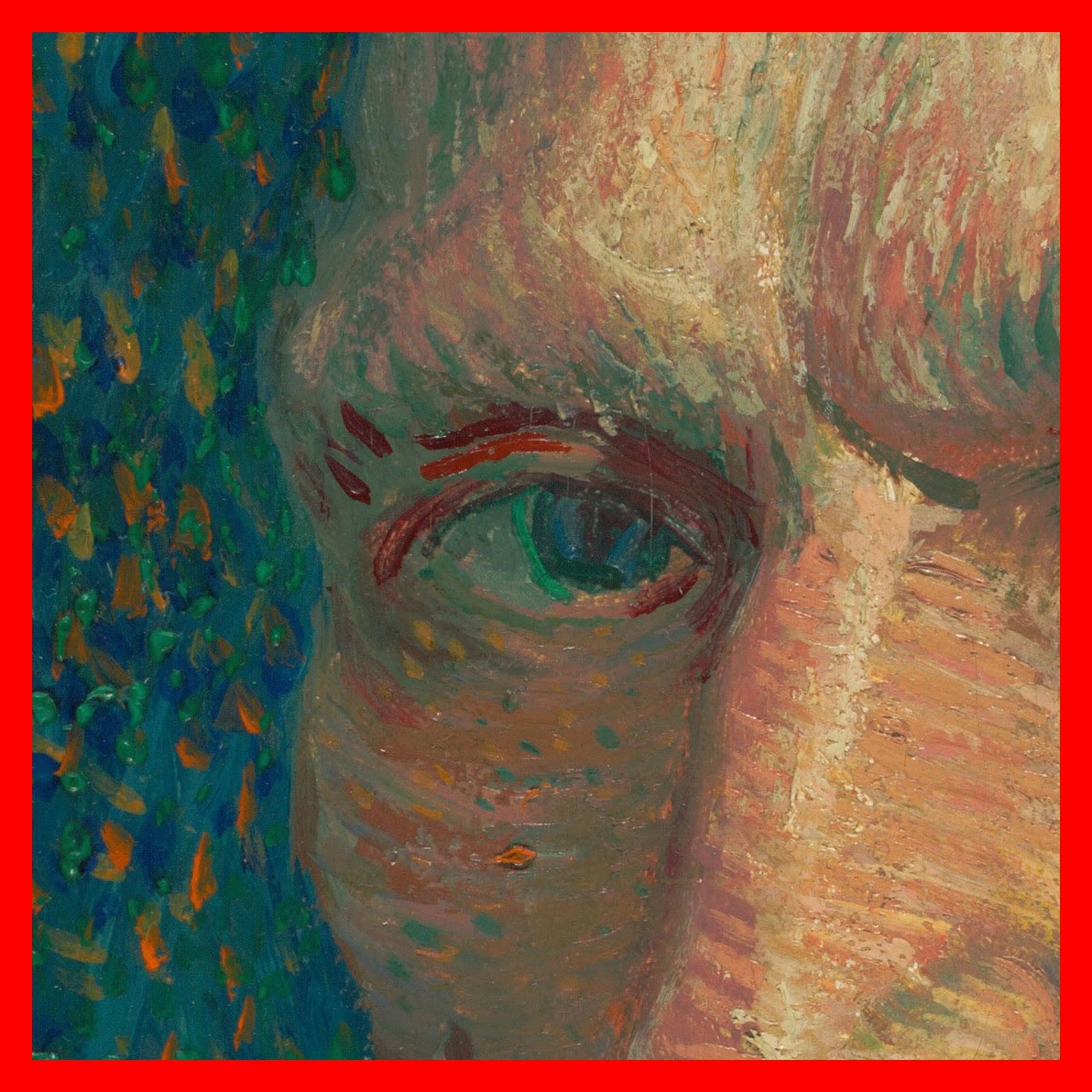}};
    \node[anchor = north west] at
    (\resultwidth+0.02\linewidth,-0.02\linewidth){\includegraphics[width=\resultwidth]{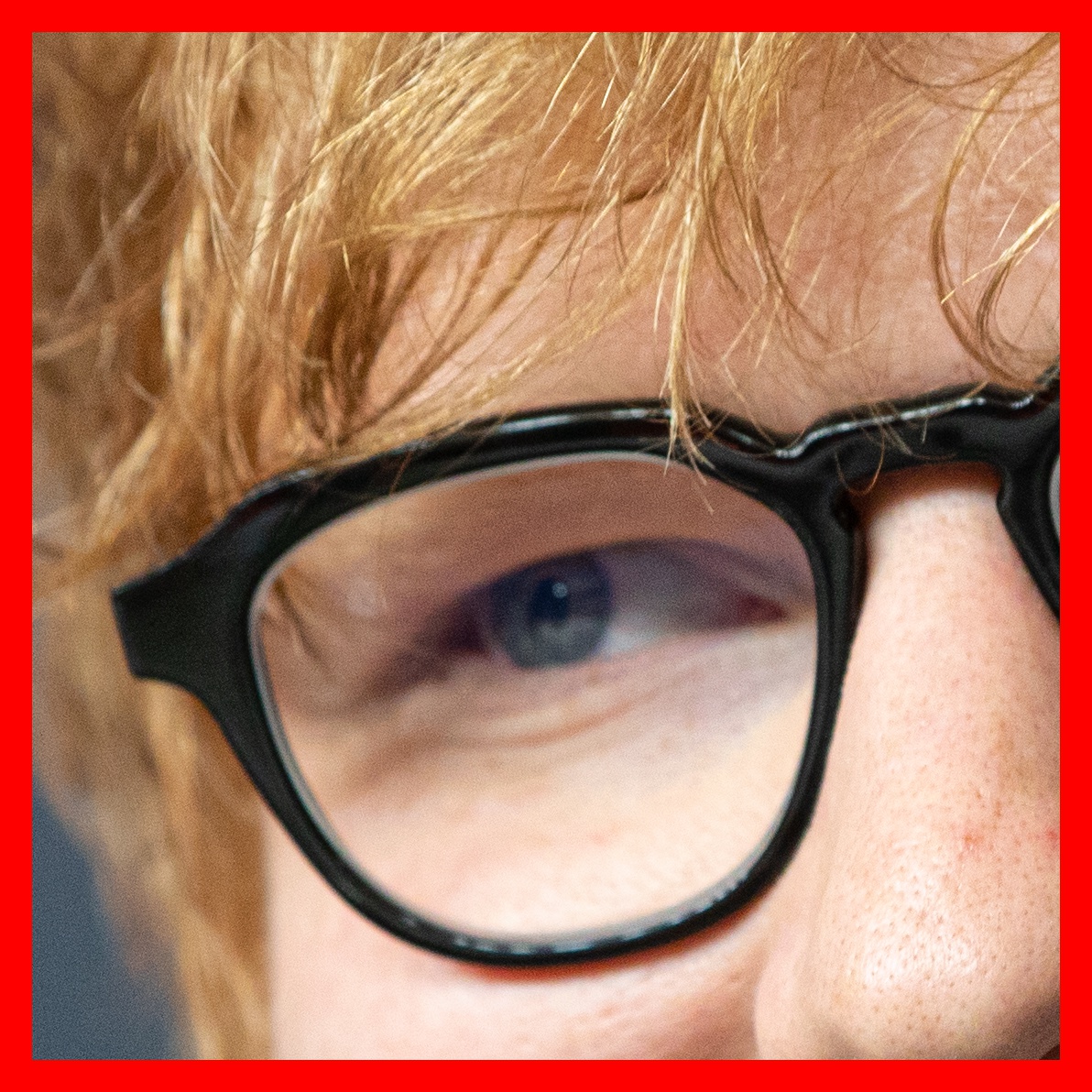}};
    \node[anchor = north west] at
    (2*\resultwidth+2*0.02\linewidth,-0.02\linewidth){\includegraphics[width=\resultwidth]{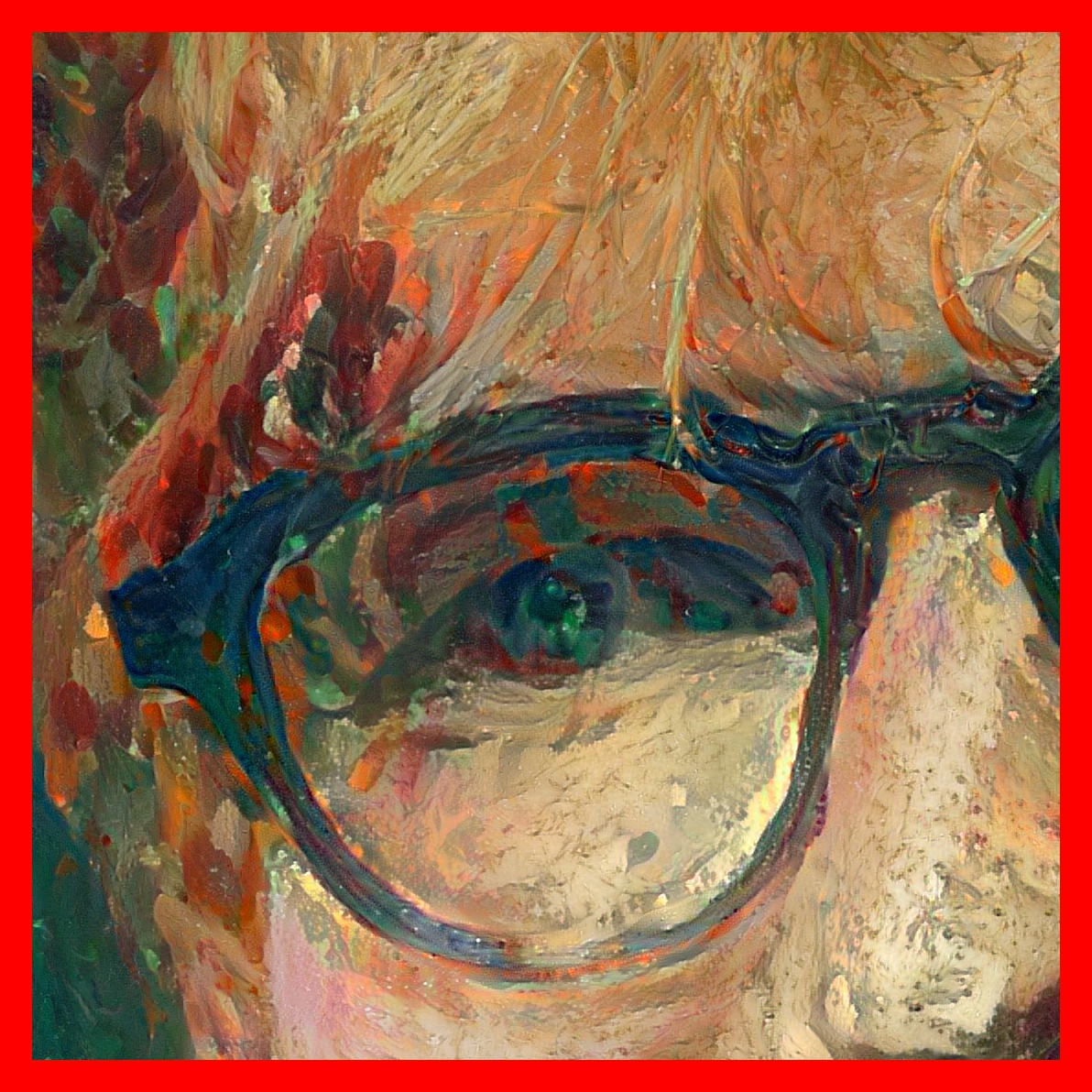}};
    \node[anchor = north west] at
    (3*\resultwidth+3*0.02\linewidth,-0.02\linewidth){\includegraphics[width=\resultwidth]{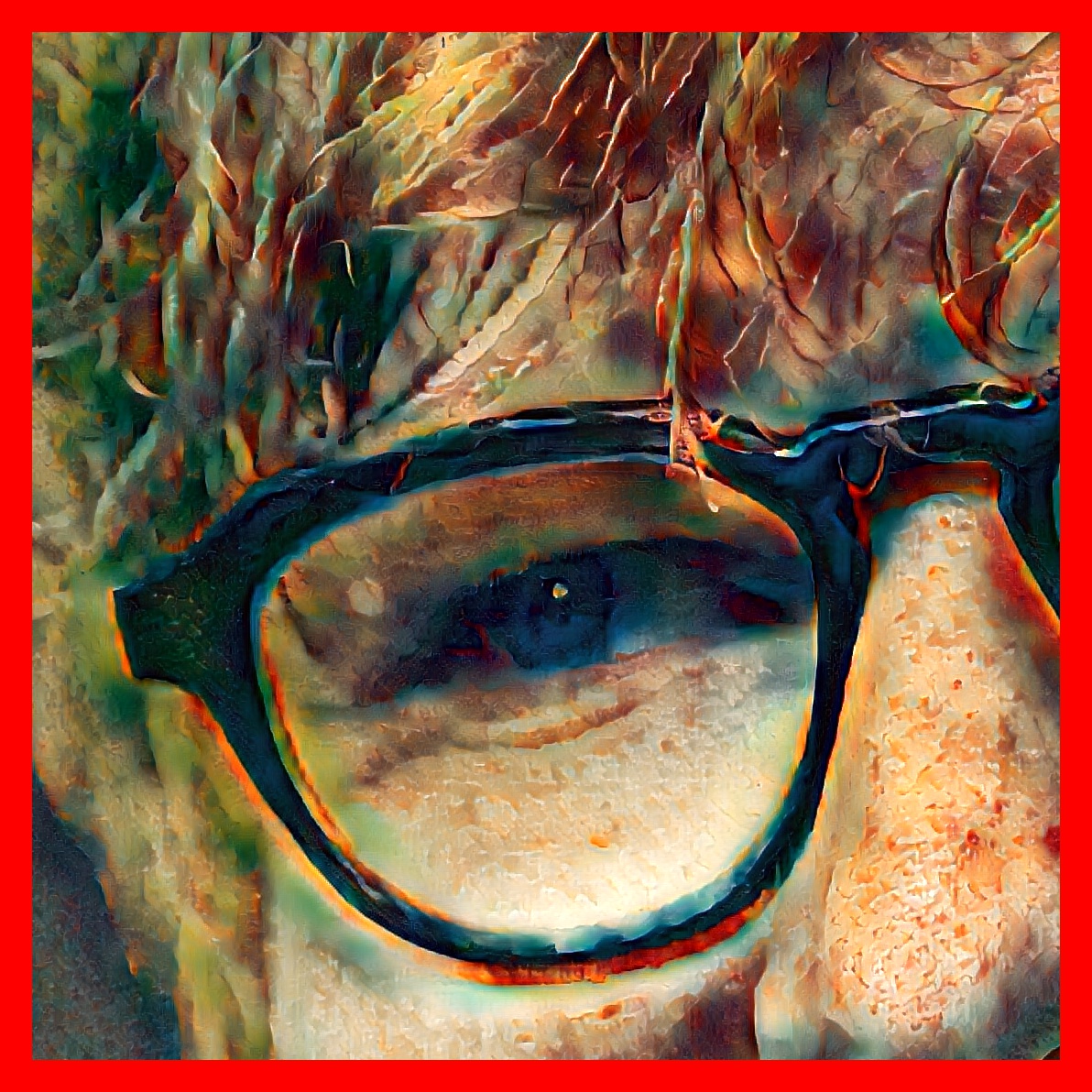}};
    \node[anchor = north west] at
    (4*\resultwidth+4*0.02\linewidth,-0.02\linewidth){\includegraphics[width=\resultwidth]{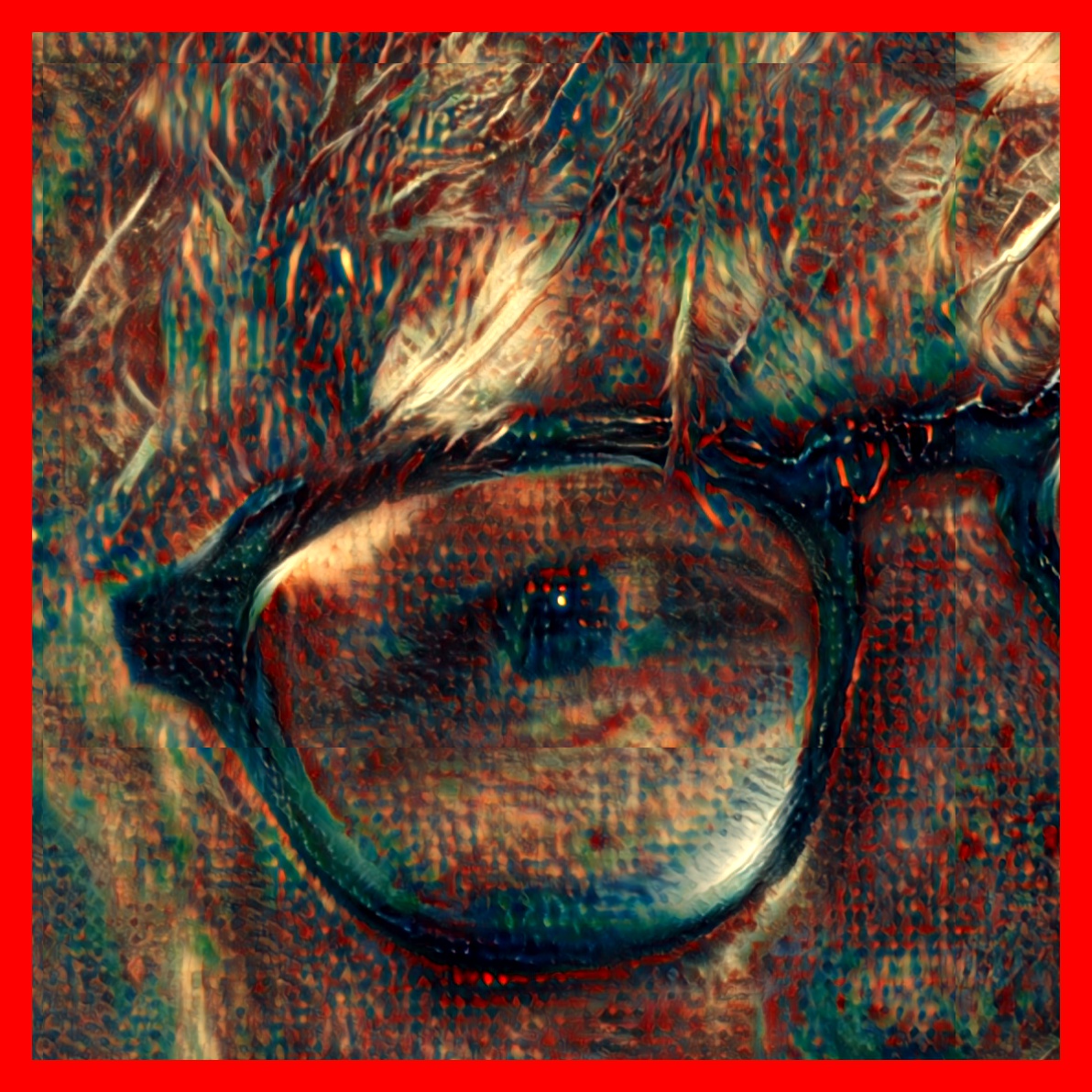}};
\end{tikzpicture}

\caption{Comparison of UHR style transfers. For each example, top row, left to right: style, content, our result (SPST), CD~\cite{Wang_2020_CVPR}, URST~\cite{Chen_Wang_Xie_Lu_Luo_towards_ultra_resolution_neural_style_transfer_thumbnail_instance_normalization_AAAI2022}. Bottom row: zoom in of the corresponding top row. First row: content ($3168\times 4752$), style ($2606\times 3176$), SPST uses three scales; third row: content (3024$\times$4032), style (3024$\times$3787), SPST uses three scales; fifth row: content ($4480\times5973$), style ($6000\times 4747$), SPST uses four scales.
In comparison to our results, state of the art very fast methods produce images with many defects: halo effect, neural artifacts, blending, unfaithful color palette, ... This result in images that do not look like painting contrary to SPST outputs.}
    \label{fig:style_transfer_comparison}
\end{figure*}

\subsection{Identity test for style transfer quality assessment}


\begin{table}[]
    \centering
    \small
    \begin{tabular}{@{}l@{}ccccc@{}}
        \toprule
        Method & PSNR  $\uparrow$ & SSIM $\uparrow$ & LPIPS $\downarrow$ & Gram $\downarrow$ & Time $\downarrow$\\
        \midrule
        SPST   & \underline{$24.6$} & $\mathbf{0.454}$ & $\mathbf{0.352}$ & $\mathbf{1.99\mathrm{e}5}$ & 25.1 \\
        SPST-fast   & $\mathbf{\mathbf{24.6}}$ & \underline{$0.438$} & \underline{$0.446$} & \underline{$4.08\mathrm{e}5$} &  4.96\\
        \midrule
        CD     & $21.8$ & $0.413$ & $0.500$ & $4.28\mathrm{e}7$ & \underline{0.373}\\
        URST   & $19.0$ & $0.413$ & $0.546$ & $6.77\mathrm{e}7$ & \textbf{0.232}\\
        \bottomrule
    \end{tabular}
    \caption{Quantitative evaluation of \emph{identity test} for UHR style transfer.
    Results include PSNR, SSIM~\cite{ssim}, LPIPS~\cite{Zhang_2018_CVPR}, the Gram (\emph{style distance}) metrics, and computation time in minutes for our results (SPST), its faster alternative (SPST-fast), CD~\cite{Wang_2020_CVPR} and URST~\cite{Chen_Wang_Xie_Lu_Luo_towards_ultra_resolution_neural_style_transfer_thumbnail_instance_normalization_AAAI2022}.
    All metrics are averages using $79$ HR paintings used as both content and style. Best results are in bold, second best underlined.
    Our iterative procedures SPST and SPST-fast are the best for all the image fidelity metrics but are respectively 100 and 20 times slower.}
    \label{tab:metrics}
\end{table}

Style transfer is an ill-posed problem by nature.
We {introduce here} an \emph{identity test} to evaluate if a method is able to reproduce a painting when using the same image for both content and style.
Two examples of this sanity check test are shown in Figure~\ref{fig:identity_test}.
We observe that our iterative algorithm is slightly less sharp than the original style image, yet high-resolution details from the paint texture are faithfully conveyed.
In comparison, the results of CD~\cite{Wang_2020_CVPR} suffer from color deviation and frequency artifacts while URST~\cite{Chen_Wang_Xie_Lu_Luo_towards_ultra_resolution_neural_style_transfer_thumbnail_instance_normalization_AAAI2022} applies a style transfer that is too homogeneous and present color and scale issues as already discussed. Again corresponding results for Figure~\ref{fig:identity_test} for SPST-fast are only reproduced in supp. mat. for readability.

Some previous works introduced a \emph{style distance}~\cite{Wang_2020_CVPR} that corresponds to the Gram loss for some VGG19 layers, showing again that fast approximate methods try to reproduce the algorithm of Gatys \emph{et al.} which we extend to UHR images.
Since we explicitly minimize this quantity, it is not fair to only consider this criterion for a quantitative evaluation.
For this reason, we also calculated PSNR, SSIM~\cite{ssim} and LPIPS~\cite{Zhang_2018_CVPR} metrics on a set of $79$ paint styles (see supp. mat.) to quantitatively evaluate our results.
We further report the ``Gram'' metric, that is, the style loss of Equation~\eqref{eq:style_loss} using the original Gram loss of Equation~\eqref{eq:gatys_gram_loss}, computed on UHR results using our localized approach.
The average scores reported in  Table~\ref{tab:metrics} confirm the good qualitative behavior discussed earlier: SPST and SPST-fast are by far the best for all the scores.
However, SPST and SPST-fast are respectively 100 and 20 times slower than the fastest method.


\begin{figure*}%
\newlength{\idtestwidth}
\setlength{\idtestwidth}{0.235\linewidth}%
\begin{tikzpicture}[spy using outlines={rectangle,magnification=14,height=\idtestwidth,width=\idtestwidth,
    every spy on node/.append style={thick}}, every node/.style={inner sep=0,outer sep=0}]
    \path (0cm,0cm) -- (\linewidth, 0cm);
\node[anchor=south west] (style) at (0,0){\includegraphics[width=\idtestwidth]{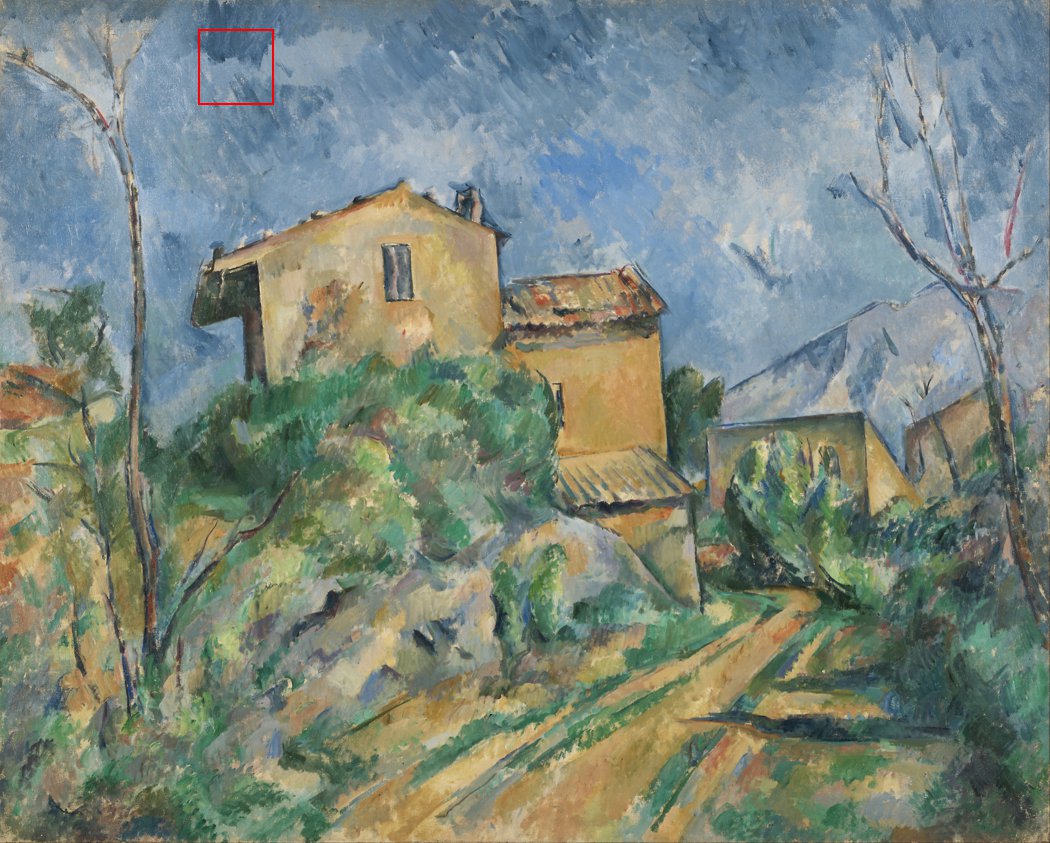}};
\node[anchor=south west]  (ours) at
    (\idtestwidth+0.02\linewidth,0){\includegraphics[width=\idtestwidth]{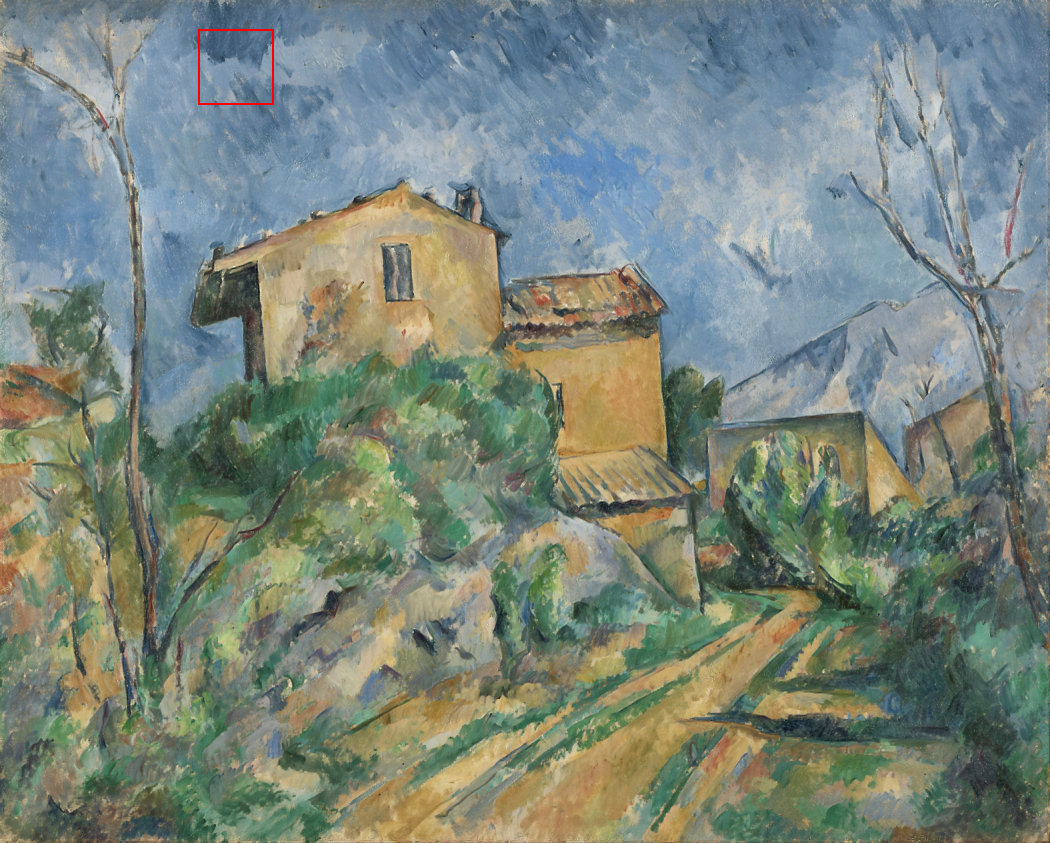}};
\node[anchor=south west]  (collab) at
    (2*\idtestwidth+2*0.02\linewidth,0){\includegraphics[width=\idtestwidth]{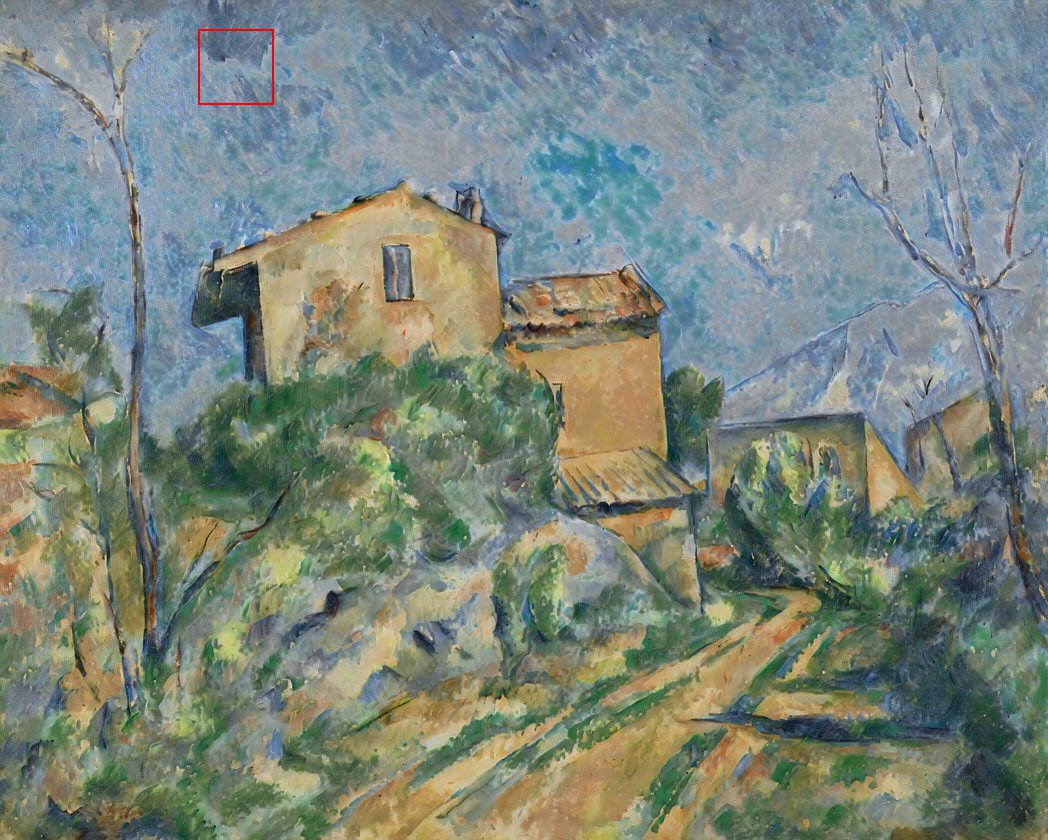}};
\node[anchor=south west] (urst) at
    (3*\idtestwidth+3*0.02\linewidth,0){\includegraphics[width=\idtestwidth]{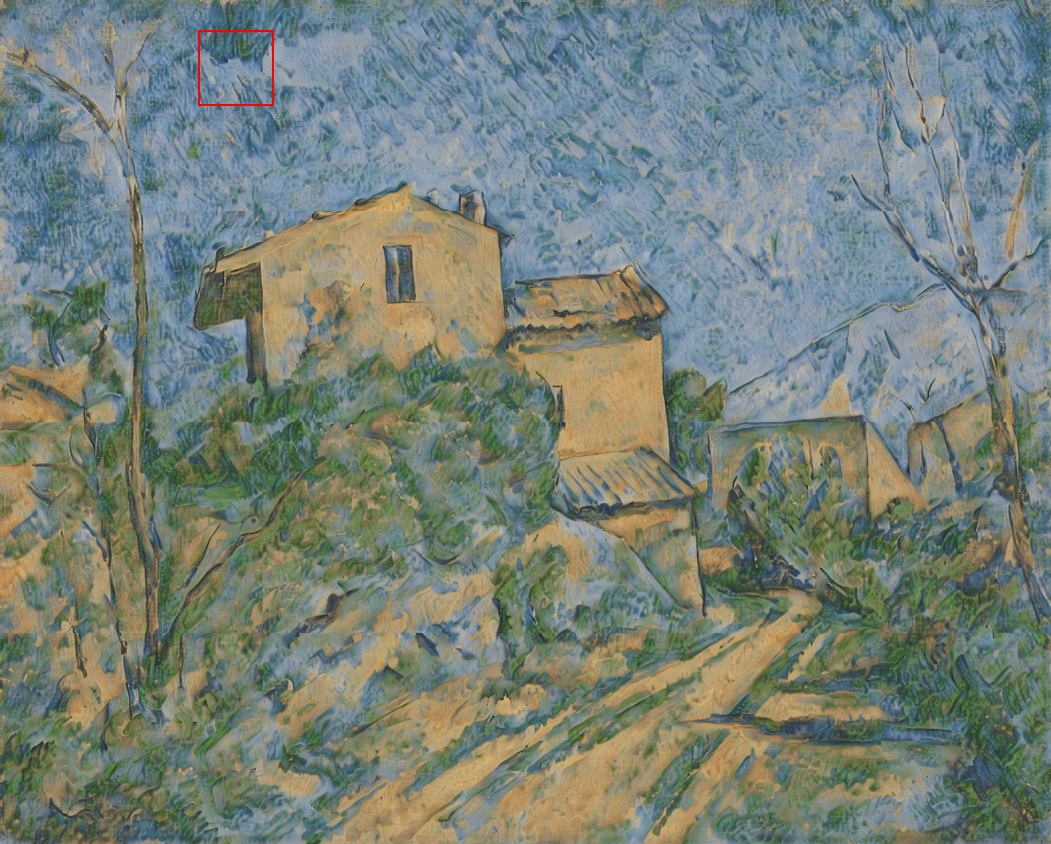}};
\node[] (styleLabel) at (0.5*\idtestwidth,3.55) {{Style/content image}};
\node[] (spstLabel) at (1.5*\idtestwidth+0.02\linewidth,3.55) {{SPST (ours)}};
\node[] (cdLabel) at (2.5*\idtestwidth+2*0.02\linewidth,3.55) {{CD}};
\node[] (urstLabel) at (3.5*\idtestwidth+3*0.02\linewidth,3.55) {{URST}};
\node[anchor = north west] at
(0,-0.02\linewidth){\includegraphics[width=\idtestwidth]{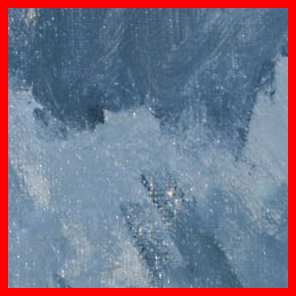}};
\node[anchor = north west] at
(\idtestwidth+0.02\linewidth,-0.02\linewidth){\includegraphics[width=\idtestwidth]{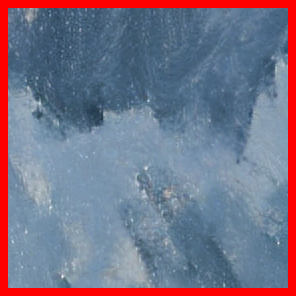}};
\node[anchor = north west] at
(2*\idtestwidth+2*0.02\linewidth,-0.02\linewidth){\includegraphics[width=\idtestwidth]{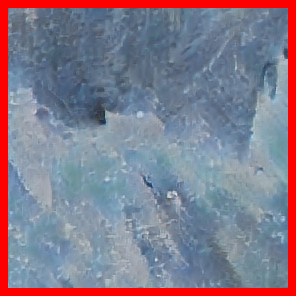}};
\node[anchor = north west] at
(3*\idtestwidth+3*0.02\linewidth,-0.02\linewidth){\includegraphics[width=\idtestwidth]{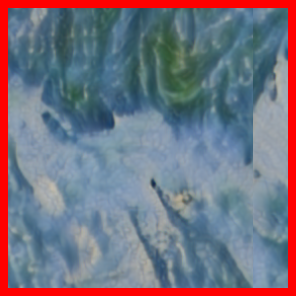}};
\end{tikzpicture}

\vspace{0.015\linewidth}

\begin{tikzpicture}[spy using outlines={rectangle,magnification=14,height=\idtestwidth,width=\idtestwidth,
    every spy on node/.append style={thick}}, every node/.style={inner sep=0,outer sep=0}]
    \path (0cm,0cm) -- (\linewidth, 0cm);
\node[anchor=south west] (style) at (0,0){\includegraphics[width=\idtestwidth]{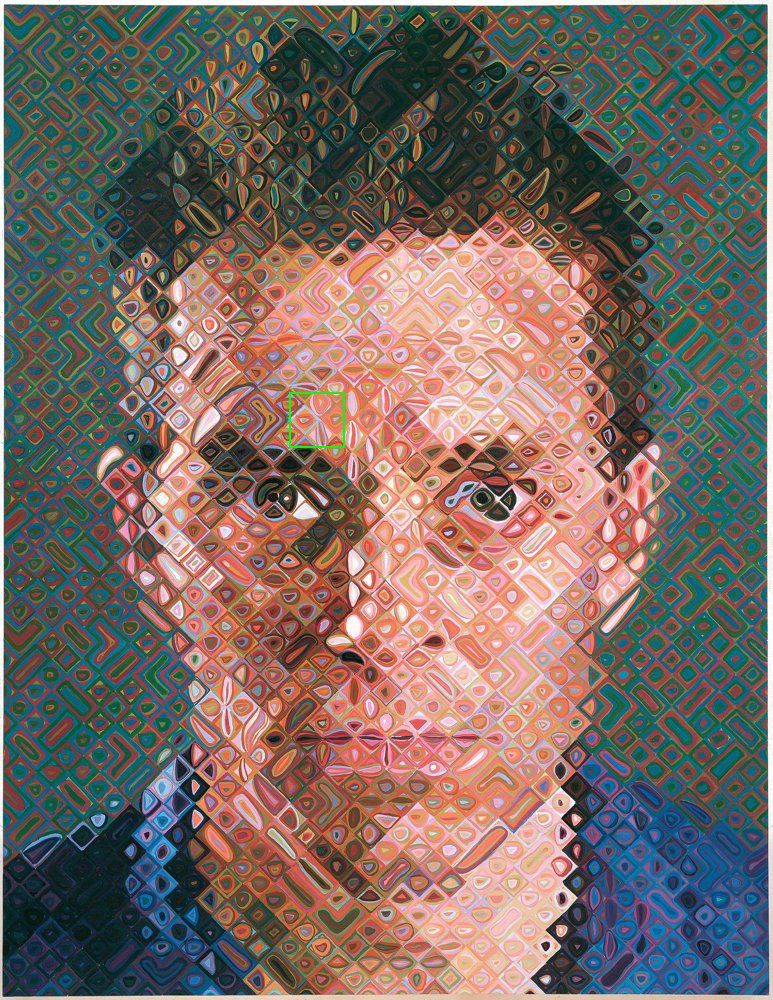}};
\node[anchor=south west]  (ours) at
    (\idtestwidth+0.02\linewidth,0){\includegraphics[width=\idtestwidth]{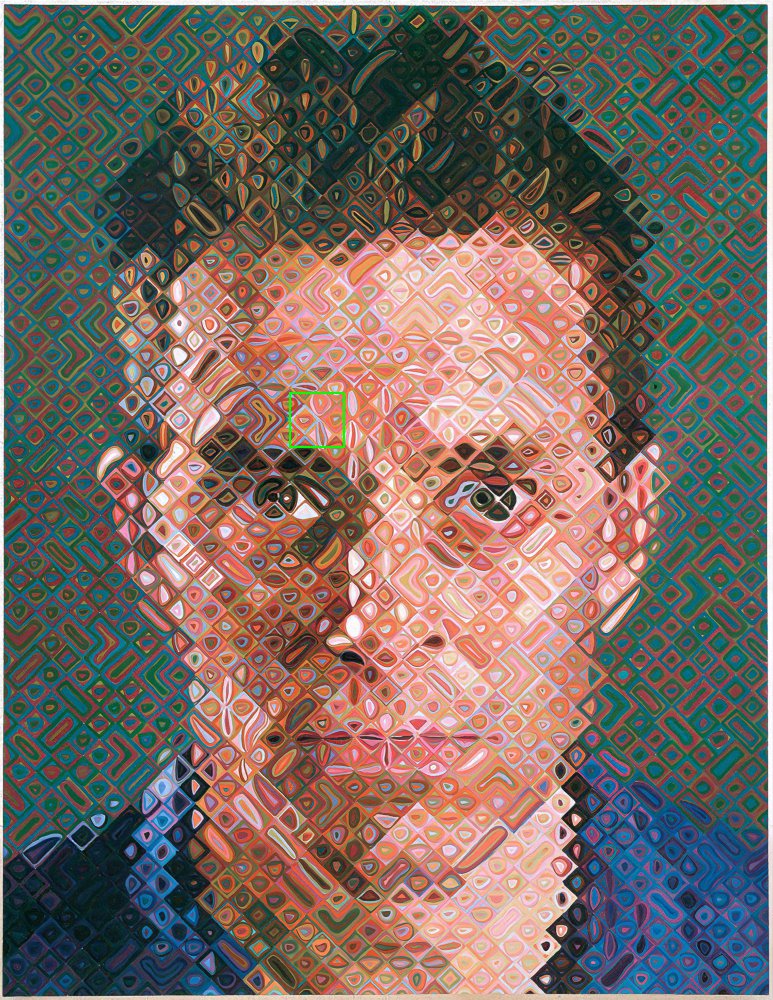}};
\node[anchor=south west]  (collab) at
    (2*\idtestwidth+2*0.02\linewidth,0){\includegraphics[width=\idtestwidth]{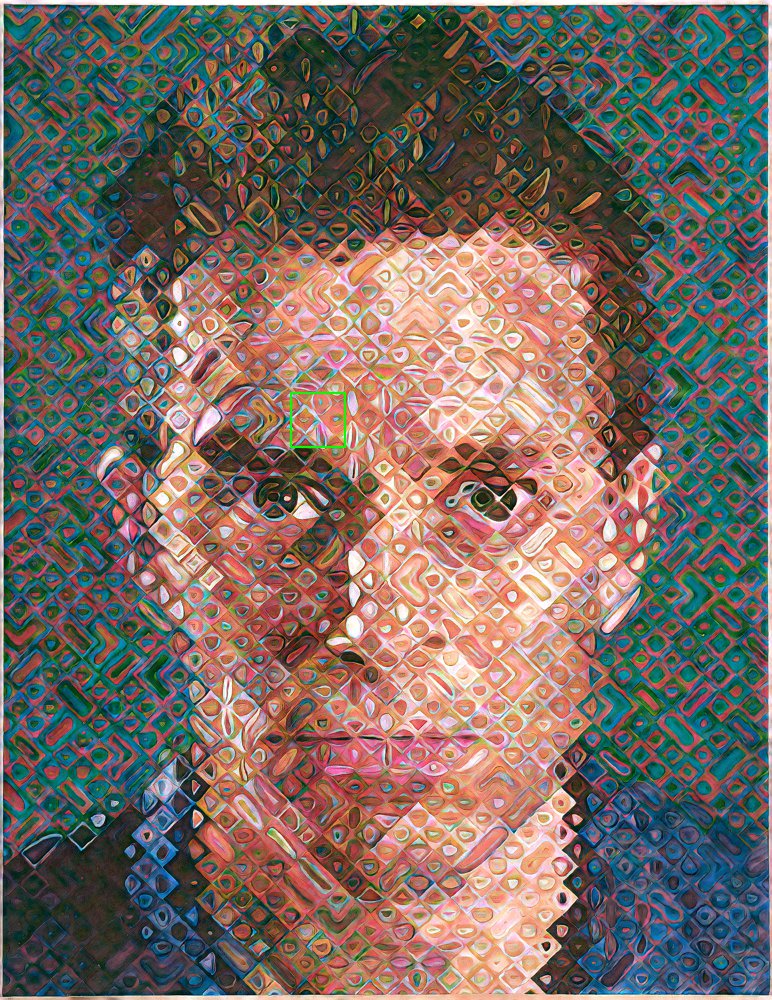}};
\node[anchor=south west] (urst) at
    (3*\idtestwidth+3*0.02\linewidth,0){\includegraphics[width=\idtestwidth]{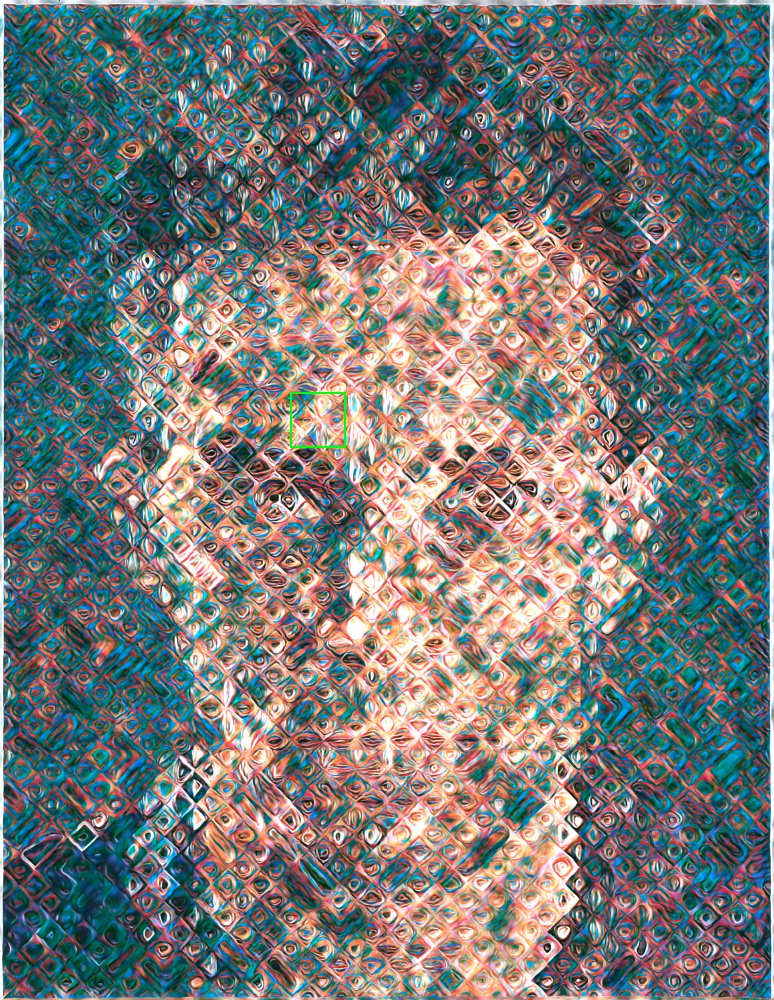}};

\node[anchor = north west] at
(0,-0.02\linewidth){\includegraphics[width=\idtestwidth]{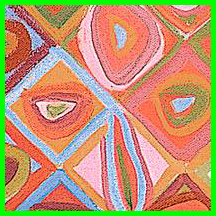}};
\node[anchor = north west] at
(\idtestwidth+0.02\linewidth,-0.02\linewidth){\includegraphics[width=\idtestwidth]{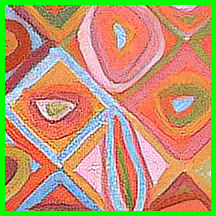}};
\node[anchor = north west] at
(2*\idtestwidth+2*0.02\linewidth,-0.02\linewidth){\includegraphics[width=\idtestwidth]{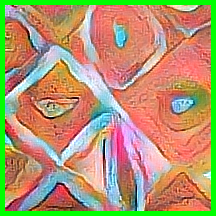}};
\node[anchor = north west] at
(3*\idtestwidth+3*0.02\linewidth,-0.02\linewidth){\includegraphics[width=\idtestwidth]{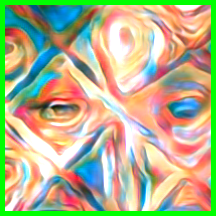}};
\end{tikzpicture}
\caption{Identity test: a style image is transferred to itself. We compare three style transfer strategies. From left to right: ground truth style, our result (SPST), CD~\cite{Wang_2020_CVPR}, URST~\cite{Chen_Wang_Xie_Lu_Luo_towards_ultra_resolution_neural_style_transfer_thumbnail_instance_normalization_AAAI2022}.
First row: The style image has resolution 3375$\times$4201; Third row: The style image has resolution 3095$\times$4000 (UHR images have been downscaled by $\times$4 factor to save memory).
Second and fourth row: Close-up view with true resolution.
Observe that our results are more faithful to the input painting and do not suffer from color blending.}
\label{fig:identity_test}
\end{figure*}

\subsection{\textcolor{coverletter}{Perceptual study}}

To further compare our results, we performed a \textcolor{coverletter}{perceptual study} comparing the fast version of our algorithm (SPST-fast) to CD~\cite{Wang_2020_CVPR} and URST~\cite{Chen_Wang_Xie_Lu_Luo_towards_ultra_resolution_neural_style_transfer_thumbnail_instance_normalization_AAAI2022}.

The \textcolor{coverletter}{perceptual study} consisted of several evaluation instances, each of which compared four images: the style used for the transfer and the results of the three methods (SPSt-fast, URST, and CD), which were displayed at random positions for each evaluation instance. Each participant was asked to select the result closest to the style of the style image among the three displayed results.

Participants were presented three types of experiments, each of which had five instances to evaluate, thus yielding a total of 15 instances to evaluate per candidate. The results were saved only if the participant conducted the whole test. The first two experiments aim to compare the results of our identity test. In one case, the overall performance of the methods is evaluated by displaying the complete results at a resolution of 1280$\times$720, and in the other case, the performance of the methods on fine details is evaluated by displaying a close-up of the results at a size of 512$\times$512. For the identity test, 79 painting styles were used and each participant was shown five random instances for the global evaluation and another five for the detail evaluation. The third experiment aims to compare the results of the three methods when transferring a painting style image to a generic content image. Only the overall performance of the methods is compared displaying the whole results at a resolution of 1280$\times$720. In this case, 13 pairs of style/content images were used, and five instances were randomly shown to each participant.

A total of 61 participants took the test, yielding a total of 305 evaluations for each type of experiment. All invited   participants  were image processing experts in academy and industry. The results of the study are shown in Table~\ref{tab:user-study}.
They confirm that our approach, both for the identity test (global and close-up) and the transfer of a painting style to any image, is by far superior to CD and URST in terms of visual quality: Our method is considered better more than 90$\%$ of the times as the one that better reproduces the style of the painting (Table~\ref{tab:user-study} third column).

\begin{table}[t]
    \centering
    \begin{tabular}{@{}lccc@{}}
    \toprule
     & \multicolumn{3}{c}{Voting results (\%)}\tabularnewline
     & Id global & Id detail & Style transfer \tabularnewline
    \midrule

    CD & \underline{6.56} &\underline{22.95} & \underline{4.92} \tabularnewline
    URST & 0.33 & 2.29 & 4.26 \tabularnewline
    SPST-fast & \textbf{93.11} & \textbf{74.75} & \textbf{90.82} \tabularnewline
    \bottomrule
    \end{tabular}
    \caption{
    This \textcolor{coverletter}{perceptual study} results shows the percentage of times each method was selected out of the 305 comparisons for each experiment. Best results are in bold, second best underlined.
    }
    \label{tab:user-study}
\end{table}



\section{\textcolor{coverletter}{Conclusion}}

This work presented the SPST algorithm, a provably correct extension of 
the Gatys \emph{et al.} style transfer algorithm to UHR images.
Regarding visual quality, our algorithm outperforms competing UHR methods by conveying a true painting feel thanks to faithful
HR details such as strokes, paint cracks, and canvas texture.
This is clearly supported by our \textcolor{coverletter}{perceptual study} and our proposed quantitative \emph{identity test}.
SPST also allows for the synthesis by example of high-quality UHR textures.
While the baseline SPST method can become prohibitively slow, even though its complexity scales linearly with image size,
we proposed a faster alternative SPST-fast that limits computations as the scale grows by exploiting the stability of multiscale style transfer.

As we have demonstrated, very fast methods do not reach a satisfying quality. They fail our proposed identity test due to the presence of many artifacts, and our results are considered more faithful to the style image by a vast majority of users.
This work also leads to conclude that very fast high-quality style transfer remains an open problem and that our results provide a new standard to assess the overall quality of such algorithms.


This work opens the way for several future research directions, from allowing local control for UHR style transfer~\cite{Gatys_etal_Controlling_perceptual_factors_in_neural_style_transfer_CVPR2017} to training fast CNN-based models to reproduce our results.
Another promising direction is to extend our framework to video or radiance fields style transfer for which reaching ultra-high resolution would be beneficial.

\bigskip
{\noindent\textbf{Acknowledgements:} B. Galerne and L. Raad acknowledge the support of the project MISTIC (ANR-19-CE40-005).}
\bigskip

{\noindent\textbf{Supplementary material:}
A full preprint version with higher image quality and  complete supplementary material is available here:\\%
\url{https://hal.archives-ouvertes.fr/hal-03897715/en}}


\printbibliography

\end{document}